\documentclass[twoside]{article}

\usepackage[accepted]{aistats2025}
\usepackage[round]{natbib}
\usepackage{nicefrac}

\usepackage{booktabs}
\usepackage{dsfont}
\usepackage{amsthm,amsmath,amsfonts,amssymb}
\usepackage{xcolor}
\newtheorem{prop}{Proposition}
\newcommand{\dd}{\mathop{}\! \mathrm{d}}
\usepackage{url}
\usepackage{algpseudocode}
\usepackage{algorithm}
\algrenewcommand\algorithmicindent{0.5em}%

\DeclareMathOperator*{\argmin}{arg\,min}
\usepackage{mathtools}
\usepackage{etextools}

\DeclarePairedDelimiter\floor{\lfloor}{\rfloor}
\usepackage{bbm}
\usepackage{multicol}
\usepackage{subcaption}
\usepackage{comment}

\usepackage{natbib}
\usepackage{color}
\usepackage{amssymb}
\usepackage{mathrsfs}
\usepackage{amsthm}

\begin{document}

%

%

\twocolumn[

\aistatstitle{Poisoning Bayesian Inference via Data Deletion and Replication}

\aistatsauthor{ Matthieu Carreau \And Roi Naveiro \And William N. Caballero }

\aistatsaddress{CUNEF Universidad \\ Nantes Université, \\ École Centrale Nantes, \\ CNRS LS2N, UMR 6004, \\ F-44000 Nantes, France 
 \And  CUNEF Universidad \And University of Arkansas} ]

\begin{abstract}

Research in adversarial machine learning (AML) has shown that statistical models are vulnerable to maliciously altered data. However, despite advances in Bayesian machine learning models, most AML research remains concentrated on classical techniques. Therefore, we focus on extending the white-box model poisoning paradigm to attack generic Bayesian inference, highlighting its vulnerability in adversarial contexts. A suite of attacks are developed that allow an attacker to steer the Bayesian posterior toward a target distribution through the strategic deletion and replication of true observations, even when only sampling access to the posterior is available.
Analytic properties of these algorithms are proven and their performance is empirically examined in both synthetic and real-world scenarios. With relatively little effort, the attacker is able to substantively alter the Bayesian's beliefs and, by accepting more risk, they can mold these beliefs to their will. By carefully constructing the adversarial posterior, surgical poisoning is achieved such that only targeted inferences are corrupted and others are minimally disturbed.

\end{abstract}

\section{INTRODUCTION}


Statistical analyses traditionally assume that the data-generation process is completely determined by a set of unknown parameters. However, adversarial machine learning (AML) challenges this assumption by allowing data to be systematically perturbed by a self-interested adversary, thereby incorporating a strategic element to the model-learning and prediction tasks. This insurgent perspective has permeated the algorithmic-culture of statistics but, with few exceptions \citep[e.g., see][]{gonzalez2021hypothesis}, is less explored in the data-modeling culture \citep{breiman2001statistical}. The impact of such omittance is conspicuously relevant in Bayesian settings since all inferences are underpinned by a posterior distribution calculated from the observed data. Thus, this research develops white-box poisoning attacks against generic Bayesian inference, underscoring its vulnerability in adversarial settings. 

We consider an oblivious Bayesian who fails to recognize the presence of an adversary. No opponent modeling has been accomplished, and the adversary is not considered in the likelihood or prior. The adversary desires to poison the true data such that the resulting posterior is drawn towards an alternative distribution of their choosing. Subject to a set of perturbation bounds, the attacker intends to accomplish this goal by only replicating or deleting truly observed data points. The resulting problem is a stochastic integer program characterized by an intractable objective function.

Adopting the perspective of the attacker, we develop a suite of attack methods under these conditions. Even though our objective function is intractable, we are able to compute an unbiased estimate of its gradient and embed it within myriad solution methods. Multiple variants of a two-stage stochastic gradient descent (SGD) attack are developed, where the continuous relaxation of the problem is optimized first, followed by the identification of a high-quality integer solution; several theoretical properties are explored.
Additionally, we introduce an alternative family of single-stage attacks that systematically searches along a perturbed gradient estimate, ensuring integrality constraints are always met. Empirical illustrations demonstrate the damaging effects of data poisoning on \textit{oblivious} Bayesian inference i.e., when the attacker's presence is overlooked. The effectiveness of each attack is further explored through systematic application against varied Bayesian models, showing substantive improvement over the benchmark fast-gradient sign method attack.




\section{RELATED WORK}

Among other factors, AML attacks are often typified by the attacker's intent and their knowledge \citep{rios2023adversarial}. Poisoning attacks corrupt the learning process whereas evasion (decision-time) attacks fool learned models in operation. Similarly, white-box attacks assume the attacker knows the form of the statistical model being learned and the data being used for that purpose, whereas gray- and black-box attacks assume partial knowledge and complete ignorance of these elements, respectively. Attacks of varied forms against support vector machines and neural networks are well-represented in the literature, e.g., see \citet{biggio2012poisoning}, \citet{xiao2015support}, \citet{kurakin2016adversarial}, and \citet{yang2017generative}. Numerous applications have been explored, varying from computer-vision \citep{kurakin2018adversarial} to large-language-model applications \citep{wan2023poisoning}. \citet{biggio2018wild} provide a thorough review of early AML research, subsequently updated by \citet{li2022review} and \citet{vorobeychik2022adversarial}, whereby Bayesian approaches are notably underrepresented. Likewise, \citet{cina2023wild} provide a tailored review on poisoning attacks in which the dearth of Bayesian emphasis persists. Although some AML research has focused on generative methods \citep[e.g., see][]{mei2015security} and Bayesian neural networks \citep[e.g., see][]{carbone2020robustness}, they are the exception, not the rule.

Nevertheless, although usually performed for distinct purposes, the corruption of sample data is not a wholly new concept in statistics. \citet{cook1977detection} considered the deletion of data points in a linear regression context for outlier detection. A Bayesian analog was developed by \citet{johnson1983predictive} to determine the effect of specified data points on posterior predictive distributions. Moreover, following the example of \citet{huber1992robust}, robust statistics often considers the effect of a contaminating distribution on inference procedures. Such work continues contemporaneously. \cite{broderick2020automatic} develop a robustness metric for $Z$-estimators by identifying the minimum proportion of sample data that must be deleted to reverse the statistician's conclusions. Recent work also examines the sensitivity of analogous MCMC-based analysis to the limited removal of data \citep{nguyen2024sensitivity}. This perspective is conceptually linked to the identification and use of coresets that seeks to remove a large proportion of the data set while preserving the statistician's conclusions \citep{zhang2021bayesian, yang2023towards}. 

Despite these conceptual similarities, the bulk of such previous work considers data corruption as a means to an end (e.g., for data reduction). It neglects that statistical inference occurs within a broader multi-agent setting. This perspective is espoused by \citet{schorfheide2012use}, \citet{gonzalez2021hypothesis}, \citet{bates2022principal} and \citet{bates2023incentive}. Akin to \citet{kamenica2011bayesian}, recent Bayesian perspectives in this vein reduce (partially) to information-design problems; they vary primarily upon the underlying inference task and the means through which information is modified. Individual tasks (e.g., hypothesis testing) with domain-specific influence have been considered, but study of adversarial data corruption on generic Bayesian inference is lacking. 





\section{BACKGROUND}

The \textit{oblivious} Bayesian, targeted by the adversary, is an agent (referred to as the defender, $D$) performing Bayesian inference on an unknown quantity $\theta \in \mathbb{R}^d$. Upon observing data $\mathbf{X}=(X_1, ..., X_n)^\top \in \mathcal{X} \subseteq \mathbb{R}^{n \times p}$, the posterior distribution encodes all information about $\theta$. Under untainted data, the posterior is computed as follows:
\begin{eqnarray*}
    \pi(\theta | \mathbf{X} ) = \frac{1}{Z} \exp \left( \sum_{i=1}^n \log \pi(X_i | \theta) \right) \pi(\theta)
\end{eqnarray*}
where $\pi(\theta)$ is the prior distribution, $\pi(X_i \vert \theta)$ is the likelihood function, and $Z$ is the normalization constant. 
However, computing the exact posterior is generally not feasible, so approximation methods must be used. Two widely employed families of approximation techniques are sampling-based (e.g., MCMC) and optimization-based (e.g., variational Bayes) methods. For the remainder of this manuscript, we assume that $D$ approximates the posterior using MCMC methods.

The adversary (attacker, $A$) poisons the inferences made by $D$ through the removal and/or replication of data points from $\mathbf{X}$ to steer the posterior towards a target distribution, i.e., the \textit{adversarial posterior}. This attack is represented by a vector $w \in \mathbb{Z}^n_{\geq 0}$ of length $n$ of non-negative integers, where $w_i = k > 1$ indicates that the $i$-th data point is repeated $k$ times, $w_i = 0$ indicates that the data point is removed, and $w_i = 1$ indicates that the data point is unchanged. The posterior induced by attack $w$ is thus given by:

\begin{eqnarray} \label{eq:weighted_posterior}
    \pi_w(\theta | \mathbf{X} ) = \frac{1}{Z(w)} \exp \left( \sum_{i=1}^n w_i \log \pi(X_i | \theta) \right) \pi(\theta)
\end{eqnarray}

The goal of the attacker is to find $w$ that will drive the posterior as close as possible to the target distribution while making minimal changes to remain undetected, as formalized in the next section.

\section{PROBLEM FORMULATION}

$A$ desires to select $w \in \mathbb{Z}^n_{\geq 0}$ that steers the posterior distribution towards the adversarial posterior, $\pi_A(\theta)$. The similarity between the tainted posterior $\pi_w(\theta | \mathbf{X})$ and the adversarial posterior $\pi_A(\theta)$ is evaluated using the forward Kullback-Leibler (KL) divergence, a choice we justify subsequently. To avoid detection, the attacker cannot change many data points. We formalize this by allowing a maximum number $B$ of \emph{data point manipulations} (deletions and replications). We also limit the number of repetitions for any given data point. Thus, the attacker's problem is given by:

\begin{eqnarray}\label{eq:main}
\min_w && \text{KL}(\pi_A(\theta) \parallel \pi_w(\theta | \mathbf{X})) \\
\text{s.t.} && \|w - \mathbf{1}\|_1 \leq B \nonumber \\
&& \|w \|_\infty \leq L  \nonumber \\
&& w \in \mathbb{Z}^n_{\geq 0} \nonumber
\end{eqnarray}

Herein, $\mathbf{1}$ is a  length-$n$ vector of ones, $B \in \mathbb{Z}_{\geq 0}$ is the maximum allowable manipulations, and $L \in \mathbb{Z}_{> 0}$ is the maximum repetitions of a given data point. 

Through algebraic manipulation, it is straightforward to see that minimizing the KL divergence in problem \eqref{eq:main} is equivalent to minimizing:
\begin{eqnarray*}
- w^\top \cdot \mathbb{E}_{\pi_A(\theta)}\left[ f_X(\theta)\right] + \log Z(w)
\end{eqnarray*}
where $f_X(\theta) = \left[ \log \pi(X_1 | \theta), \dots, \log\pi(X_n | \theta) \right]$. In general, this objective function cannot be evaluated exactly because the log-normalization constant, $\log Z(w)$, may be intractable. Moreover, even if $\log Z(w)$ was known, computing the expectation $\mathbb{E}_{\pi_A(\theta)}\left[ f_X(\theta)\right]$ cannot generally be accomplished analytically. The requirement that $w$ be an integer vector further complicates the optimization problem as well. Therefore, the difficulties in expeditiously identifying an optimal attack are apparent. 

Nevertheless, problem \eqref{eq:main} is endowed with useful analytical properties that suggest reasonable heuristics to approximate its solution. 

\begin{prop} \label{prop1}
    The objective function in problem \eqref{eq:main} is convex in $w$.
\end{prop}
\small
\begin{proof}
    The gradient of the objective function is given by:
    \begin{eqnarray*}
    & \nabla_w \big(\log Z(w) - w^\top \cdot \mathbb{E}_{\pi_A(\theta)}\left[ f_X(\theta)\right] \big) = \\
    & \qquad  \qquad \qquad \mathbb{E}_{\pi_w(\theta \vert \mathbf{X})}\left[ f_X(\theta)\right] - \mathbb{E}_{\pi_A(\theta)}\left[ f_X(\theta)\right]
    \end{eqnarray*}
    This holds because distributions of form \eqref{eq:weighted_posterior} constitute a subset of an exponential family with parameter $w$, log-partition function $\log Z(w)$, and sufficient statistics given by the log-likelihoods of each data point \citep{campbell2019sparse}. It is well known that the gradient of the log-partition function with respect to its natural parameter is the mean of the sufficient statistic. Hence, we have $\nabla_w \log Z(w) = \mathbb{E}_{\pi_w(\theta \vert \mathbf{X})}\left[ f_X(\theta)\right]$.
    Similarly, the Hessian of the objective function is given by:
    \begin{eqnarray} \label{eq:cov}
    \nabla_w \left( \nabla_w \log Z(w) \right) = \text{Cov}_{\pi_w(\theta | \mathbf{X})}\left( f_X(\theta), f_X(\theta)\right)
    \end{eqnarray}
    Since the Hessian equals the covariance matrix of the sufficient statistics with respect to $\pi_w$, it is positive semi-definite $\forall w$, and the objective function is convex.
\end{proof}
\normalsize


Moreover, the proof of Proposition \ref{prop1} implies a method for computing unbiased estimates of the gradient of our objective function with respect to the weights $w$. This suggests the use of projected SGD for solving a continuous relaxation of problem \eqref{eq:main}. The gradient of our objective function can be written as a sum of two expectations, allowing us to compute unbiased estimates of the gradient through sampling. Specifically, if we sample $\left \{\theta_i \right\}_{i=1}^P \overset{\text{iid}}{\sim} \pi_w(\theta | \mathbf{X})$ and $\left \{\theta_j \right\}_{i=1}^Q \overset{\text{iid}}{\sim} \pi_A(\theta)$, we can approximate the gradient as follows:
\begin{equation*}
    \frac{1}{P} \sum_{i=1}^{P} f_X(\theta_i) - \frac{1}{Q} \sum_{j=1}^{Q} f_X(\theta_j)
\end{equation*}

It can be shown that projected SGD is assured to converge to an optimal solution of the continuous relaxation under fairly general conditions. 

 \begin{prop} \label{prop2}
Let $\hat{\mu}_w = \frac{1}{P} \sum_{i=1}^{P} f_X(\theta_i)$ and $\hat{\eta} = \frac{1}{Q} \sum_{j=1}^{Q} f_X(\theta_j)$, and suppose $\left \{\theta_i \right\}_{i=1}^P \overset{\text{iid}}{\sim} \pi_w(\theta | \mathbf{X})$ and $\left \{\theta_j \right\}_{i=1}^Q \overset{\text{iid}}{\sim} \pi_A(\theta)$. Define $\mathcal{W}$ as a convex set of feasible $w$, and assume projected SGD is used to find iterates $w^t$ using the $L^2$ projection. If 
\begin{enumerate}
    \item $g(w) =\mathbb{E}_{\pi_w(\theta \vert \mathbf{X})}\left[ f_X(\theta)\right] - \mathbb{E}_{\pi_A(\theta)}\left[ f_X(\theta)\right]$ exists and is finite, $\forall w \in \mathcal{W}$, 
    \item the eigenvalues of $\text{Cov}_{\pi_w(\theta | \mathbf{X})}\left( f_X(\theta), f_X(\theta)\right)$ are lower-bounded by some $c>0$, $\forall w \in \mathcal{W}$,
    \item the eigenvalues of $\text{Cov}_{\pi_w(\theta|\mathbf{X})}(\hat{\mu}_w, \hat{\mu}_w)$, $\forall w \in \mathcal{W}$, and $\text{Cov}_{\pi_A(\theta)}(\hat{\eta}, \hat{\eta})$ are bounded above by $a$ and $b$, respectively, and
    \item the learning rate $\gamma_t = \frac{1}{ct}$,
\end{enumerate}

then 

$$\mathbb{E}\left[\lVert w^t - \tilde{w} \rVert^2 \right]\le \max \left\{ \frac{M^2}{c^2t}, \frac{\lVert w^0 - \tilde{w} \rVert^2}{t}\right\}.
$$ 

\noindent where $\tilde{w}$ is the optimal solution to the continuous relaxation, and $M^2=n(a+b) + \max_{w \in \mathcal{W}} \ \lVert g(w) \rVert^2 $.
\end{prop}
\begin{proof} \small
See Section 1 of supplementary material.
\end{proof} \normalsize


Thus, projected SGD converges to the optimal of the relaxation at a rate of $\mathcal{O}\left(\nicefrac{1}{t} \right)$ under the conditions specified in Proposition \ref{prop2}. Convergence can also be guaranteed under different conditions, such as with alternative learning rate schedules. For further details on projected SGD in similar contexts, we refer the reader to \citet{nemirovski2009robust} and \citet[Section 5.9,][]{shapiro2021lectures}. These convergence results typically assume that iid samples are drawn from $\pi_w(\theta | \mathbf{X})$ and $\pi_A(\theta)$. This assumption may not always hold in practice, as sampling often relies on MCMC methods when analytical expressions for these distributions are not available. Despite this, Section \ref{sec:empirical_evaluation} suggests that high-quality solutions can be attained even in the absence of independent samples.
 
Interestingly, computing the gradient of the objective function in problem \eqref{eq:main} does not require a closed-form for the adversarial posterior, only the ability to sample from it. This is not true of the reverse KL divergence, thereby motivating the form of equation \eqref{eq:main}. 

\begin{prop}
The gradient of the reverse KL divergence in $w$, i.e., $\nabla_w \text{KL}\left[\pi_w(\theta | \mathbf{X}) \parallel \pi_A(\theta) \right]$, equals

\begin{align*}
      -\mathbb{E}_{\pi_w(\theta \vert \mathbf{X})} \left[ f_X(\theta) \right] & \mathbb{E}_{\pi_w(\theta \vert \mathbf{X})} \left[\log \left( \frac{\pi(\theta)}{\pi_A(\theta)} \right) \right]  \\
    & +\mathbb{E}_{\pi_w(\theta \vert \mathbf{X})}\left[\log \left( \frac{\pi(\theta)}{\pi_A(\theta)} \right) f_X(\theta) \right] \\& + \text{Cov}_{\pi_w(\theta | \mathbf{X})} \left[ f_X, f_X^\top w \right] .
\end{align*}
    
\end{prop}
\begin{proof} \small
See Section 2 of supplementary material.
\end{proof} \normalsize

%

%
\noindent
Thus, if the reverse KL divergence is substituted in equation \eqref{eq:main}, it is feasible to compute sampling-based unbiased estimates of the objective function's gradient; however, it requires evaluating $\pi_A(\theta)$ which may not be available. By using the forward KL divergence, we avoid this complexity and increase applicability.

\section{SOLUTION METHODS}

Leveraging the gradient estimate identified previously, this section sets forth the rounded-relaxation and integer-steps-coordinate-descent attack families, along with a modified baseline attack from the literature. 

Beginning with the rounded-relaxation family, we note that Propositions \ref{prop1} and \ref{prop2} imply a natural heuristic for solving problem \eqref{eq:main}. This involves first solving its continuous relaxation via projected SGD, and then solving a constrained rounding problem to obtain an integer feasible solution. This heuristic is detailed in Algorithm \ref{alg:sgd}. For illustrative clarity, define the feasible set for the relaxed problem as
\begin{equation*}
    \mathcal{W} = \left\{w \in \mathbb{R}^n \mid w \succeq 0, \|w\|_\infty \leq L, \|w - \mathbf{1}\|_1 \leq B \right\}
\end{equation*}
and the $L^2$ projection operator as 

$$
\Pi_\mathcal{W}\left(w\right) = \argmin_{w' \in \mathcal{W}} \lVert w' - w\rVert^2_2.
$$

\begin{algorithm}[htb]
\caption{SGD Rounded Relaxation (SGD-R2)}\label{alg:sgd}
\begin{algorithmic}
\State \textbf{Input:} Initial point $w$, learning rate schedule $(\gamma_t)_{t=0}^\infty$, number of samples $P$ and $Q$
\State \textbf{Initialize:} Set iteration count $t = 0$
\Repeat
    \State Sample $\left( \theta_i \right)_{i=1}^P \sim \pi_w(\theta | \mathbf{X})$ 
    \State Sample $\left( \theta_j \right)_{j=1}^Q \sim \pi_A(\theta)$
    \State Compute gradient estimate:
    \[
    \hat{g} \gets \frac{1}{P} \sum_{i=1}^{P} f_X(\theta_i) - \frac{1}{Q} \sum_{j=1}^{Q} f_X(\theta_j)
    \]
    \State Update $w$ using a projected SGD step:
    \[
    w_{\text{proj}} \gets \Pi_\mathcal{W}\left(w - \gamma_t \hat{g}\right) 
    \]
    \State Set $w \gets w_{\text{proj}}$
    \State Increment iteration count $t \gets t + 1$
\Until{stopping criterion is met}
\State Solve the constrained rounding problem
\[
w^* = \argmin_{w' \in \mathbb{Z}^n_{\geq 0} \cap \mathcal{W}} \Vert w - w'\Vert_2^2
\]
\State \Return $w^*$
\end{algorithmic}
\end{algorithm}

Note that in each iteration of Algorithm~\ref{alg:sgd}, we must generate samples from $ \pi_w(\theta | \mathbf{X}) $, with $ w $ changing at every step. These samples are typically obtained using MCMC methods, which directly impact the per-iteration time complexity of the algorithm. When MCMC is required, the sampling step becomes the most computationally intensive part of any $\hat{g}$-based search, effectively dictating the overall time complexity based on the needed sample sizes $ P $ and $ Q $. In practice, achieving reasonable effective sample sizes often necessitates large values of $ P $ and $ Q $, though this requirement varies across different MCMC techniques. An important empirical observation is that since $ w $ changes gradually, starting each MCMC run with samples from the previous iteration can significantly reduce the number of burn-in iterations needed for convergence, thereby alleviating some of the computational burden.



Under the conditions presented in Proposition \ref{prop2}, projected SGD will converge to an optimal solution of the continuous relaxation; otherwise, it is intended to identify a high-quality solution. An algorithm termination criterion found to work well in practice entailed stopping the search when an estimate of the objective function's decrease becomes negligible compared to that estimated in the first iteration. Such estimates are based on a Taylor expansion at $w$. Finally, the constrained rounding problem in the last step can be solved optimally using a straightforward procedure, as described in the following proposition.
%
\begin{prop} \label{prop4}
Letting  $ N_{\text{max}} = B - \sum_i \lfloor \Delta_i \rfloor $ and $ \Delta_i = |w_i - 1| $, an optimal solution to the constrained rounding problem can be found by setting 

\begin{equation*}
    w^*_i = 1 + \text{sign}(w_i - 1)(\lfloor \Delta_i \rfloor + \alpha_i), 
\end{equation*}
where $\alpha_i \in \{0,1\}$ are assigned as follows:
\begin{enumerate}
    \item Initialize $\alpha_i=0, \ \forall i$,
    \item If $ \left\Vert \left\{ i \middle| \Delta_i - \lfloor \Delta_i \rfloor > \frac{1}{2} \right\} \right\Vert \le N_{max}$ then for each $i$, set $ \alpha_i = \mathds{1}_{\left( \frac{1}{2}, 1 \right]}(\Delta_i - \lfloor \Delta_i \rfloor)$,
    \item Otherwise, set $ \alpha_i = 1 $ for the $ N_{\text{max}} $ data points with the highest values of $ \Delta_i - \lfloor \Delta_i \rfloor $.
\end{enumerate}
\end{prop}
\begin{proof} \small
See Section 3 of supplementary material. 
\end{proof} \normalsize

We propose two variants of Algorithm~\ref{alg:sgd} designed to improve empirical runtime. We observed that the convergence of Algorithm~\ref{alg:sgd} for solving the continuous relaxation problem can slow down when $ w $ reaches the boundary of the feasible set $ \mathcal{W} $. This slowdown appears to be due to the gradient becoming nearly orthogonal to the boundary, resulting in minimal changes in $ w $ after projection. To accelerate convergence, we replace SGD with the Adam optimizer within Algorithm~\ref{alg:sgd}. Under suitable conditions, Adam converges at a rate of $\mathcal{O}\left(\nicefrac{1}{\sqrt{t}} \right)$; see \citet{kingma2014adam} and \citet{reddi2019convergence}. Since Adam scales the gradient based on a second-moment estimate, it allows for larger steps in $ w $ at each iteration. Specifically, we consider $ w - \Pi_{\mathcal{W}}(w - \gamma_t \hat{g}) $ as an estimate of the gradient, apply an Adam update, and then project the result back onto the feasible set $ \mathcal{W} $. We refer to this heuristic as Adam Rounded Relaxation (Adam-R2).

A faster variation of Algorithm \ref{alg:sgd} involves using the second-order Taylor expansion of the objective function. The Hessian $\mathbf{H}$ can be easily estimated from the samples $\left( \theta_j \right)_{j=1}^Q$ using equation \eqref{eq:cov}. 
Indeed, the strong law of large numbers guarantees that the estimate  $\mathbf{\hat{H}}$ converges almost surely to $\mathbf{H}$ and is thus strongly consistent. 
The update for $w$ is then given by:
\begin{equation*}
    w \gets \argmin_{w' \in \mathcal{W}} \hat{g}^\top(w'-w) + \frac{1}{2}(w'-w)^\top\mathbf{\hat H}(w'-w)
\end{equation*}
This approach leverages curvature information to provide a more informed update direction and is referred to as Second Order Rounded Relaxation (2O-R2).  

Having explored the R2 family of attacks, we pivot our attention to two simpler techniques. The first serves as a modified baseline from the literature, and the latter is an alternative family to the R2 methods.

\paragraph{Fast Gradient Sign Method (FGSM):} This algorithm is adapted from \citet{goodfellow2014explaining}. In the FGSM, we start with $w = \mathbf{1}$ and compute a single estimation $\hat{g}$ of $\nabla_w \mathcal{D}_{KL}[\pi_A(\theta) || \pi_w(\theta | X)]$. If $L\ge 2$, we select the $B$ data points with the highest gradient magnitudes and update their weights by $\pm 1$, based on the sign of the partial derivatives. Alternatively, if $L=1$, only deletions are feasible; the $B$ data points having the most positive gradients are selected and their weights updated by -1.

\begin{algorithm}[htb]
\caption{Integer-Steps Coordinate Descent}\label{alg:iscd}
\begin{algorithmic}
\State Initialize $w \gets \mathbf{1}$
\Repeat
    \State Estimate the gradient $\hat{g}$ and the Hessian $\mathbf{\hat{H}}$ 
    \State Choose a data point:
    \begin{align*}
        j \gets \argmin_i & \Big\{ - |\hat{g}_i| + \frac{1}{2}\mathbf{\hat{H}}_{i,i} \ \Big| \ 1 \leq i \leq n, \\
        & \qquad w - \text{sign}(\hat{g}_i) e_i \in \mathcal{W} \Big\}
    \end{align*}
    \State Update the weight vector:
    $$
    w \gets w - \text{sign}(\hat{g}_j) e_j
    $$
\Until{
    $ \forall i$ for which $\ w - \text{sign}(\hat{g}_i) e_i \in \mathcal{W}$ we have $- |\hat{g}_i| + \frac{1}{2} \mathbf{\hat{H}}_{i,i} > 0 $, i.e., the estimated increase in the objective function is positive for all feasible neighbors or a maximum number of iterations is reached.
}
\State \Return $w$
\end{algorithmic}
\end{algorithm}

\paragraph{Integer-Steps Coordinate Descent (ISCD):} 
At each iteration of ISCD, we select a new weight vector from the set of feasible $2n$ neighbors of the current vector based on the $L^1$ distance, i.e., $\{w \pm e_i\}$, where $e_i$ is the $i$-th vector of the canonical basis of $\mathbb{R}^n$.
The selection process is based on the estimated decrease in the objective function, using either a first-order or second-order Taylor expansion, respectively referred to as 1O-ISCD and 2O-ISCD. Algorithm \ref{alg:iscd} details the second-order method. The first-order method is obtained by setting $\mathbf{\hat{H}} = 0$.
The algorithm stops when selecting any feasible neighbor is expected to increase the objective function. However, this stopping criterion may never be reached, especially if the $L^1$ constraint is not tight. Thus, we also impose a maximum number of iterations, set slightly greater than $B$.

FGSM-based approaches are widely used for generating adversarial examples and serve as benchmarks in the field. However, because most applications assume access to true gradients, we adapt our FGSM attack to utilize stochastic gradients instead. Similarly, the ISCD attacks integrate our FGSM construct within an iterative search, aiming to exploit and update gradient information more efficiently. The ISCD attacks can be viewed as the converse of SGD-R2 and its variants: while SGD-R2 follows the gradient direction but sacrifices integrality, ISCD maintains integrality by searching along a modified gradient estimate. Finally, when MCMC is required, the time complexity of both ISCD and FGSM methods is governed by the estimation of $\hat{g}$. We further explore the convergence of these attacks in Section 5.3 of the supplementary material.


\section{EMPIRICAL EVALUATION} \label{sec:empirical_evaluation}

In this section, we compare the proposed heuristics and evaluate their efficacy using both simulated and real-world examples. For all heuristics, gradient estimation is performed via sampling. When a closed-form posterior is not available, we employ MCMC methods using the No-U-Turn Sampler from the \textit{NumPyro} library \citep{phan2019composable}.
We assess the performance of different attacks using two types of metrics: (1) the KL divergence between the adversarial and induced posteriors, computed exactly when possible or approximated otherwise; and (2) various summaries of the induced posterior. These metrics are evaluated with respect to $ B $, the number of permitted data manipulations, which serves as a proxy for attack intensity.
Notice that our attacks target inference on parameters and, as such, we test it against models with interpretable parameters. While deep Bayesian models could be attacked using our framework, their parameters often lack interpretability, making them less compelling targets for the types of manipulations we study.
All computations were performed on a server equipped with 88 \textit{Intel(R) Xeon(R) E5-2699 v4} CPUs and 252 GB of memory. The code, including the hyperparameters used for each heuristic, is available at \url{https://github.com/Matthieu-Carreau/Poisoning_Bayesian_Inference}.
The computational times for the attacks presented in the following sections are reported in Section 7 of the supplementary material.



\subsection{Simulation Study: Bayesian Linear Regression} \label{subsec:simstudy}

We consider a linear regression model with a single predictor. The simulated data consists of 100 observations, where the predictor is normally distributed. The response  is generated using a linear model with an intercept $ \beta_0 = 0.5 $, a slope $ \beta_1 = 0.3 $, and Gaussian noise with a standard deviation of 0.5.
%
%
%
We consider the conjugate case such that $\beta \,|\, \sigma^2 \sim \mathcal{N}(\mu_0, \sigma^2 \Lambda_0^{-1})$, $\sigma^2 \sim \text{Inv-Gamma}(a_0, b_0)$, $\mu_0 = [0, 0]^\top$, $\Lambda_0 = \frac{1}{100}I_2$, $a_0 = 2$, and $b_0 = 2$. Given these priors, the observed data $\mathcal{D}$, and the weight vector $w$, the tainted posterior distribution for $\beta$ and $\sigma^2$ follows a normal-inverse-gamma (NIG) distribution. The posterior parameters $\mu_n$, $\Lambda_n$, $a_n$, and $b_n$ are functions of $w$. Section 4 of the supplementary material provides derivations.
\begin{figure*}[h!]
\begin{multicols}{4}
    \includegraphics[width=\linewidth]{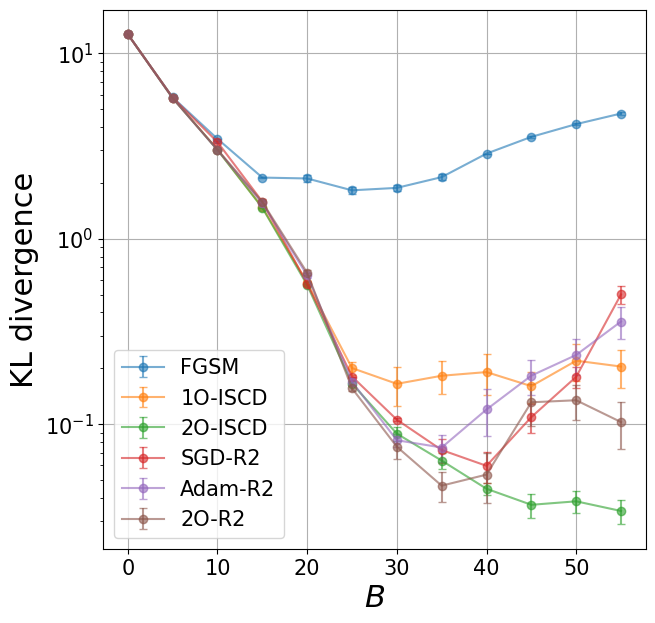}
    \subcaption{}\label{subfig:heuristics_nig_KL}\par 
    \includegraphics[width=\linewidth]{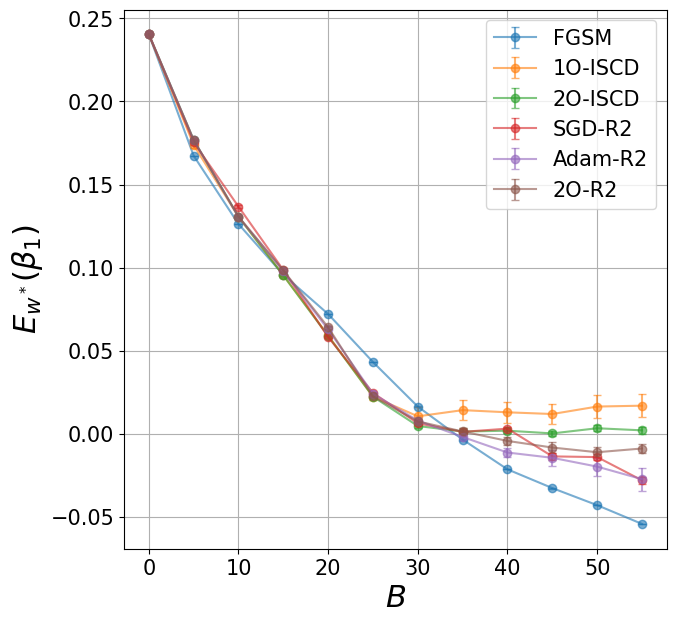}
    \subcaption{}\label{subfig:heuristics_nig_mean}\par 
    \includegraphics[width=\linewidth]{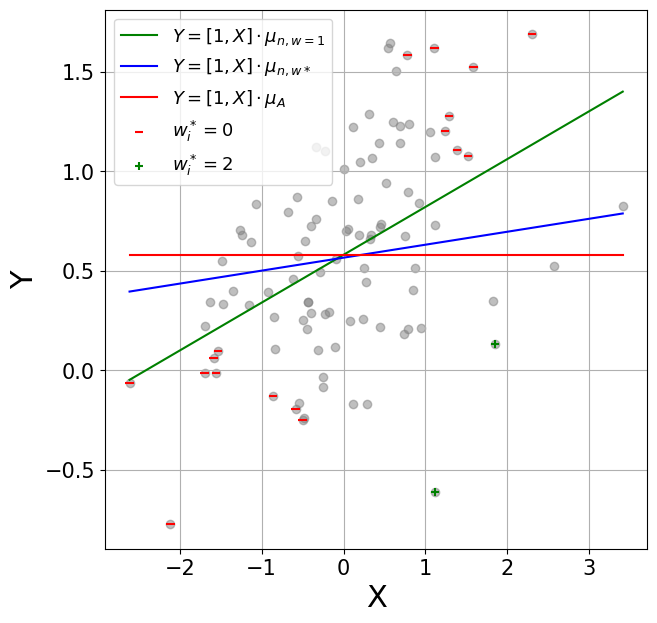}
    \subcaption{}\label{subfig:attacks_nig}\par 
    \includegraphics[width=\linewidth]{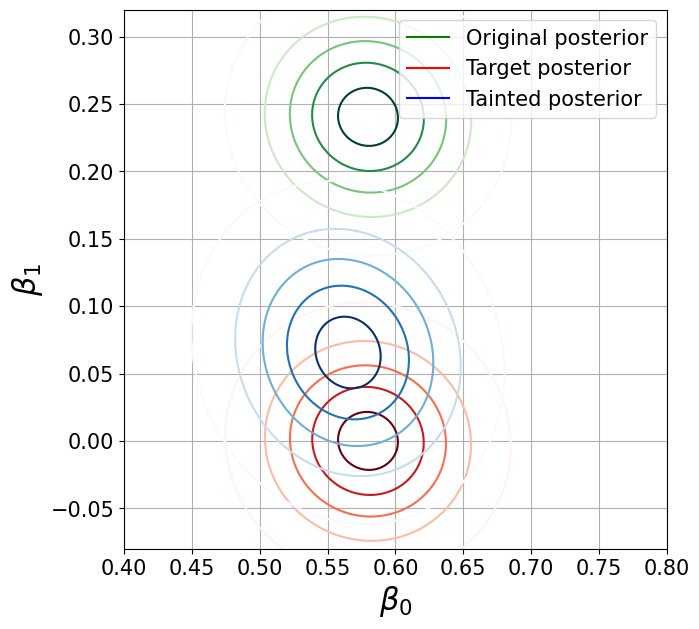}
    \subcaption{}\label{subfig:posterior_comp_nig}\par 
    \end{multicols}
\caption{(a) KL divergence vs $B$ for several heuristics. (b) $\mathbb{E}_{w^*}\left[\beta_1\right]$ vs $B$. (c) Datapoints removed (in red) and duplicated (in green) for $B=20$ under the 2O-R2 heuristic. (d) Original, induced and adversarial marginal posteriors over $(\beta_0, \beta_1)$ for $B=20$ using the 2O-R2 heuristic.}
\label{fig:attacks_nig}
\end{figure*}

The adversarial posterior is also specified as a normal-inverse-gamma distribution. The attacker's goal is to shift the posterior mean of $ \beta_1 $ towards zero while minimally affecting other aspects of the posterior inference, such as the MAP estimates of other parameters and the widths of credible intervals. To achieve this, we set the adversarial posterior parameters to match those of the original posterior, except that the second coordinate of the mean vector $ \mu_A $ is set to zero.



Although the KL divergence between two NIG distributions has a closed-form expression, permitting exact computation of the gradient and Hessian with respect to $ w $; we demonstrate that our attacks do not rely on this analytical information. Instead, we compute the gradient via sampling, highlighting the generality of our approach. We applied our heuristics to solve the attacker problem for $ L = 2 $, adjusting the attack intensity by varying $ B $ from 5 to 55. Each experiment was repeated 30 times. Figures \ref{subfig:heuristics_nig_KL} and \ref{subfig:heuristics_nig_mean} display the mean and two standard errors for the KL divergence and the mean of the induced posterior, respectively, as functions of $ B $. Additional experiments regarding the convergence speed of different heuristics are provided in Section 5.3 of the supplementary material.


The results indicate that the FGSM heuristic performs significantly worse than the other methods in terms of the objective function, particularly for $B > 30$, where the KL starts increasing with attack intensity. However, FGSM still shifts the posterior of $\beta_1$ in the desired direction. Unlike other heuristics, FGSM does not stop when the target mean of 0 is reached; instead, the mean continues to decrease.
The differences between the other heuristics become more pronounced for $B \geq 30$, where each alternative achieves a KL divergence around 0.1. Notably, the second-order heuristics outperform the first-order ones, emphasizing the value of curvature information. 
For $30 \leq B \leq 40$, the 2O-ISCD and 2O-R2 heuristics consistently deliver the best results. Interestingly, R2-based heuristics show an increase in KL divergence after $B=40$, a trend not observed with 2O-ISCD. As discussed in Section 5.2 of the supplementary material, this behavior likely derives from the rounding procedure.

Figures \ref{subfig:attacks_nig} and \ref{subfig:posterior_comp_nig} respectively illustrate the points removed and duplicated under the attack identified by the 2O-R2 heuristic with $B=20$, as well as the corresponding original, induced, and adversarial marginal posteriors over ($\beta_0$, $\beta_1$). Points with high leverage on the regression line are removed or duplicated to shift the slope toward the attacker's target. This suggests that outliers strongly influence attack efficacy, indicating that a noisy dataset is easier to manipulate; see Section 5.4 of the supplementary material for further discussion. Remarkably, manipulating only 20\% of the data points results in a significant overlap between the induced and adversarial posteriors, demonstrating the attack's efficiency and its ability to target only specified aspects of the posterior. If the attacker increases $B$ to 30, the adversarial and tainted means are almost identical without further posterior disruption (see Section 5.1 of supplementary material). 

A notable application of our framework is creating attacks aimed at manipulating posterior uncertainty. Experiments on this are provided in Section 5.5 of the supplementary material.

\subsection{Bayesian Linear Regression for Boston Housing Price Inference}

To evaluate our heuristics on real-world data, we use the Boston Housing Dataset \citep{Harrison1978Boston}. Herein, the defender performs inference on the parameters of a linear model to predict house prices. The training set comprises $n=404$ entries, each representing a house, with $d=13$ covariates describing their characteristics and a response variable, \textit{MEDV}, representing the median house price.

\begin{figure*}[ht!]
\begin{multicols}{5}
    \includegraphics[width=\linewidth]{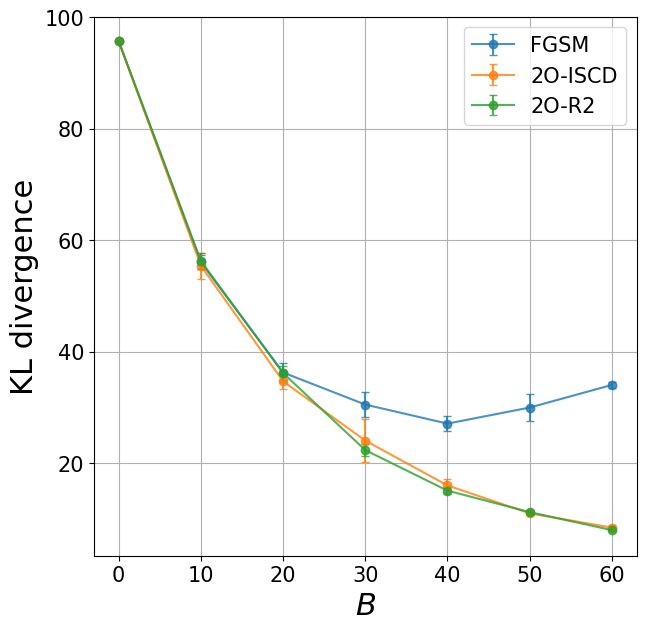}
    \subcaption{}\label{subfig:heuristics_hs_KL}\par 
    \includegraphics[width=\linewidth]{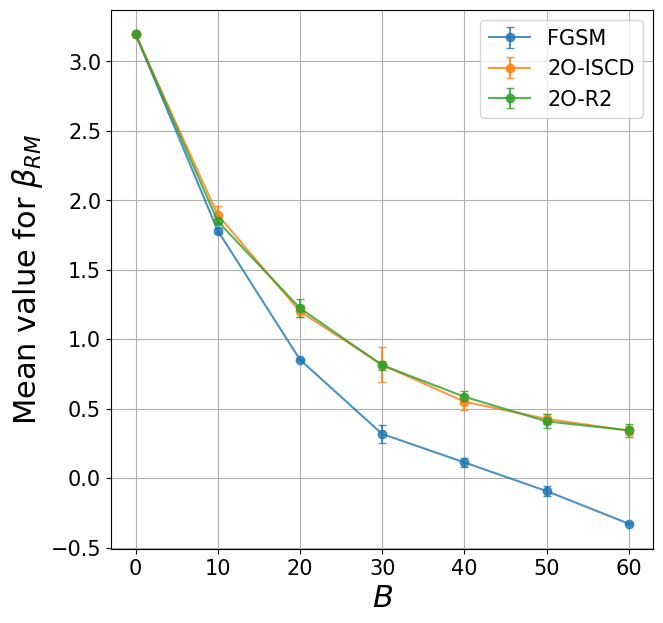}
    \subcaption{}\label{subfig:heuristics_hs_mean}\par 
    \includegraphics[width=\linewidth]{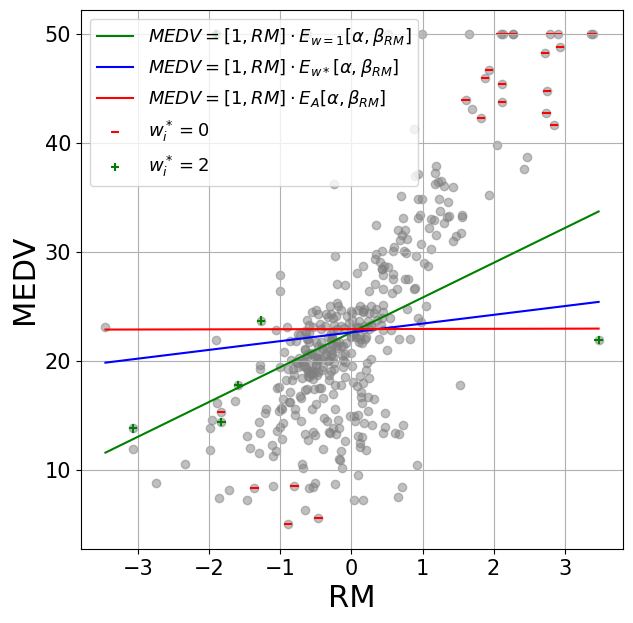}
    \subcaption{}\label{subfig:attacks_hs}\par 
    \includegraphics[width=\linewidth]{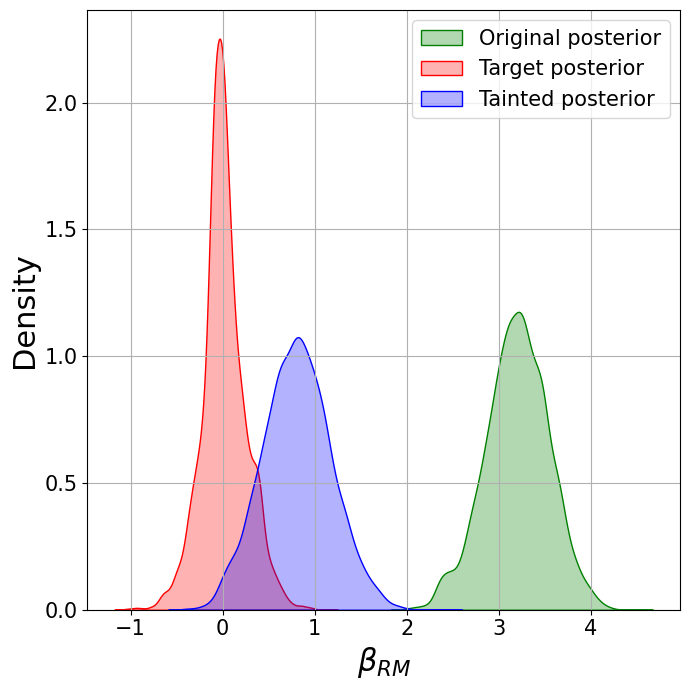}
    \subcaption{}\label{subfig:posterior_comp_hs}\par 
    \includegraphics[width=\linewidth]{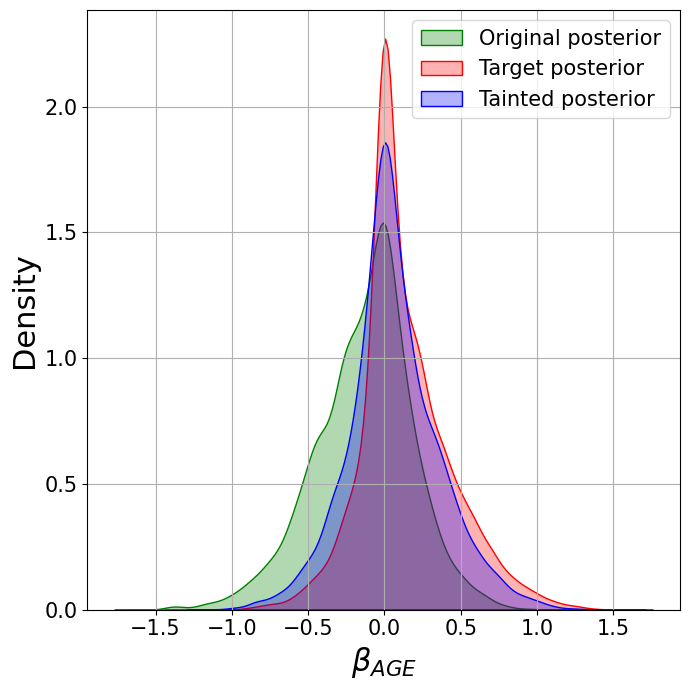}
    \subcaption{}\label{subfig:posterior_comp_hs_bis}\par 
    \end{multicols}
\caption{Attacks against the horseshoe model to steer the RM coefficient towards 0. (a) KL divergence between VI approximations. (b) $\mathbb{E}_{w^*}\left[\beta_{RM}\right]$ (c) Datapoints removed (in red) and duplicated (in green) for $B=30$ under the 2O-ISCD heuristic. (d) Original, induced and adversarial marginal posteriors over $\beta_{RM}$. (e) Original, induced and adversarial marginal posteriors over $\beta_{AGE}$.}
\label{fig:HS}
\end{figure*}

We assume the defender employs a sparsity-inducing horseshoe prior \citep{carvalho09aHorseshoe} on the regression parameters. The model includes an intercept $\alpha $, a vector of regression coefficients $ \beta \in \mathbb{R}^d $, a global shrinkage parameter $ \tau $, a vector of local shrinkage parameters $ \lambda \in \mathbb{R}^d $, and a noise level $ \sigma $. Full specification of likelihood and prior is provided in Section 6.1 of the supplementary material. The attacker's main objective is to influence the inference on $ \beta_{RM} $, the coefficient associated with the number of rooms in a house, by shifting it toward zero. Note that, in this case, there is no analytical expression for the posterior or the objective function, but our framework can still be utilized. To define the adversarial posterior, we first draw samples from the true posterior via MCMC, set the parameters $ \sigma_A $, $ \alpha_A $, and $ \beta_A $ equal to the means of the samples, and set $ \beta_{A, RM} = 0 $. A synthetic dataset $ \tilde{D} $ is then generated with these estimates, and $ \pi_A $ is defined as the posterior given $ \tilde{D} $, sampled via MCMC.



We performed attacks using FGSM, 2O-ISCD, and 2O-R2, varying the constraint $B$ across $\{10, 20, 30, 40, 50, 60\}$, and repeated each experiment five times. For each resulting weight vector $w^*$, we evaluated the attack's effectiveness using three metrics: the KL divergence between the tainted and adversarial posteriors, approximated via variational methods, as well as $\mathbb{E}_{w^*}[\beta_{\text{RM}}]$, and $\pi_{w^*}(\beta_{\text{RM}} < 0)$, calculated with MCMC samples. Averages plus/minus two standard errors of the first two metrics as functions of $B$ are shown in Figures \ref{subfig:heuristics_hs_KL} and \ref{subfig:heuristics_hs_mean}, with results for the third metric in Section 6.1 of the supplementary material.
As $B$ increases, both the KL divergence and the mean value of $\beta_{RM}$ decrease. Notably, the FGSM heuristic reduces the objective function at a slower rate compared to 2O-ISCD and 2O-R2. Beyond $B = 40$, FGSM becomes overly aggressive, overshooting the adversarial posterior mean. This results in the tainted posterior shifting downward, assigning a higher probability to $\beta_{RM} < 0$ than the other two heuristics. Based on KL divergence, 2O-ISCD and 2O-R2 perform better by avoiding such overshooting.



To visualize the impact of the attacks, Figure \ref{subfig:attacks_hs} shows the data points that were deleted and duplicated in the (\textit{RM}, \textit{MEDV}) plane under a poisoning attack with $B=30$ using the 2O-ISCD heuristic. The changes made align with the goal of reducing the inferred value of $\beta_{RM}$. Specifically, the removed points are either expensive houses with many rooms or cheap houses with few rooms. Additionally, one outlier, which has a medium \textit{MEDV} but a high \textit{RM}, is duplicated.

Figures \ref{subfig:posterior_comp_hs} and \ref{subfig:posterior_comp_hs_bis} display the original, target, and induced marginal posteriors for $\beta_{RM}$ and $\beta_{AGE}$ (i.e., coefficient for a house's age). Figure \ref{subfig:posterior_comp_hs} shows that the attack effectively shifts the posterior for $\beta_{RM}$ toward zero; manipulating approximately 7\% of the data reduces the posterior mean of $\beta_{RM}$ from 3.2 to 0.8. Alternatively, if 12\% of the data is corrupted, the posterior density of $ \beta_{RM}$ is highly concentrated near 0; see Section 6.1 supplementary material for details. 
Figure \ref{subfig:posterior_comp_hs_bis} demonstrates the surgical nature of the attack: whereas $\beta_{RM}$ is effectively targeted, the marginal posterior for $\beta_{AGE}$ remains largely unchanged, reflecting the attacker's intent to leave other parameters unaffected. Additional marginal posterior plots in Section 6.1 of the supplementary material further highlight the targeted nature of this attack. 


We also ran the attack using the NIG prior, with results presented in Section 6.1.2 of the supplementary material. Interestingly, in both cases, similar data points were replicated and removed. Thus, when ample data is available, the attacker may not need to know the true prior. By knowing only the likelihood, the attacker can craft high-quality attacks using an assumed prior, extending our method to a gray-box setting. Further results analyzing the effects of different priors on attack performance are provided in Section 6.1.3 of the supplementary material.

Finally, we would like to emphasize a key distinction between our attacks and other data deletion strategies, such as those proposed by \citet{broderick2020automatic}. Previous work is often more single-minded than our own. Authors remove a few data points to significantly alter inference on a specific parameter, regardless of the effects on other parameters. In contrast, our method targets the full joint posterior distribution. We demonstrated our methodology in a scenario where the objective was to shift the posterior for a specific parameter in a given direction while otherwise preserving it. If the attacker were allowed to modify the posteriors of other parameters, even more effective attacks could be designed. 

Moreover, in some instances, substantive effects on posterior inference can still be achieved through minimal data perturbation. Such effects are highly tailorable via the definition of the adversarial posterior and may induce dramatic real-world consequences.



\subsection{Case Study: Mexico Microcredit}

\begin{figure*}[h]
\centering
\begin{multicols}{3}
    \includegraphics[width=0.7\linewidth]{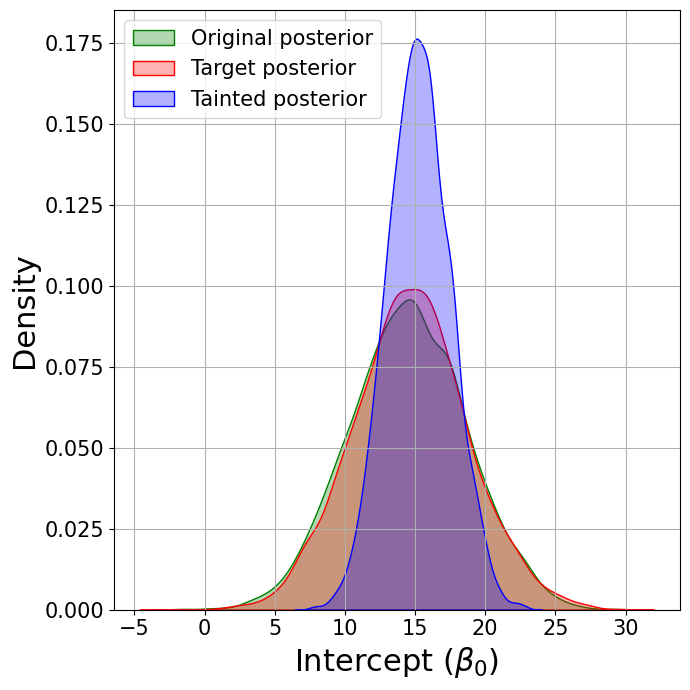}
    \subcaption{Posterior for $\beta_0$}\par
    \includegraphics[width=0.7\linewidth]{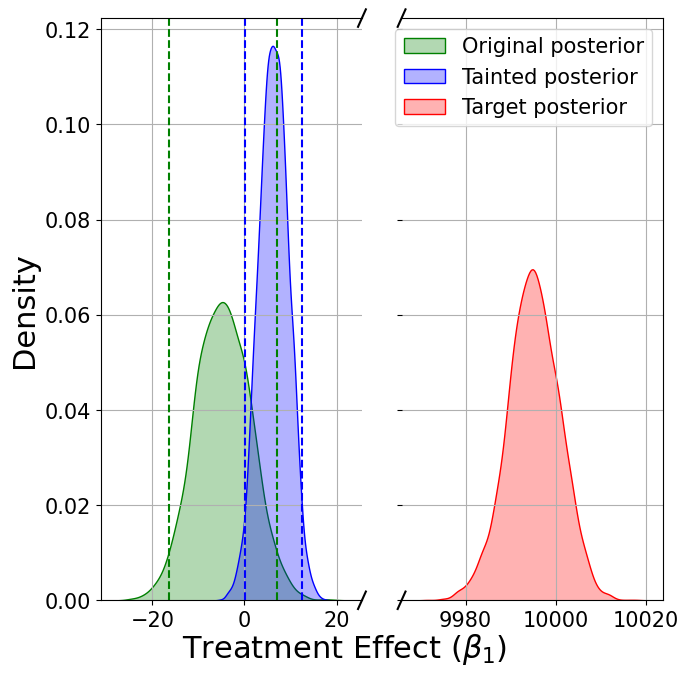}
    \subcaption{Posterior for $\beta_1$}\par
    \includegraphics[width=0.7\linewidth]{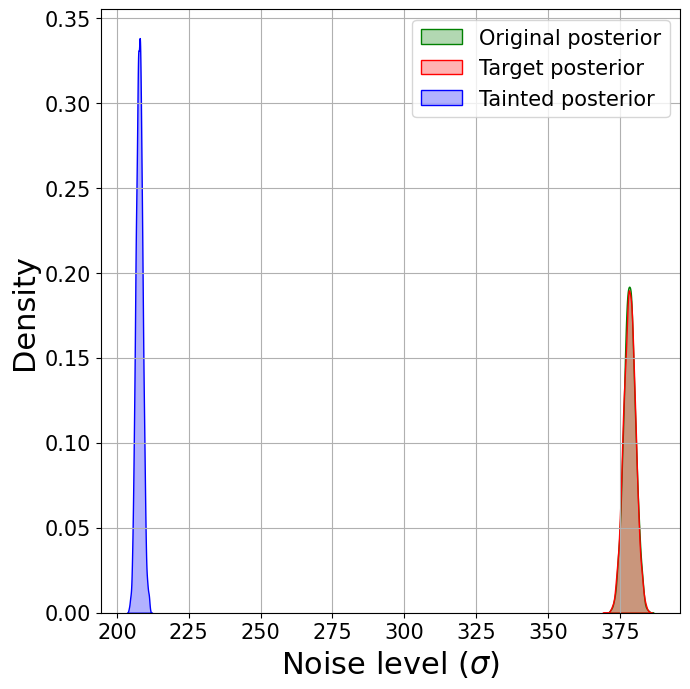}
    \subcaption{Posterior for $\sigma$}\par
\end{multicols}
\caption{Attack on the microcredit linear regression for $B=20$ and $L=2$.}
\label{fig:microcredit_B=20_lambda=10000}
\end{figure*}

In this section, we illustrate the potential policy implications of our poisoning attacks by examining a randomized controlled trial on microcredit in Mexico \citep{angelucci2015microcredit}. In this trial, 16{,}560 businesses were randomly assigned to a microcredit treatment ($x_i=1$) or control ($x_i=0$), with the goal of assessing how the treatment influenced business profit $y_i$. Following \citet{nguyen2024sensitivity}, we fit  
\begin{eqnarray*}
    y_i \;&\sim&\; \mathcal{N}(\,\beta_0 + \beta_1\,x_i,\;\sigma^2),\\
\beta_0,\,\beta_1,\, \log(\sigma)\,&\sim&\, t(3,\,0,\,1000),
\end{eqnarray*}
where $\beta_1$ represents the average treatment effect (ATE) of microcredit. Under untainted data, the posterior mean of $\beta_1$ is $-4.71$ and its 95\% credible interval slightly crosses zero, suggesting no reliable boost to profits due to microcredits, or even a mild negative impact. This naturally leads to a policy decision \emph{against} expanding microcredit.

We carried out a poisoning attack on this dataset using 2O-ISCD, allowing a maximum of $B=20$ manipulations, and specifying an adversarial posterior strongly favoring a positive $\beta_1$. Figure~\ref{fig:microcredit_B=20_lambda=10000} shows the original, target, and induced posteriors for all model parameters. Remarkably, manipulating only 20 observations (just 0.12 \% of the data) shifts the posterior so that $\beta_1$ now has a mean of $6.28$ with a 95\% credible interval $[0.02,\ 12.43]$.  Faced with this inference, a policymaker would likely conclude that microcredit substantially increases profits and thus opt for a program expansion, reversing the original conclusion. 

Note that the same attack also significantly alters the posterior for $\sigma$, illustrating how small but carefully chosen data modifications can affect multiple parameters simultaneously. Additional experiments in Section 6.2 of the supplementary material examine how different adversarial posteriors and attack budgets $B$ can further refine these shifts. Overall, these results highlight (1) the impact poisoning can have on decisions driven by Bayesian estimates and (2) the scalability of our attack methods to large-sample settings.

\subsection{Bayesian Logistic Regression for Spam Classification}

Additional experimentation was performed on the spam classification data set of \citet{sms_spam_collection_228}. 
Detailed results and discussion are included in Section 6.3 of the supplemental material.



\section{CONCLUSIONS}

This work introduced novel poisoning attacks on generic Bayesian inference, involving the deletion and replication of data samples to steer the posterior toward an adversarial target. Our attacks are applicable to any Bayesian model, provided MCMC sampling is possible. We formalized the problem as an integer program with an intractable objective function and proposed multiple heuristics for its solution. These methods were effective in various Bayesian settings, demonstrating that minimal data perturbations can significantly shift the posterior. Besides generating effective attacks, our framework can also assess Bayesian inference robustness against worst-case data manipulation, providing a clear metric for evaluating the impact of data perturbations on posterior estimates.

Regarding possible defenses to the proposed attacks, one might argue that removing duplicates is a simple defense. However, real-world datasets often include genuine duplicates, risking the loss of valid data. Even if duplicates are removed, our framework adapts seamlessly to data deletions alone, an attack far harder to detect; highlighting the need for more research on robust defenses against adversarial manipulation in Bayesian inference.

Another promising avenue for future research is adapting our work for Bayesian case influence analysis. By setting the adversarial posterior to the prior, one can identify data points whose deletion causes a minimal prior-to-posterior update. Moreover, while our attacks allow precise poisoning, they also require a fully specified adversarial posterior. Future work may thus study how to design adversarial posteriors to achieve specific goals. Ideas such as entropic tilting might be relevant in this context. Additionally, extending our approach to larger, more complex Bayesian models (e.g., hierarchical, spatial, spatio-temporal) is crucial. While scalability is challenging with MCMC, integrating variational inference methods could enable applications to larger models efficiently.

\subsubsection*{Acknowledgements}

This research was funded, in whole or in part, by the French National Research Agency (ANR) under the project OWL "ANR-23-IAS3-0003-01".


\bibliography{biblio}

\clearpage

\clearpage

\thispagestyle{empty}
\onecolumn

\setcounter{section}{0}

\aistatstitle{Supplementary Material for: Poisoning Bayesian Inference via Data Deletion and Replication}

\section{PROOF OF PROPOSITION 2}

Herein, we derive convergence conditions for solving the continuous relaxation within the SGD-R2 method. \citet{nemirovski2009robust} and \citet{shapiro2021lectures} show that, when solving

$$
\min_{w \in \mathcal{W}'} \left\{f(w) = \mathbb{E}\left[F(w,\theta) \right]\right\}
$$

\noindent such that 
\begin{enumerate}
    \item $\mathcal{W}' \subset \mathbb{R}^n$ is a non-empty, closed, bounded, and convex set, and
    \item$f$ is differentiable and $c$-strongly convex on $\mathcal{W}'$
\end{enumerate}

\noindent projected SGD converges to an optimal solution when the following are satisfied:

\begin{enumerate}
    \setcounter{enumi}{2}
    \item it is possible to sample independent $\hat{g}(w)$ such that $\mathbb{E}[\hat{g}(w)] = \nabla f(w)$,
    \item $\mathbb{E}\left[\lVert \hat{g}(w)\rVert^2 \right] \le M^2$, where $\lVert\cdot\rVert$ is the Euclidean distance, and
    \item $\gamma_t =\frac{1}{ct}$.
\end{enumerate}

\noindent That is, denoting $\tilde{w}$ as the optimal solution and $w^0$ as the deterministic starting point, for any iteration $t$

$$\mathbb{E}\left[\lVert w^t - \tilde{w} \rVert^2 \right]\le \max \left\{ \frac{M^2}{c^2t}, \frac{\lVert w^0 - \tilde{w} \rVert^2}{t}\right\}.
$$ 

With regard to the continuous relaxation within SGD-R2, we consider each of these conditions in turn and relate them to those provided in proposition 2.  

\begin{itemize}
    \item \textbf{Condition 1}: By construction, $\mathcal{W}$ is non-empty, closed, bounded and convex. 
    \item \textbf{Condition 2}: Since
        \begin{eqnarray*}
    g(w) = \nabla_w \big(\log Z(w) - w^\top \cdot \mathbb{E}_{\pi_A(\theta)}\left[ f_X(\theta)\right] \big) =
    & \mathbb{E}_{\pi_w(\theta \vert \mathbf{X})}\left[ f_X(\theta)\right] - \mathbb{E}_{\pi_A(\theta)}\left[ f_X(\theta)\right]
        \end{eqnarray*}
    the objective function is differentiable on $\mathcal{W}$ if the above difference exists and is finite for all $w \in \mathcal{W}$. Moreover, because
    \begin{eqnarray} \label{eq:cov2}
    \nabla_w \left( \nabla_w \log Z(w) \right) = \text{Cov}_{\pi_w(\theta | \mathbf{X})}\left( f_X(\theta), f_X(\theta)\right)
    \end{eqnarray}
    the SGD-R2 objective function is $c$-strongly convex on $\mathcal{W}$ if the eigenvalues of $\text{Cov}_{\pi_w(\theta | \mathbf{X})}\left( f_X(\theta), f_X(\theta)\right)$ are lower-bounded by some $c>0$ $\forall w \in \mathcal{W}$.
    \item \textbf{Condition 3}: Independently sample $\left( \theta_i \right)_{i=1}^P \sim \pi_w(\theta | \mathbf{X})$ and $\left( \theta_j \right)_{j=1}^Q \sim \pi_A(\theta)$, forming the unbiased estimate 
    $$\hat{g}(w) = \frac{1}{P} \sum_{i=1}^{P} f_X(\theta_i) - \frac{1}{Q} \sum_{j=1}^{Q} f_X(\theta_j).$$ 
    \item \textbf{Condition 4}: Denoting  $\mu_w= \mathbb{E}_{\pi_w(\theta|\mathbf{X})}[f_X(\theta)]$ and $\eta= \mathbb{E}_{\pi_A(\theta)}[f_X(\theta)]$, notice that
 
    \begin{align*}
    \mathbb{E}_{\substack{\theta_i \sim \pi_w(\theta|\mathbf{X}) \\
    \theta_j \sim \pi_A(\theta)}}[\lVert  \hat{g}(w) \rVert^2] &= \mathbb{E}_{\substack{\theta_i \sim \pi_w(\theta|\mathbf{X}) \\
    \theta_j \sim \pi_A(\theta)}} \left[ \left\lVert \frac{1}{P} \sum_{i=1}^{P} f_X(\theta_i) - \frac{1}{Q} \sum_{j=1}^{Q} f_X(\theta_j) \right\rVert^2 \right] \\
     & = \mathbb{E}_{\pi_{w}(\theta|\mathbf{X})}\left[\left\lVert \frac{1}{P} \sum_{i=1}^{P}  f_X(\theta_i) \right\rVert^2 \right] - 2 \mu^\top_w \eta + \mathbb{E}_{\pi_{A}(\theta)}\left[\left\lVert \frac{1}{Q} \sum_{j=1}^{Q} f_X(\theta_j) \right\rVert^2 \right].
    \end{align*}

\noindent Letting $\hat{\mu}_w = \frac{1}{P} \sum_{i=1}^{P} f_X(\theta_i)$ and $\hat{\eta} = \frac{1}{Q} \sum_{j=1}^{Q} f_X(\theta_j)$ then $\mathbb{E}_{\pi_w(\theta|\mathbf{X})} [\hat{\mu}_w] = \mu_w$ and $\mathbb{E}_{\pi_A(\theta)} [\hat{\eta}] = \eta$ 

\begin{align*}
\mathbb{E}_{\pi_{w}(\theta|\mathbf{X})}\left[\left\lVert \frac{1}{P} \sum_{i=1}^{P}  f_X(\theta_i) \right\rVert^2 \right] &= tr(\text{Cov}_{\pi_w(\theta|\mathbf{X})}(\hat{\mu}_w, \hat{\mu}_w)) + \lVert\mu_w \rVert^2 \\
\mathbb{E}_{\pi_{A}(\theta)}\left[\left\lVert \frac{1}{Q} \sum_{j=1}^{Q} f_X(\theta_j) \right\rVert^2 \right] &=tr\left(\text{Cov}_{\pi_A(\theta)}(\hat{\eta}, \hat{\eta})\right) + \lVert\eta \rVert^2 
\end{align*}

Thus, if the eigenvalues of $\text{Cov}_{\pi_w(\theta|\mathbf{X})}(\hat{\mu}_w, \hat{\mu}_w)$ and $\text{Cov}_{\pi_A(\theta)}(\hat{\eta}, \hat{\eta})$ are bounded above by $a$ and $b$, respectively,

\begin{align*}
    \mathbb{E}_{\substack{\theta_i \sim \pi_w(\theta|\mathbf{X}) \\
    \theta_j \sim \pi_A(\theta)}}[\lVert  \hat{g}(w) \rVert^2] &= tr(\text{Cov}_{\pi_w(\theta|\mathbf{X})}(\hat{\mu}_w, \hat{\mu}_w)) + tr\left(\text{Cov}_{\pi_A(\theta)}(\hat{\eta}, \hat{\eta})\right) +\lVert \mu_w - \eta \rVert^2 \\
    & \le n(a+b) + \lVert \mu_w - \eta \rVert^2 \\
    & \le n(a+b) + \max_{w \in \mathcal{W}} \ \lVert \mu_w - \eta \rVert^2  \\
    & = n(a+b) + \max_{w \in \mathcal{W}} \ \lVert g(w) \rVert^2  \\
    & = M^2
\end{align*}

    \item \textbf{Condition 5}: Satisfied by construction.
\end{itemize}

\newpage

\section{PROOF OF PROPOSITION 3}

In this section, we derive the expression for the gradient of the reverse KL divergence with respect to the weights $w$. Notice that the induced posterior for weights $w$ can be written as

$$
\pi_w(\theta \vert \mathbf{X}) =  \exp \left[ w^\top \cdot f_X(\theta) - \log Z(w) \right] \cdot \pi(\theta).
$$

Thus,

\begin{align*}
    \text{KL}(\pi_w(\theta | \mathbf{X}) \parallel \pi_A(\theta) ) &= \mathbb{E}_{\pi_w(\theta \vert \mathbf{X})}\left[\log \exp \left[ w^\top \cdot f_X(\theta) - \log Z(w) \right]\right] + \mathbb{E}_{\pi_w(\theta \vert \mathbf{X})}\left[\log \frac{\pi(\theta)} {\pi_A(\theta)}\right], \\
    &= w^\top \mathbb{E}_{\pi_w(\theta \vert \mathbf{X})}[ f_X(\theta)] - \mathbb{E}_{\pi_w(\theta \vert \mathbf{X})}[\log Z(w)] + \mathbb{E}_{\pi_w(\theta \vert \mathbf{X})}\left[\log \frac{\pi(\theta)} {\pi_A(\theta)}\right],\\
    &= w^\top \nabla_w \log Z(w) - \log Z(w) + \mathbb{E}_{\pi_w(\theta \vert \mathbf{X})}\left[\log \frac{\pi(\theta)} {\pi_A(\theta)}\right], 
\end{align*}

\noindent  and, by the product rule 

\begin{align}
    \nabla_w \text{KL} &= w^\top \nabla^2_w \log Z(w) + \nabla_w \log Z(w) - \nabla_w \log Z(w) + \nabla_w \mathbb{E}_{\pi_w(\theta \vert \mathbf{X})}\left[\log \frac{\pi(\theta)} {\pi_A(\theta)}\right]   \notag\\
    &= w^\top \nabla^2_w \log Z(w) + \nabla_w \mathbb{E}_{\pi_w(\theta \vert \mathbf{X})}\left[\log \frac{\pi(\theta)} {\pi_A(\theta)}\right],
\end{align}

\noindent Assuming the associated regularity conditions, we can interchange integration and differentiation such that 

\begin{align}
    \nabla_w \text{KL} &= w^\top  \nabla^2_w \log Z(w) + \int \log \left(\frac{\pi( \theta  )}{ \pi_A( \theta  )} \right) \nabla_w \pi_w(\theta | \mathbf{X}) \dd \theta.
\end{align}

\noindent By the identity $\nabla_w \pi_w(\theta | \mathbf{X}) =\pi_w(\theta | \mathbf{X}) \nabla_w \log (\pi_w(\theta | \mathbf{X} ))$, we can write

\begin{align*}
    \nabla_w \text{KL}
    &= w^\top  \nabla^2_w \log Z(w) + \int \log \left(\frac{\pi( \theta  )}{ \pi_A( \theta  )} \right) \pi_w(\theta | \mathbf{X}) \nabla_w \log \pi_w(\theta | \mathbf{X}) \dd \theta \\
    &= w^\top  \nabla^2_w \log Z(w) + \mathbb{E}_{\pi_w(\theta \vert \mathbf{X})} \left[\log \left(\frac{\pi( \theta  )}{ \pi_A( \theta  )} \right) \nabla_w \log \pi_w(\theta | \mathbf{X}) \right]
\end{align*}

\noindent As previously seen in deriving the KL divergence formula, the log of our posterior is

\begin{align*}
   \log \left[\exp \left[ w^\top \cdot f_X(\theta) - \log Z(w) \right] \cdot \pi(\theta) \right] &=  w^\top \cdot f_X(\theta) - \log Z(w) + \log \pi(\theta) 
\end{align*}

\noindent implying

\begin{align*}
    \nabla_w \text{KL}
    &= w^\top  \nabla^2_w \log Z(w) + \mathbb{E}_{\pi_w(\theta \vert \mathbf{X})} \left[\log \left(\frac{\pi( \theta  )}{ \pi_A( \theta  )} \right) \left( f_X(\theta)  - \nabla_w  \log Z(w) \right) \right]
\end{align*}

\noindent Recalling that $\mathbb E_{\pi_w(\theta \vert \mathbf{X})} [f_X(\theta)] = \nabla_w \log Z(w)$, properties of the expectation operator yield the following:

\begin{align*}
    \nabla_w \text{KL}
    &= w^\top  \nabla^2_w \log Z(w) + \mathbb{E}_{\pi_w(\theta \vert \mathbf{X})} \left[\log \left(\frac{\pi( \theta  )}{ \pi_A( \theta  )} \right) \left( f_X(\theta)  - \mathbb{E}_{\pi_w(\theta \vert \mathbf{X})} [f_X(\theta)]  \right) \right] \\
    & = w^\top  \nabla^2_w \log Z(w) + \mathbb{E}_{\pi_w(\theta \vert \mathbf{X})} \left[\log \left(\frac{\pi( \theta  )}{ \pi_A( \theta  )} \right) \left( f_X(\theta) \right) \right] - \mathbb{E}_{\pi_w(\theta \vert \mathbf{X})} \left[ \mathbb{E}_{\pi_w(\theta \vert \mathbf{X})} [f_X(\theta)] 
    \log \left(\frac{\pi( \theta  )}{ \pi_A( \theta  )}\right)
    \right] \\
    &= w^\top  \nabla^2_w \log Z(w) + \mathbb{E}_{\pi_w(\theta \vert \mathbf{X})} \left[\log \left(\frac{\pi( \theta  )}{ \pi_A( \theta  )} \right) \left( f_X(\theta) \right) \right] -   \mathbb{E}_{\pi_w(\theta \vert \mathbf{X})} \Big[f_X(\theta)\Big] \mathbb{E}_{\pi_w(\theta \vert \mathbf{X})}\left[
    \log \left(\frac{\pi( \theta  )}{ \pi_A( \theta  )}\right) \right]
\end{align*}

Finally, because in the exponential family the Hessian of the log-partition function is the covariance matrix of the sufficient statistics, we have

\begin{align*}
    \nabla_w \text{KL}
    &= w^\top \text{Cov}_{\pi_w(\theta \vert \mathbf{X})} \left[f_X(\theta), f_X(\theta) \right] \ + \mathbb{E}_{\pi_w(\theta \vert \mathbf{X})} \left[\log \left(\frac{\pi( \theta  )}{ \pi_A( \theta  )} \right) \left( f_X(\theta) \right) \right] \\& \qquad  \qquad \qquad -   \mathbb{E}_{\pi_w(\theta \vert \mathbf{X})} \Big[f_X(\theta)\Big] \mathbb{E}_{\pi_w(\theta \vert \mathbf{X})}\left[
    \log \left(\frac{\pi( \theta  )}{ \pi_A( \theta  )}\right) \right]\\
    & =  \text{Cov}_{\pi_w(\theta \vert \mathbf{X})} \left[f_X(\theta), w^\top f_X(\theta) \right] \ + \mathbb{E}_{\pi_w(\theta \vert \mathbf{X})} \left[\log \left(\frac{\pi( \theta  )}{ \pi_A( \theta  )} \right) \left( f_X(\theta) \right) \right] \\ & \qquad \qquad \qquad-   \mathbb{E}_{\pi_w(\theta \vert \mathbf{X})} \Big[f_X(\theta)\Big] \mathbb{E}_{\pi_w(\theta \vert \mathbf{X})}\left[
    \log \left(\frac{\pi( \theta  )}{ \pi_A( \theta  )}\right) \right]
\end{align*}

\newpage

\section{PROOF OF PROPOSITION 4}

In this section, we justify the procedure used in final step of the rounded relaxation heuristics.
It consists of solving the following constrained rounding problem:
\begin{eqnarray*}
\min_{w' \in \mathbb{Z}^n} && \Vert w' - w\Vert_2^2 \\
\text{s.t.} && \|w' - \mathbf{1}\|_1 \leq B \nonumber \\
&& \|w' \|_\infty \leq L  \nonumber \\
&& w'  \succeq 0 \nonumber
\end{eqnarray*}
where $w \in \mathcal{W}$ is a solution to the continuous relaxation of the attacker's problem.

The first subsection proves that the procedure given in proposition 4 solves the rounding problem when the last two constraints are ignored.
The second subsection proves that adding these constraints does not change the set of minimizers as the continuous weight vector $w$ already satisfies the 3 constraints.

\subsection{Constrained rounding procedure}

Set aside the last two constraints of the rounding problem and consider the set of integer weight vectors $\mathbb{Z}^n \cap \mathcal{B}_1(\mathbf{1}, B)$, where $\mathcal{B}_1(\mathbf{1}, B)$ is the ball of center $\mathbf{1}$ and radius $B$ for the $L^1$ norm. Likewise, define a variable $\alpha \in \mathbb{Z}^n$ which is related to $w'$ through the vectors $\Delta \in \mathbb{R}^n$ and $\varepsilon \in \mathbb{R}^n$ by the following equations:
\begin{equation*}
    \begin{aligned}
        \Delta_i &= |w_i - 1|, \\
        \varepsilon_i &= \Delta_i - \floor{\Delta_i}, \\
        w'_i &= 1 + \text{sign}(w_i - 1)(\lfloor \Delta_i \rfloor + \alpha_i), \\
    \end{aligned}
\end{equation*}
where the convention $\text{sign}(0) = 1$ is used to ensure a bijection between $w'$ and $\alpha$. 


Using this change of variables, it can be observed that
\begin{equation*} 
    \begin{aligned}
        \left\Vert w' - \mathbf{1} \right\Vert_1 &= \sum_i |\floor{\Delta_i} + \alpha_i|, \\
        \left\Vert w' - w \right\Vert_2^2 &= \sum_i \big(1 + \text{sign}(w_i - 1)(\floor{\Delta_i} + \alpha_i) - w_i)\big)^2, \\
        &= \sum_i \big(\text{sign}(w_i - 1)(\floor{\Delta_i} + \alpha_i) - \text{sign}(w_i - 1)\Delta_i)\big)^2, \\
        &= \sum_i (\varepsilon_i - \alpha_i)^2, \\
    \end{aligned}
\end{equation*}

\noindent and the rounding problem can be reformulated as 
\begin{eqnarray}
\label{eq:problem_alpha}
\min_{\alpha \in \mathbb{Z}^n} &&\Vert \varepsilon - \alpha \Vert_2^2 \\
\text{s.t.} && \sum_i |\floor{\Delta_i} + \alpha_i| \leq B \nonumber.
\end{eqnarray}

Moreover, note that, for any feasible $\alpha \in \mathbb{Z}^n$ such that $\exists i_0, \alpha_{i_0} \notin \{0, 1\}$, there exists some $\tilde{\alpha} \in \mathbb{Z}^n$ having $\left\Vert \varepsilon - \tilde{\alpha} \right\Vert_2^2
< \left\Vert \varepsilon - \alpha \right\Vert_2^2$. Namely, it is a trivial task to verify that the following four rules yield feasible $\tilde{\alpha}$ that are closer to $\varepsilon$. 

\begin{enumerate}
    \item If $\exists \alpha_{i_0} > 1$, then $\forall i \in {1,..., n}$ choose 
    \begin{equation}   
        \tilde{\alpha}_i = \left\{
            \begin{array}{ll}
                \alpha_i & \mbox{if } i \neq i_0 \\
                1 & \mbox{if } i = i_0
            \end{array}
        \right. \nonumber
    \end{equation}
    
    \item If $\exists \alpha_{i_0} < -\floor{\Delta_{i_0}}$, then choose 
    \begin{equation}   
        \tilde{\alpha}_i = \left\{
            \begin{array}{ll}
                \alpha_i & \mbox{if } i \neq i_0 \\
                -\floor{\Delta_i} & \mbox{if } i = i_0
            \end{array}
        \right. \nonumber
    \end{equation}
    
    \item If $-\floor{\Delta_i} \leq \alpha_{i_0} < 0$ and $\sum_i |\floor{\Delta_i} + \alpha_i| \leq B - 1$,  choose 
    \begin{equation}   
        \tilde{\alpha}_i = \left\{
            \begin{array}{ll}
                \alpha_i & \mbox{if } i \neq i_0 \\
                \alpha_i + 1 & \mbox{if } i = i_0
            \end{array}
        \right. \nonumber
    \end{equation}
    
    \item If $-\floor{\Delta_i} \leq \alpha_{i_0} < 0$, and $\sum_i |\floor{\Delta_i} + \alpha_i| = B$, then $\exists i_1, \alpha_{i_1} \geq 1$, choose
    \begin{equation}   
        \tilde{\alpha}_i = \left\{
            \begin{array}{ll}
                \alpha_i & \mbox{if } i \notin \{i_0, i_1\} \\
                \alpha_i + 1 & \mbox{if } i = i_0\\
                \alpha_i - 1 & \mbox{if } i = i_1
            \end{array}
        \right. \nonumber
    \end{equation}
    
\end{enumerate}

The iterated application of these rules implies that the feasible region of problem \eqref{eq:problem_alpha} is a subset of $\{0,1\}^n$, thereby ensuring that for all $i$, $\alpha_i^2 = \alpha_i$ and $|\floor{\Delta_i} + \alpha_i| = \floor{\Delta_i} + \alpha_i$. This allows the constraint and the objective function, respectively, to be expressed as affine functions of $\alpha$:
\begin{equation*} 
    \begin{aligned}
        \sum_i |\floor{\Delta_i} + \alpha_i| &= \left(\sum_i \floor{\Delta_i}\right) + \left(\sum_i \alpha_i\right), \\
        \Vert \varepsilon - \alpha \Vert_2^2 &= \Vert\varepsilon\Vert^2 + \sum_i \alpha_i(1 - 2\varepsilon_i). \\
    \end{aligned}
\end{equation*}
The constraint is now equivalent to $\Vert \alpha \Vert_1 \leq N_{\text{max}} = B - \sum_i \floor{\Delta_i}$ and, since $w \in \mathcal{W}$, $\sum_i \floor{\Delta_i} \leq \left\Vert w - \mathbf{1} \right\Vert_1 \leq B$ and $N_{\text{max}} \geq 0$, one can observe that $\alpha = 0$ is always feasible. 
The problem therefore reduces to the selecting up to $N_{\text{max}}$ indices for which $\alpha_i = 1$ and the remaining are set to $\alpha_i = 0$.

By construction, it is apparent that assigning $\alpha_i=1$ reduces the objective function only when $\varepsilon_i > \frac {1}{2}$. Moreover, by linearity of the objective function and the binary nature of the decision variables, the optimal selection of $\alpha_i$ can be found by selecting the $N_{max}$ indices having the largest $\varepsilon_i$-values exceeding $\frac{1}{2}$. Such is the procedure outlined in proposition 4. Notice that the solution of the rounding problem might not be unique in two cases: if $2\varepsilon_i - 1 = 0$ for some $i$, or in case of equality among the $\varepsilon_i$.

\subsection{Equivalence of the two problems}

As $w \in \mathcal{W}$, any minimizer $w^* \in \mathbb{Z}^n$ of the rounding problem in $\mathcal{B}_1(\mathbf{1}, B)$ is also in $\mathcal{W}$, because it satisfies the non-negativity and $L^\infty$-norm constraints.
Indeed, if we consider a solution $\alpha^* \in \{0, 1\}^n$ and its associated $w^* \in \mathbb{Z}^n$ found as described in the previous section, for any $i$, we have:
\begin{enumerate}
    \item if $\varepsilon_i < \frac{1}{2}$, then $\alpha^*_i = 0$ and $|w^*_i - 1| = \floor{\Delta_i} \leq \Delta_i \leq \Delta_i + \frac{1}{2}$, and
    \item if $\varepsilon_i \geq \frac{1}{2}$, then $|w^*_i - 1| = \floor{\Delta_i} + \alpha^*_i \leq \Delta_i + \alpha^*_i - \frac{1}{2} \leq \Delta_i + \frac{1}{2}$.
\end{enumerate}
Therefore, we always have $|w^*_i - 1| \leq \Delta_i + \frac{1}{2} = |w_i - 1| + \frac{1}{2}$. Noticing that $w_i \in [0, L]$, it must be that $w^*_i \in \left[-\frac{1}{2}, L + \frac{1}{2}\right]$ and, since $w^* \in \mathbb{Z}^n$, this implies $w^*_i \in [0, L]$. Taken collectively, this shows that

\begin{equation*}
    \argmin_{w' \in \mathbb{Z}^n \cap \mathcal{B}_1(\mathbf{1}, B)} \Vert w' - w\Vert_2^2 
    \subseteq \mathbb{Z}^n \cap \mathcal{W}.
\end{equation*}

\noindent and, since $\mathbb{Z}^n \cap \mathcal{W} \subseteq \mathbb{Z}^n \cap \mathcal{B}_1(\mathbf{1}, B)$, it is equivalent to consider the problem in $\mathcal{W}$ or $\mathcal{B}_1(\mathbf{1}, B)$, i.e with or without the non-negativity and $L^\infty$-norm constraints:

\begin{equation*}
    \argmin_{w' \in \mathbb{Z}^n \cap \mathcal{B}_1(\mathbf{1}, B)} \Vert w' - w\Vert_2^2 
    = 
    \argmin_{w' \in \mathbb{Z}^n \cap \mathcal{W}} \Vert w' - w\Vert_2^2.
\end{equation*}

\newpage

\section{WEIGHTED POSTERIOR FOR BAYESIAN LINEAR REGRESSION WITH NORMAL-INVERSE-GAMMA PRIOR}

This section provides the derivation of the weighted posterior parameters for Bayesian linear regression with a Normal-Inverse-Gamma prior, based on a perturbed dataset. The perturbation is represented by the weight vector $w \in \mathbb{R}^n$, or equivalently, the diagonal matrix $W = \text{diag}(w) \in \mathbb{R}^{n \times n}$.

The probabilistic model we are assuming is
\begin{equation*}
    \begin{aligned}
    \sigma^2 &\sim \text{Inv-Gamma}(a_0, b_0), \\
    \beta \,|\, \sigma^2 &\sim \mathcal{N}(\mu_0, \sigma^2 \Lambda_0^{-1}), \\
    Y_i \,|\, X_i, \beta, \sigma^2 &\sim \mathcal{N}( X_i \beta, \sigma^2),
\end{aligned}
\end{equation*}

where $\sigma^2 \sim \text{Inv-Gamma}(a_0, b_0)$ is equivalent to $\frac{1}{\sigma^2} \sim \text{Gamma}(a_0, b_0)$.
The pdf of $\sigma^2=s$ is given by 
\begin{equation*}
    \quad f(s)=\frac{b_0^{a_0}}{\Gamma(\alpha)} s^{-a_0-1} e^{-\frac{b_0}{s}}, \ \forall s> 0.
\end{equation*}

The joint NIG prior over the parameters is equal to
\begin{align*}
\pi\left(\beta, \sigma^{2}\right)  
&= \text{Inv-Gamma}\left(\sigma^{2} \mid a_{0}, b_{0}\right) \cdot \mathcal{N}\left(\beta \mid \mu_{0}, \sigma^{2} \Lambda_0^{-1}\right) \\
&= \frac{b_0^{a_0}}{\Gamma\left(a_{0}\right)}
\left(\frac{1}{\sigma^{2}}\right)^{a_{0}+1} e^{-\frac{b_{0}}{\sigma^{2}}}
\frac{\left| \Lambda_0\right|^{\frac{1}{2}}}{(2 \pi \sigma^{2})^{\frac{d}{2}}} \exp\left(-\frac{1}{2 \sigma^{2}} (\beta - \mu_{0})^{\top}\Lambda_0(\beta - \mu_{0})\right).
\end{align*}

If we define $f_{X,y}(\beta, \sigma^{2})$ as the vector of log-likelihood contributions for each observation in the unperturbed dataset, then the likelihood of the weighted data, given the parameters $\beta$ and $\sigma^{2}$, can be expressed as:
\begin{align*}
\exp\left(w^{\top} f_{X,y}(\beta, \sigma^{2})\right) 
&= \prod_{i=1}^n \left(\frac{1}{\sqrt{2\pi\sigma^2}}\right)^{w_i} \exp\left(-\frac{(y_i - X_i \beta)^2}{2\sigma^{2}}\right)^{w_i}\\
&= \left(\frac{1}{\sqrt{2\pi\sigma^2}}\right)^{w^{\top}\mathbf{1}} \exp\left(-\frac{1}{2\sigma^{2}} \sum_{i=1}^n w_i(y_i - X_i \beta)^2\right)\\
&= \frac{1}{\left(2\pi\sigma^2\right)^{\frac{w^{\top}\mathbf{1}}{2}} }
\exp\left(-\frac{1}{2\sigma^{2}}  (y - X \beta)^{\top} W (y - X \beta)\right).\\
\end{align*}

Thus, the posterior is proportional to
\begin{align*}
\pi_w\left(\beta, \sigma^{2} \mid X, y\right)
& \propto \; \pi\left(\beta, \sigma^{2}\right) \cdot \exp\left(w^{\top} f_{X,y}(\beta, \sigma^{2})\right) \\
& \propto\left(\frac{1}{\sigma^{2}}\right)^{a_{0}+1 + {\frac{d + w^{\top}\mathbf{1}}{2}}}
\exp\left(-\frac{b_{0}}{\sigma^{2}} -\frac{1}{2\sigma^{2}} Q(\beta)\right),
\end{align*}

where $Q(\beta)$ is defined as:
\begin{align*}\label{eq:multivariate_completion_square}
Q(\beta) &= (\beta - \mu_{0})^{\top}\Lambda_{0}(\beta - \mu_{0}) + (y - X\beta)^{\top}W(y - X\beta)\\
  &= \beta^{\top} \left(\Lambda_{0} + X^{\top}WX\right) \beta - 2\beta^{\top}\left(\Lambda_{0}\mu_{0} + X^{\top}Wy\right)  + \mu_{0}^{\top} \Lambda_{0}\mu_{0} + y^{\top}Wy  \\
  &= \beta^{\top}\Lambda_n\beta - 2\beta^{\top} \Lambda_n\mu_n + c\\
  &= (\beta - \mu_n)^{\top}\Lambda_n(\beta - \mu_n) - \mu_n^{\top}\Lambda_n\mu_n +c \\   
  &= (\beta - \mu_n)^{\top}\Lambda_n(\beta - \mu_n) +c^{\ast},
\end{align*}

with
\begin{align*}
\Lambda_n &= \Lambda_0 + X^{\top}WX, \\
\mu_n &= \Lambda_n^{-1}\left(\Lambda_0\mu_{0} + X^{\top}Wy \right),\\
c &= \mu_{0}^{\top} \Lambda_0\mu_{0} + y^{\top}Wy, \\
c^{\ast} &= \mu_{0}^{\top} \Lambda_0 \mu_{0} + y^{\top}Wy - \mu_n^{\top}\Lambda_n\mu_n.
\end{align*}

Finally, we can express the posterior as a NIG distribution with parameters $(\mu_n, \Lambda_n, a_n, b_n)$, i.e., 

\begin{align*}
\pi_w\left(\beta, \sigma^{2} \mid X, y\right) 
& \propto \left(\frac{1}{\sigma^{2}}\right)^{a_{0}+1 + \frac{d + w^{\top}\mathbf{1}}{2}} 
\exp\left(-\frac{b_{0}}{\sigma^{2}}-\frac{1}{2 \sigma^{2}} ((\beta - \mu_n)^{\top}\Lambda_n(\beta - \mu_n) +c^{\ast})\right)\\
& \propto \left(\frac{1}{\sigma^{2}}\right)^{a_n+1} e^{-\frac{b_n}{\sigma^{2}}}
\left(\frac{1}{\sigma^2}\right)^{\frac{d}{2}}
\exp\left(-\frac{1}{2 \sigma^{2}} (\beta - \mu_n)^{\top}\Lambda_n(\beta - \mu_n)\right)\\
&= \text{Inv-Gamma}\left(\sigma^{2} \mid a_n, b_n \right) \cdot \mathcal{N}\left(\beta \mid \mu_n, \sigma^{2} \Lambda_n^{-1}\right), \\
\end{align*}

where
\begin{align*}
a_n &= a_{0} + \frac{1}{2}w^{\top}\mathbf{1}, \\
b_n &= b_{0} + \frac{c^{\ast}}{2}= b_{0}+\frac{1}{2}\left(\mu_{0}^{\top} \Lambda_{0} \mu_{0}+y^{\top} W y-\mu_n^{\top} \Lambda_n \mu_n\right)\;.
\end{align*}

\newpage

\section{ADDITIONAL EXPERIMENTS: ATTACKS AGAINST BAYESIAN LINEAR REGRESSION SIMULATION STUDY}

This section presents additional experiments that complement those shown in the simulation case study on attacks targeting Bayesian linear regression in Section 6.1.

\subsection{Tainted posterior for increasing values of $B$}

In relation to the attacks targeting the Bayesian linear regression with a normal-inverse-gamma prior from the simulation study in Section 6.1, Figure \ref{fig:attack_visualization_nig} illustrates the duplicated and removed data points identified by the 2O-R2 heuristic for increasing values of $B$, along with the progression of the tainted posterior as it moves closer to the adversarial posterior. As shown in Figure \ref{subfig:attack_visualization_nig_a}, carefully manipulating just 10\% of the data produces a significant shift in the posterior toward the target, while modifying 30\% of the points is sufficient to qualitatively match the target posterior.

\begin{figure*}[htbp!]
\begin{multicols}{3}
    \includegraphics[width=\linewidth]{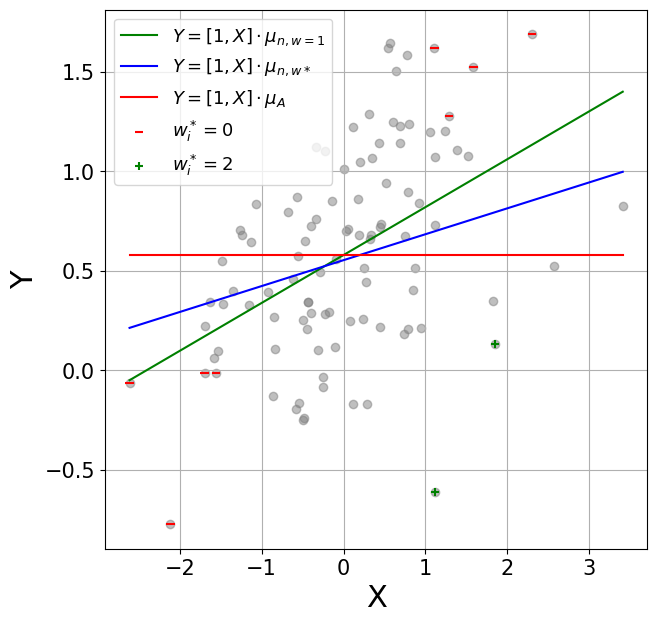}\par 
    \includegraphics[width=\linewidth]{figures/nig_model/points_2O-R2_B=20.png}\par 
    \includegraphics[width=\linewidth]{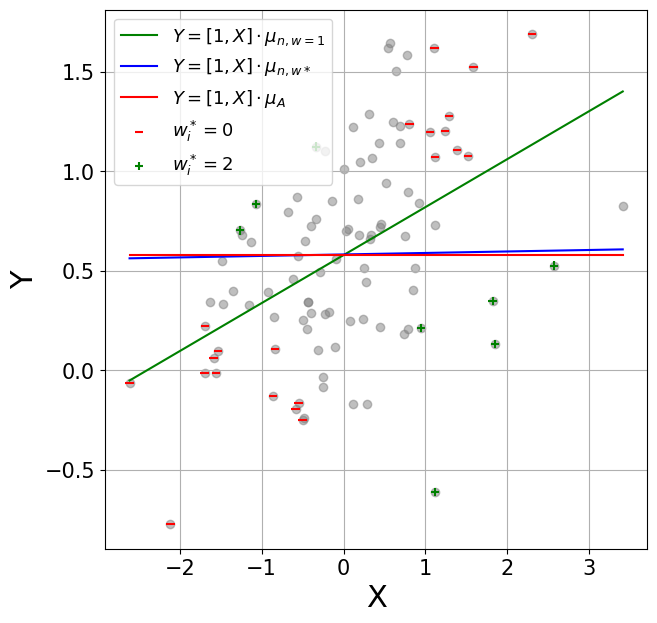}\par 
    \end{multicols}
\begin{multicols}{3}
    \includegraphics[width=\linewidth]{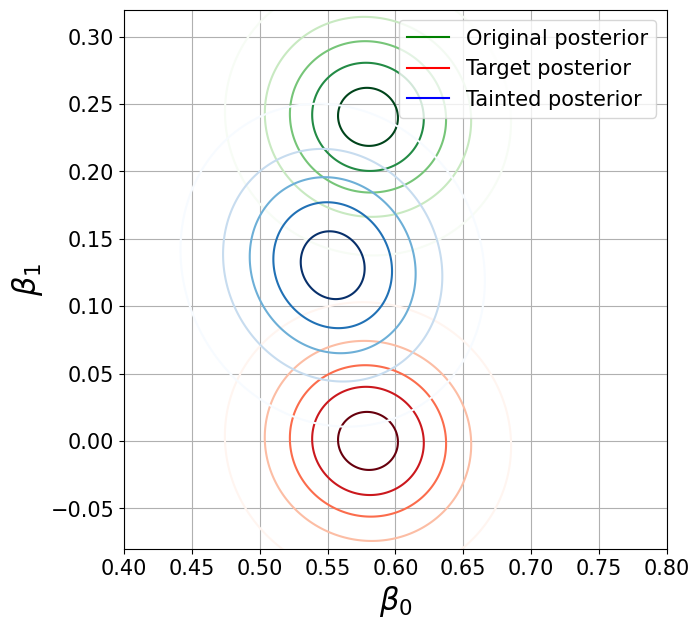}
    \subcaption{$B=10$} \label{subfig:attack_visualization_nig_a}\par
    \includegraphics[width=\linewidth]{figures/nig_model/beta_2O-R2_B=20.png}
    \subcaption{$B=20$}\label{subfig:attack_visualization_nig_b}\par
    \includegraphics[width=\linewidth]{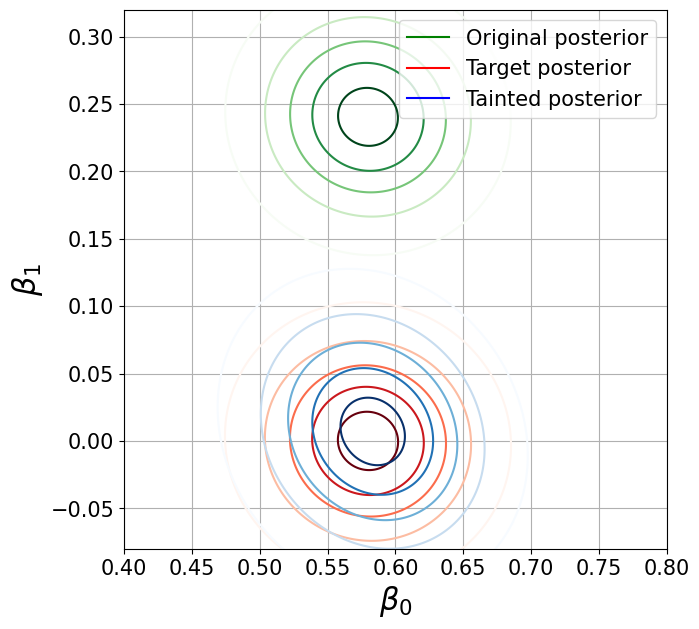}
    \subcaption{$B=30$}\label{subfig:attack_visualization_nig_c}\par
\end{multicols}
\caption{Visualization of the attack found using the 2O-R2 heuristic for $B \in \{10, 20, 30\}$. The first row shows, for each value of $B$, the data points $(X_i, Y_i)$, with those removed or duplicated highlighted. It also displays the regression lines for the equation $y = \mu_0 + \mu_1 x$, where $\mu_1$ corresponds to the posterior mean under untainted data (green), the adversarial posterior mean (red), and the posterior mean under the attack (blue). The second row presents the marginal posterior distributions for $(\beta_0, \beta_1)$, visualized as contour plots from samples of the normal-inverse-gamma model.}
\label{fig:attack_visualization_nig}
\end{figure*}

\subsection{Continuous relaxation problem}

This section studies the impact that rounding has on the solution of the continuous relaxation of the IPA problem. In addition, some variations of the rounded relaxation heuristic are explored. This is done under the experimental setting described in section 6.1. 

\begin{figure*}[h]
\begin{multicols}{3}
    \includegraphics[width=\linewidth]{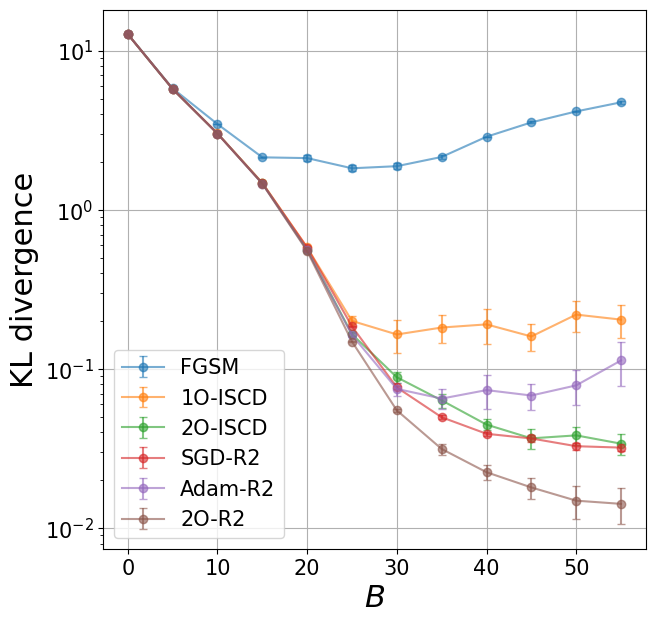}
    \subcaption{Before rounding}\par
    \includegraphics[width=\linewidth]{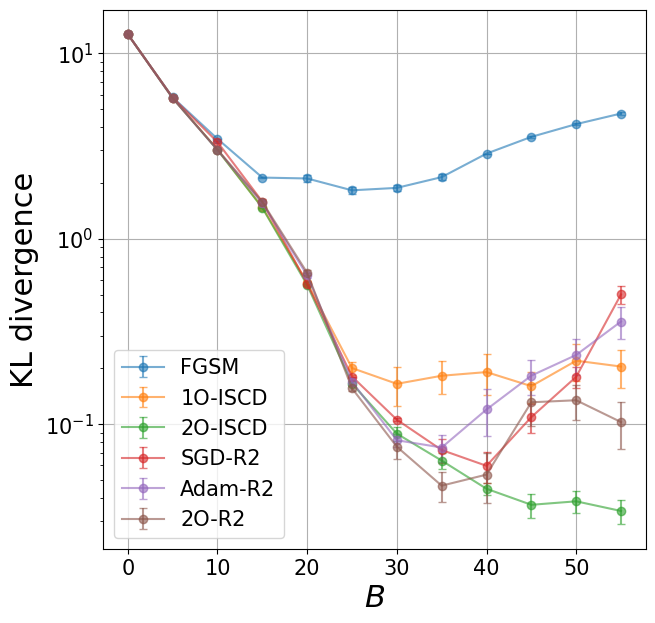}
    \subcaption{After rounding}\par
    \includegraphics[width=\linewidth]{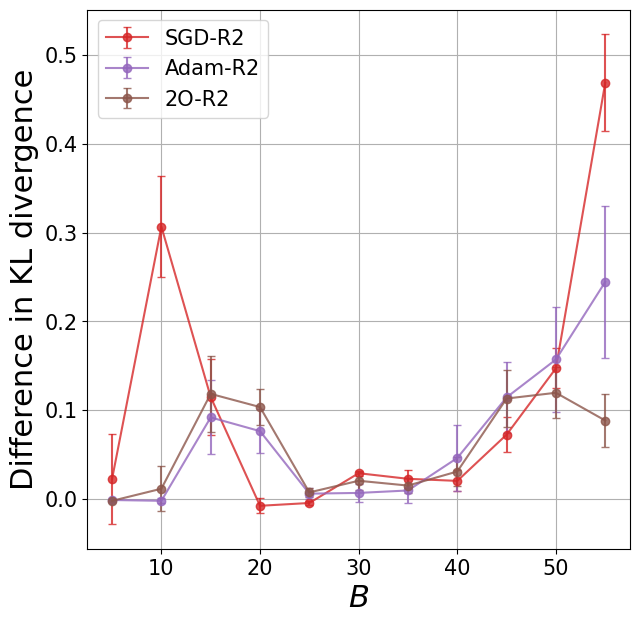}
    \subcaption{Difference due to rounding}\par
\end{multicols}
\caption{Mean plus/minus two standard errors of the objective function obtained with the continuous relaxation heuristics before and after the rounding procedure. The third figure is the difference between the two first ones.}
\label{fig:KL_before_after_rounding}
\end{figure*}

\begin{figure*}[h]
\begin{multicols}{2}
    \centering
    \includegraphics[width=0.6\linewidth]{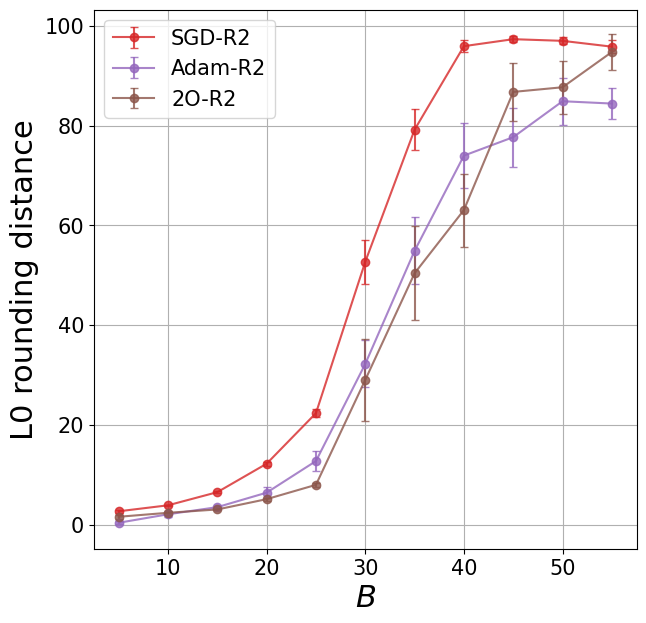}
    \subcaption{L0 distance}\label{subfig:L0_vs_B}\par
    \includegraphics[width=0.6\linewidth]{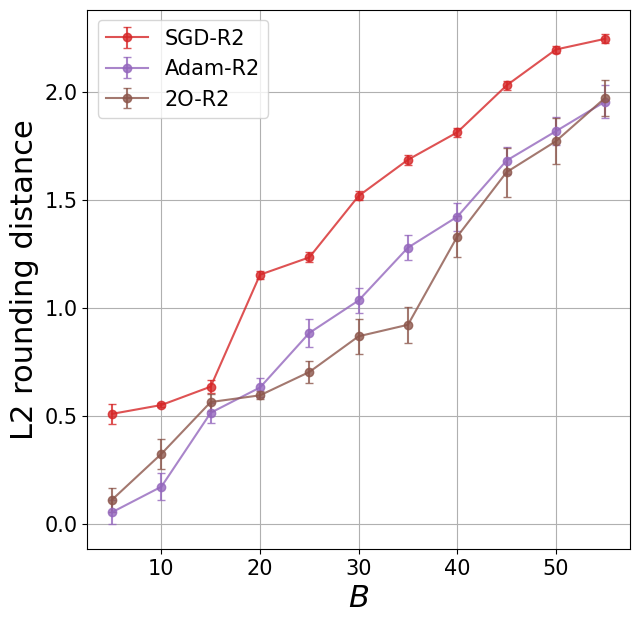}
    \subcaption{L2 distance}\label{subfig:L2_vs_B}\par
\end{multicols}
\caption{Norm between the solution to the continuous relaxation problem and the rounded vector returned. The $L_0$ norm corresponds to the number of coordinates that are changed by the rounding procedure, within a tolerance of $10^{-6}$.}
\label{fig:rounding_distance}
\end{figure*}

The results from Section 6.1 show that the KL divergence obtained from Adam-R2 and 2O-R2 does not decrease with increasing $B$, despite the expanding feasible set. Figure \ref{fig:KL_before_after_rounding} illustrates the objective function values produced by the R2 heuristics both before and after the rounding procedure, along with the difference between them. It is clear that rounding significantly affects the solution, especially for larger values of $B$, which explains why, for some heuristics, the KL divergence increases as $B$ grows. Additionally, Figure \ref{fig:rounding_distance} displays the $L_0$ and $L_2$ norms between the solution to the continuous relaxation problem and the rounded solution for the relaxation-based heuristics. Both norms increase with $B$, suggesting that the worsening performance of R2 heuristics at larger $B$ values is due to the greater impact of the rounding procedure. Specifically, the $L_0$ norm shows that the number of coordinates in the continuous solution that are already integers before rounding decreases with $B$, contributing to the larger impact of rounding for higher $B$.

\subsection{Empirical illustration of convergence speed}
\label{subsec:conv_speed}

This section provides empirical insights into the convergence speed of various heuristics. Figure \ref{fig:convergence_speed} shows the estimated evolution of the objective function across iterations for ISCD and R2 heuristics, applied to the synthetic example from Section 6.1 with $B=30$. These estimations are based on first- and second-order Taylor expansions derived from Monte Carlo estimates of the gradients and Hessians, as the closed-form KL divergence is not used. The objective function is known only up to an additive constant, and we set the origin of the y-axis at the initial KL divergence value.

Due to the convexity of the objective function, first-order Taylor expansions provide lower and upper bounds on the KL divergence decrease between successive weight vectors $w^{t}$ and $w^{t+1}$. Summing these bounds over the iterations gives an estimate of the KL divergence's evolution for each weight vector visited. Similarly, by summing second-order Taylor estimates, we obtain point estimates of the KL divergence difference between $w^{0} = \mathbf{1}$ and $w^{t}$. 
When the decrease between $w^{t}$ and $w^{t+1}$ is estimated using Taylor expansion at $w^{t}$ (or $w^{t+1}$), the resulting estimate is referred to as the ``forward 2O estimate" (or ``backward 2O estimate") in the legend. The plots also display the mean between both second-order estimates at the final iteration, along with this mean plus the KL divergence added by the rounding procedure.

As shown in the figures, 2O- and Adam-R2 heuristics converge in only a few iterations, whereas ISCD exhibits slower convergence. This rapid convergence is a notable advantage of these heuristics, given that each iteration in both methods requires a costly gradient estimation that typically involves running MCMC simulations.

\begin{figure*}[htbp!]
\begin{multicols}{3}
    \includegraphics[width=\linewidth]{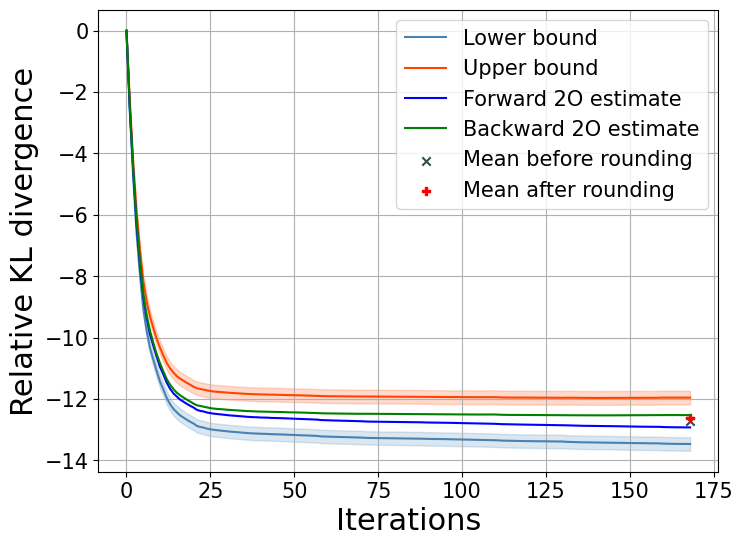}
    \subcaption{SGD-R2 ($\gamma_t=0.1$)}\par 
    \includegraphics[width=\linewidth]{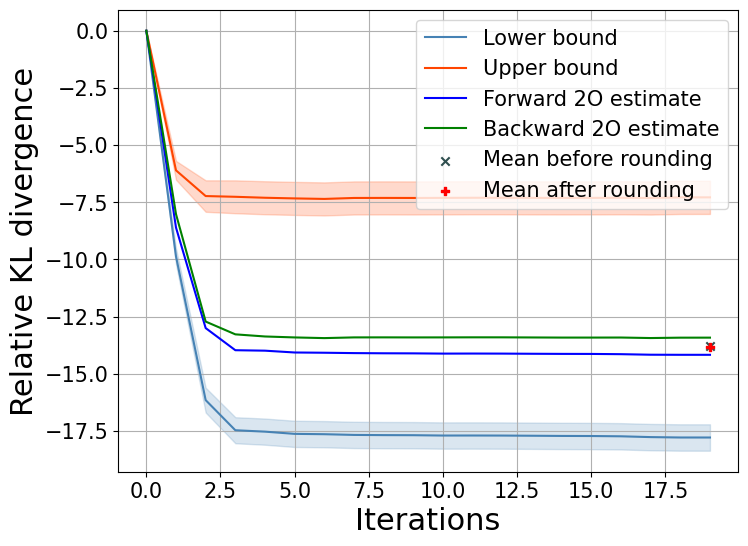}
    \subcaption{Adam-R2}\par 
    \includegraphics[width=\linewidth]{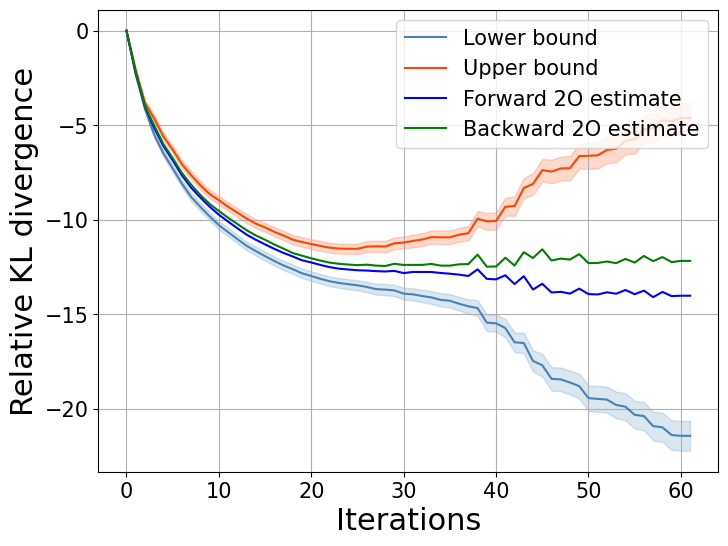}
    \subcaption{1O-ISCD}\par 
    \end{multicols}
\begin{multicols}{3}
    \includegraphics[width=\linewidth]{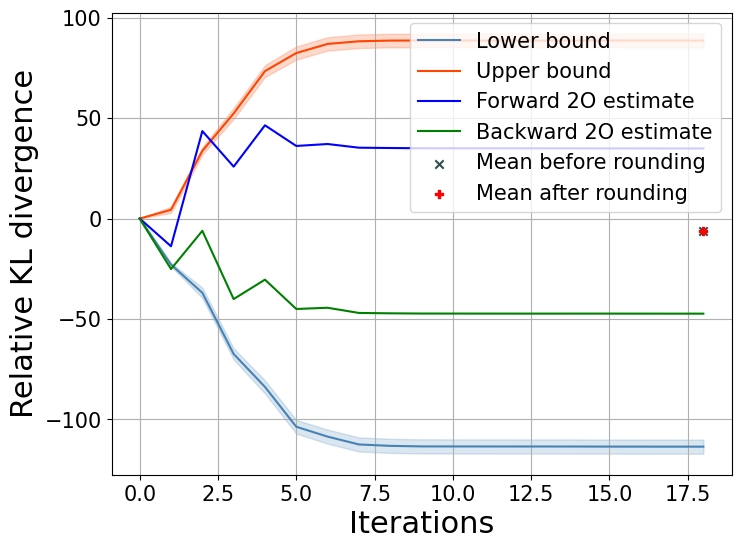}
    \subcaption{SGD-R2 ($\gamma_t=\frac{10}{t+1}$)}\par
    \includegraphics[width=\linewidth]{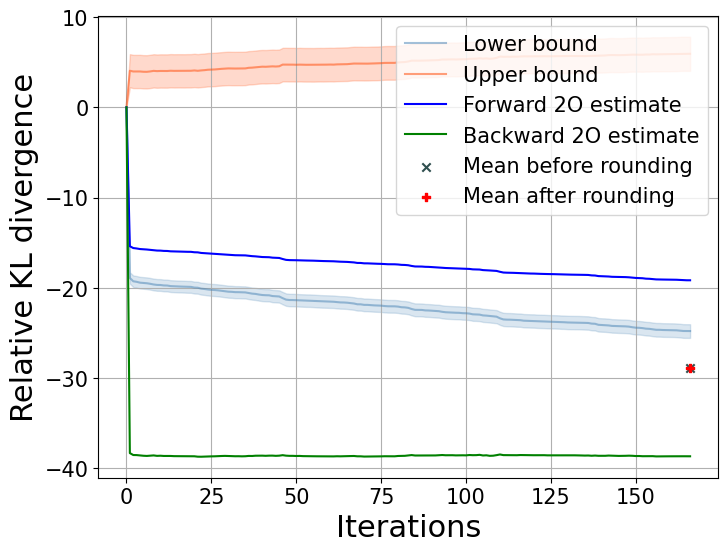}
    \subcaption{2O-R2}\par
    \includegraphics[width=\linewidth]{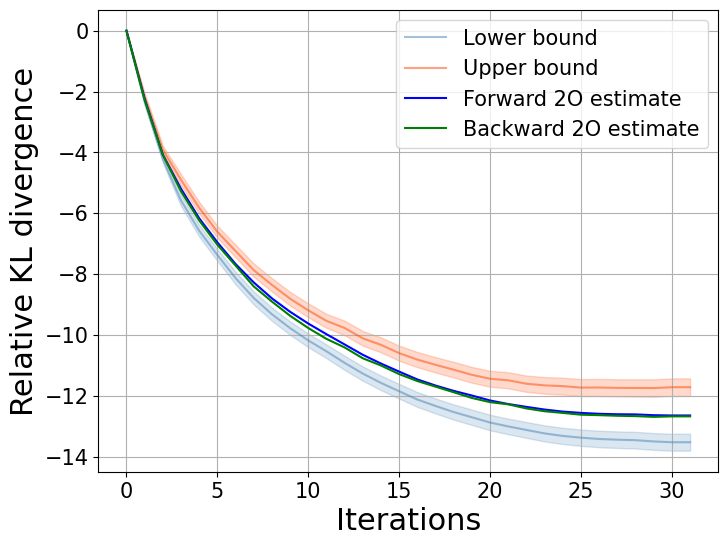}
    \subcaption{2O-ISCD}\par
\end{multicols}
\caption{Estimations of the KL divergence decrease using Taylor expansions for all iterative heuristics.}
\label{fig:convergence_speed}
\end{figure*}

\subsection{Influence of noise level}

\label{subsec:noise_level}

\begin{figure*}[h]
\begin{multicols}{3}
    \includegraphics[width=\linewidth]{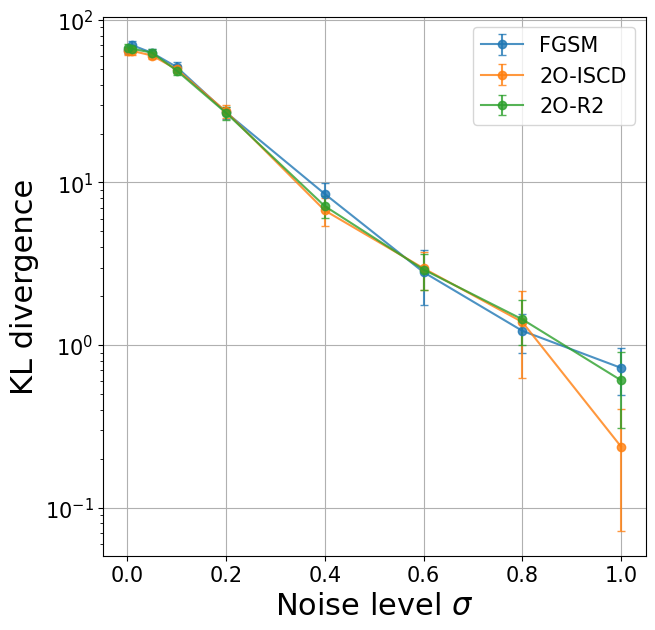}
    \subcaption{KL vs $\sigma$}\par
    \includegraphics[width=\linewidth]{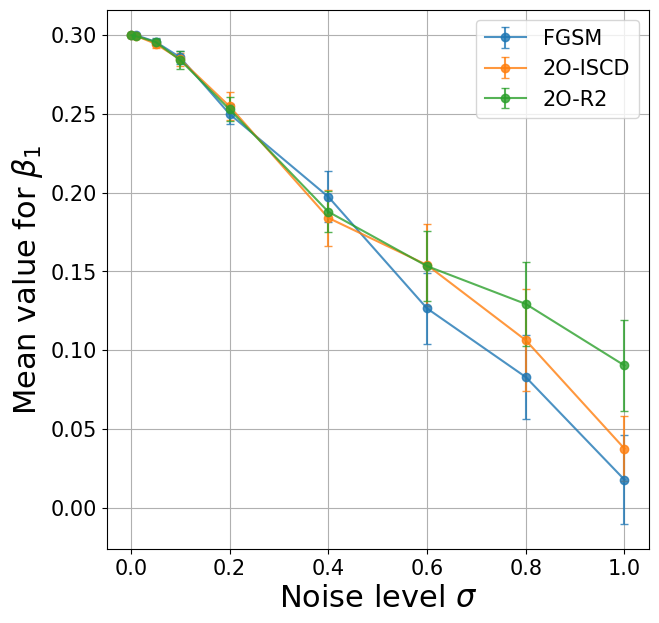}
    \subcaption{$\mathbb{E}_{w^*}\left[\beta_1\right]$ vs $\sigma$}\par
    \includegraphics[width=\linewidth]{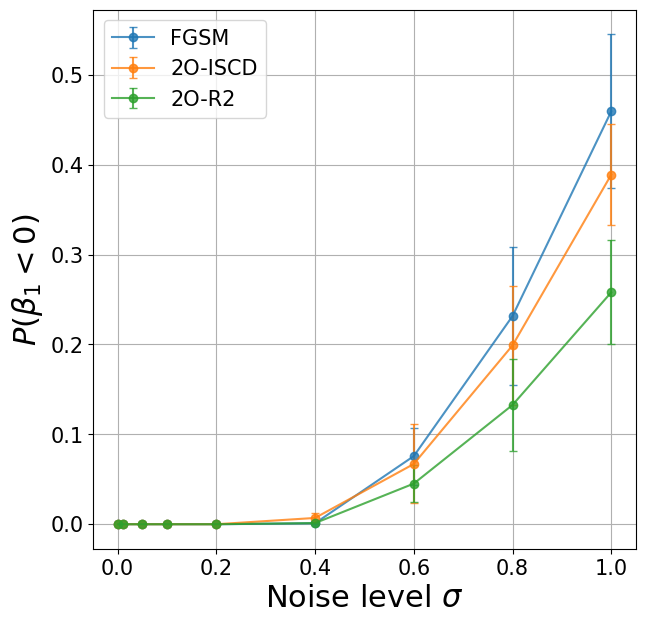}
    \subcaption{$\pi_{w^*}(\beta_1<0)$ vs $\sigma$}\par
\end{multicols}
\caption{Mean metrics for attacks against the NIG model with different noise levels}
\label{fig:noise_level_nig_metrics}
\end{figure*}

\begin{figure*}[htb]
\begin{multicols}{3}
    \includegraphics[width=\linewidth]{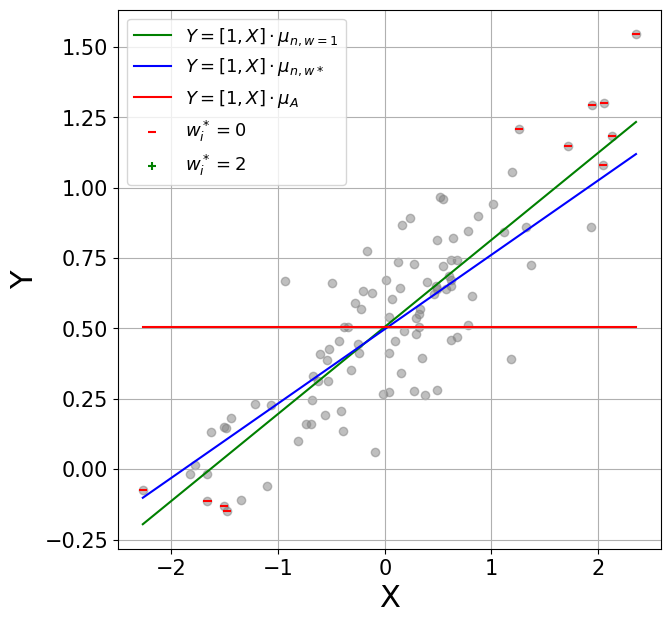}\par 
    \includegraphics[width=\linewidth]{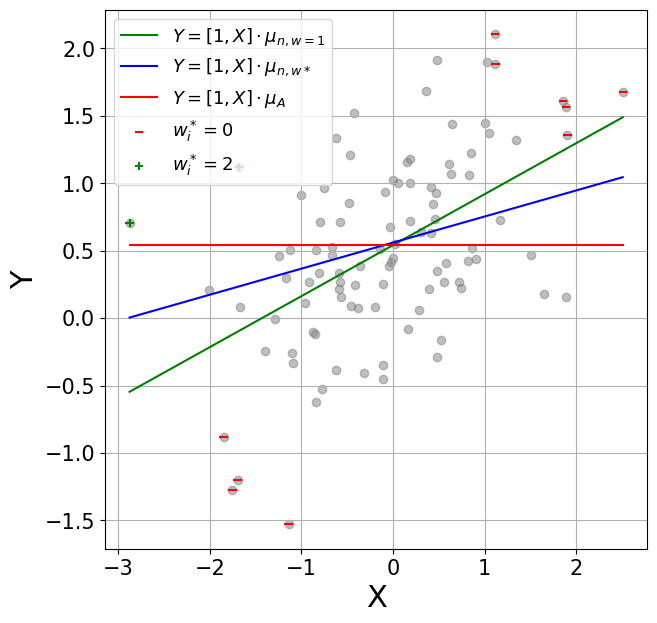}\par 
    \includegraphics[width=\linewidth]{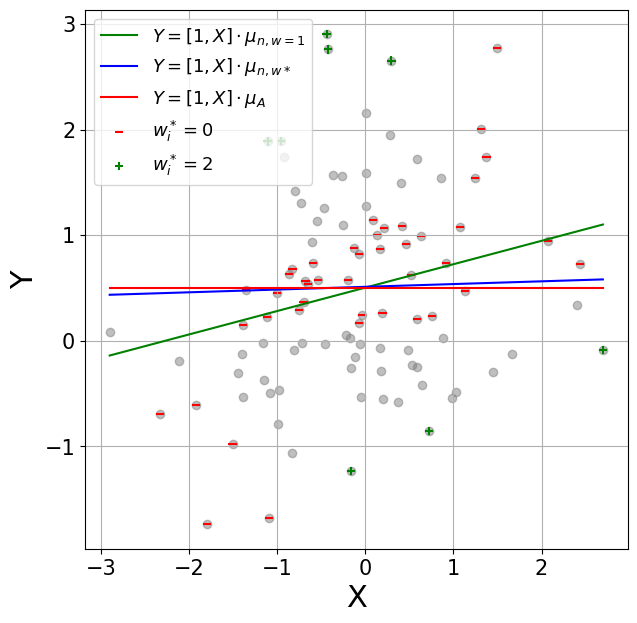}\par 
    \end{multicols}
\begin{multicols}{3}
    \includegraphics[width=\linewidth]{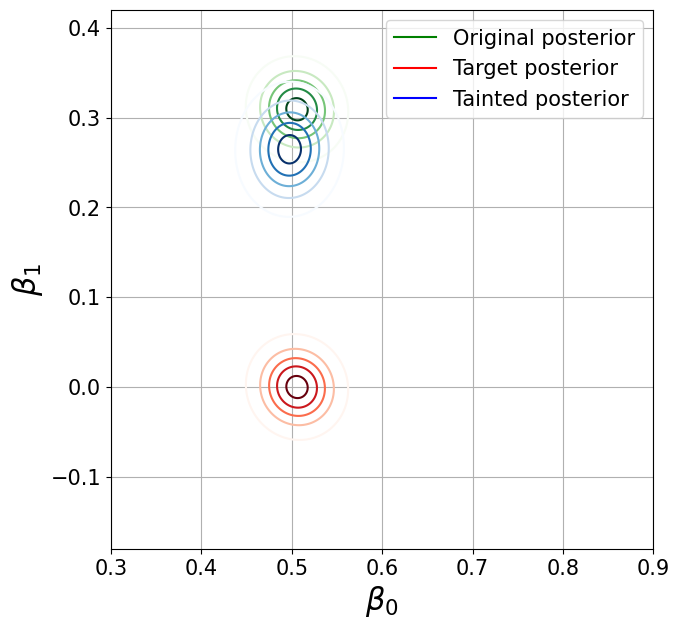}
    \subcaption{$\sigma = 0.2$}\par
    \includegraphics[width=\linewidth]{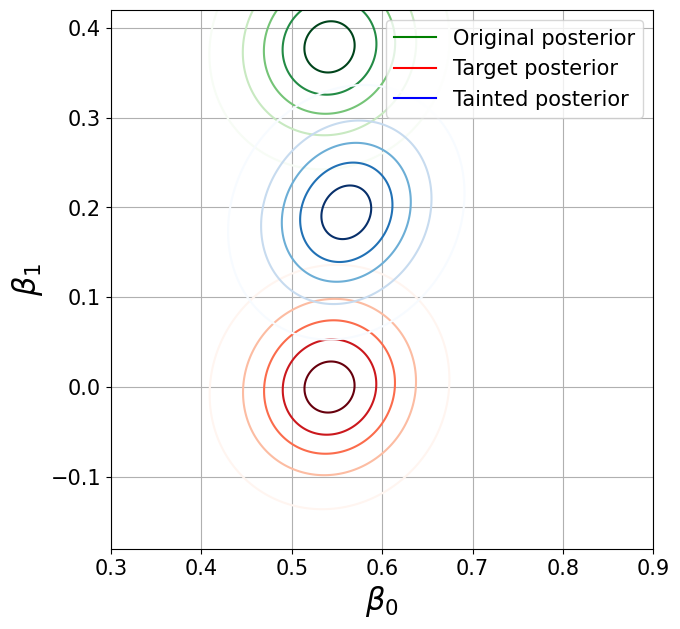}
    \subcaption{$\sigma = 0.6$}\par
    \includegraphics[width=\linewidth]{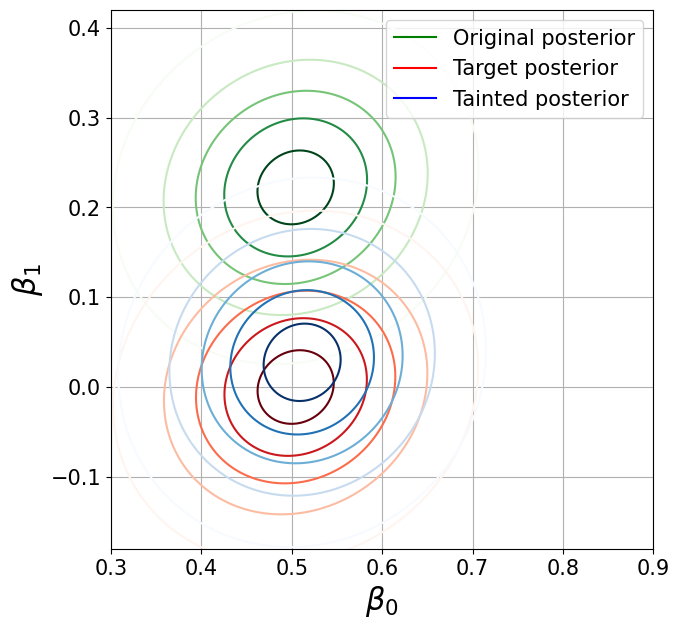}
    \subcaption{$\sigma = 1.0$}\par
\end{multicols}
\caption{Visualization of the attack using 2O-R2 heuristic, for $\sigma \in \{0.2, 0.4, 1.0\}$. The first row displays, for each value $\sigma$, the data samples $(X_i, Y_i)$ and marks the ones are removed or replicated, as well as regression lines of equation $y=\mu_0 + \mu_1 x$  where $\mu_1$ corresponds to the posterior mean under untainted data (green), the adversarial posterior mean (red), and the posterior mean under the attack (blue).
The second row presents the marginal posterior distributions for $(\beta_0, \beta_1)$, visualized as contour plots from samples of the normal-inverse-gamma model.
}
\label{fig:noise_level_nig}
\end{figure*}

In the experiments from Section 6.1 comparing the heuristics, Figure 1 showed that the first data points to remove or duplicate were those furthest from the true regression line $Y = \beta_0 + \beta_1 X$. This suggests that outliers exert a stronger influence on the posterior compared to points closer to the model's prediction, implying that noisier datasets may be easier to attack. Intuitively, since the posterior is proportional to the product of the likelihood functions $(\theta \mapsto \pi(X_i | \theta))_i$, a noisier dataset offers more diversity among these functions, giving the attacker greater flexibility to manipulate the posterior.
To explore this further, we conducted experiments by varying the noise level in the synthetic dataset. The setup was the same as in Section 6.1, except that the attack intensity was fixed at $B = 10$, and the noise standard deviation $\sigma$ was varied between 0.001 and 1.0. Each experiment was repeated 30 times with a new dataset generated for each repetition. 
We evaluated the effectiveness of the attacks using three metrics: the KL divergence between the target distribution and the tainted posterior, the mean of the marginal tainted posterior for the slope parameter $\beta_1$, and the posterior probability of the slope being negative. These metrics are shown in Figure \ref{fig:noise_level_nig_metrics} as functions of $\sigma$. Additionally, Figure \ref{fig:noise_level_nig} visualizes the data perturbations, along with the original, induced, and target marginal posteriors for $\beta$ at three different noise levels.

Across all three metrics, the attacks have minimal effect when the noise level is low, but their impact increases significantly as $\sigma$ rises, with the mean posterior for $\beta_1$ reaching approximately 0.05 when $\sigma = 1.0$. 
This can be intuitively understood from Figure \ref{fig:noise_level_nig}, which shows that at low noise levels, the dataset's limited diversity and absence of outliers restrict the effect that removing or duplicating points has on the posterior. In contrast, with higher noise levels, the presence of more outliers allows the same attack intensity to induce a larger shift.
These findings support the hypothesis that the sensitivity of models to our proposed attacks is influenced by the presence of outliers in the dataset.

\clearpage

\subsection{Attacking uncertainty}
\label{subsec:attacking_uncertainty}

All the attacks in this paper were designed to steer the posterior distribution in a specific direction, primarily focusing on shifting the posterior mean. However, a common reason for choosing Bayesian methods over frequentist ones is their ability to provide reliable uncertainty estimates. An adversary could, therefore, be motivated to target the shape of the posterior, aiming to increase or decrease its variance without altering the mean. By manipulating the posterior variance, the attacker could change the defender's level of uncertainty, which could, in turn, significantly impact their decision-making process.

This final experiment on synthetic data for linear regression assesses the robustness of the NIG model when the attack specifically targets the uncertainty over the slope parameter, $\beta_1$. We use the same synthetic dataset from Section 6.1, along with the same prior parameters.
The attacker's goal is to either increase or decrease the defender's uncertainty regarding $\beta_1$. The adversarial target is defined as a NIG distribution, where the parameters are identical to those of the original posterior, except for the covariance matrix $\Sigma_A=\Lambda_A^{-1}$. The covariance associated with $\beta_1$ is manually adjusted to $\Sigma_{A, 11} = \rho^2 \Sigma_{n, 11}$, where $\Sigma_n=\Lambda_n^{-1}$ and $\rho > 0$ represents the desired scaling factor for the attacker's uncertainty manipulation.
To ensure that $\Sigma_A$ remains positive semi-definite, we first initialize $L_A$ as the Cholesky decomposition of $\Sigma_n$, then modify the targeted coefficient by setting $L_{A, 11} \gets \rho L_{A, 11}$. Finally, $\Sigma_A$ is defined as $\Sigma_A \coloneq L_A L_A^\top$.

We conducted two versions of the experiment, setting $\rho = \frac{1}{10}$ and $\rho = 10$, where the attacker aims to respectively decrease or increase the uncertainty over $\beta_1$. The attacks were evaluated using three metrics: the KL divergence between the adversarial and tainted posteriors, the square root of the targeted variance parameter $\sqrt{\Sigma_{n, 11}}$, and the standard deviation of the marginal posterior over $\beta_1$, denoted as $\text{std}_w(\beta_1)$. Figure \ref{fig:uncertainty_metrics} shows these metrics as functions of attack intensity for both values of $\rho$.

\begin{figure*}[h]
\begin{multicols}{3}
    \includegraphics[width=\linewidth]{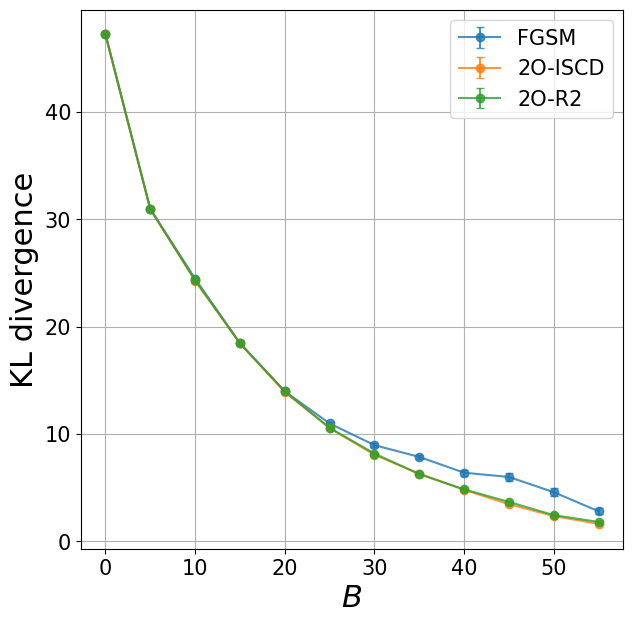}
    \subcaption{KL divergence vs $B$, $\rho=10$}\par
    \includegraphics[width=\linewidth]{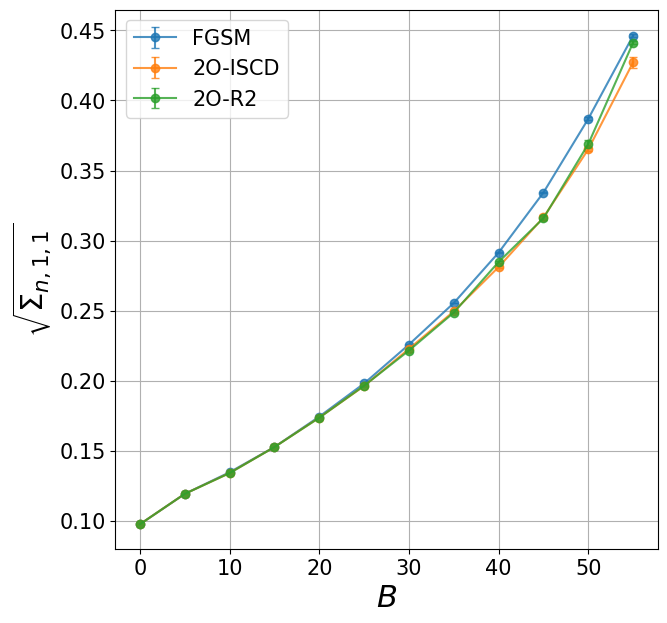}
    \subcaption{$\sqrt{\Sigma_{n, 11}}$ vs $B$, $\rho=10$}\par
    \includegraphics[width=\linewidth]{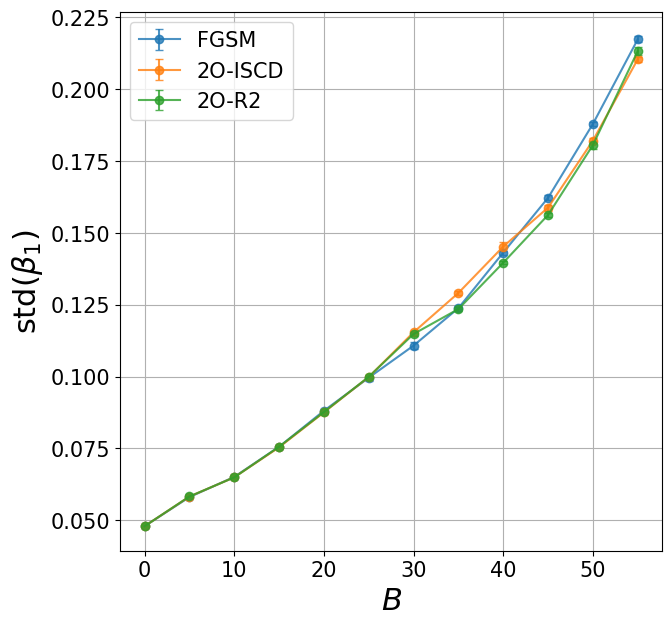}
    \subcaption{$\text{std}_w(\beta_1)$ vs $B$, $\rho=10$}\par
\end{multicols}
\begin{multicols}{3}
    \includegraphics[width=\linewidth]{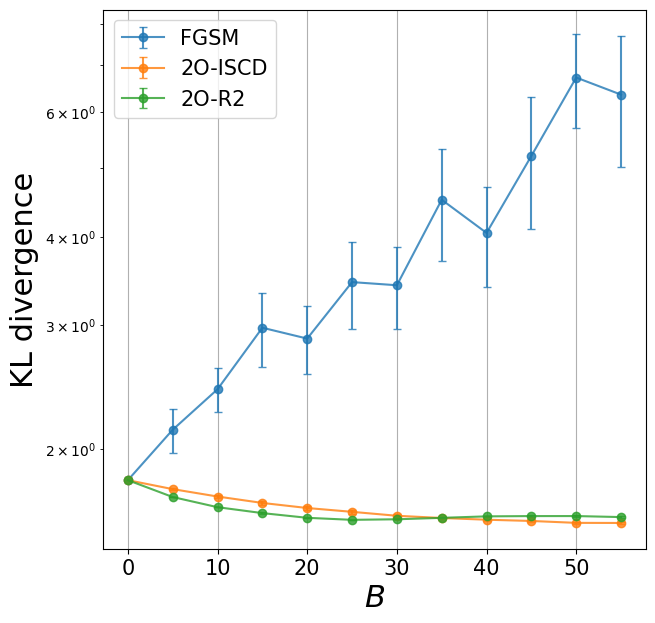}
    \subcaption{KL divergence vs $B$, $\rho=\frac{1}{10}$}\par
    \includegraphics[width=\linewidth]{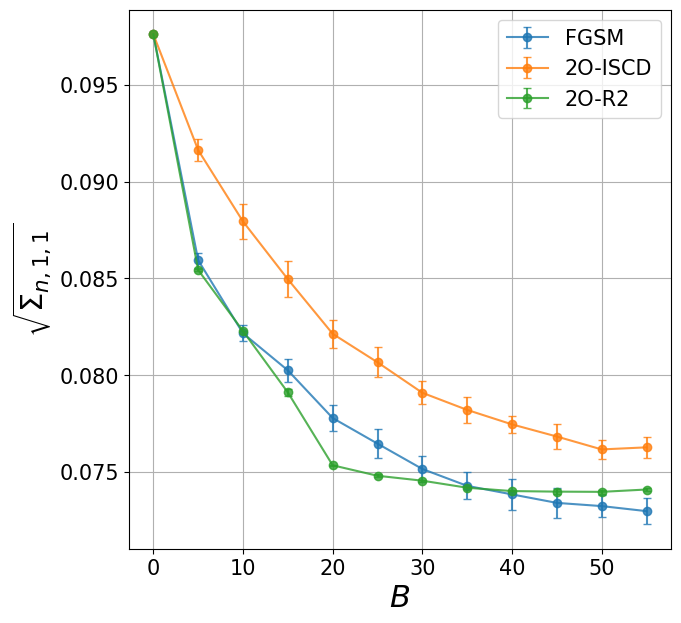}
    \subcaption{$\sqrt{\Sigma_{n, 11}}$ vs $B$, $\rho=\frac{1}{10}$}\par
    \includegraphics[width=\linewidth]{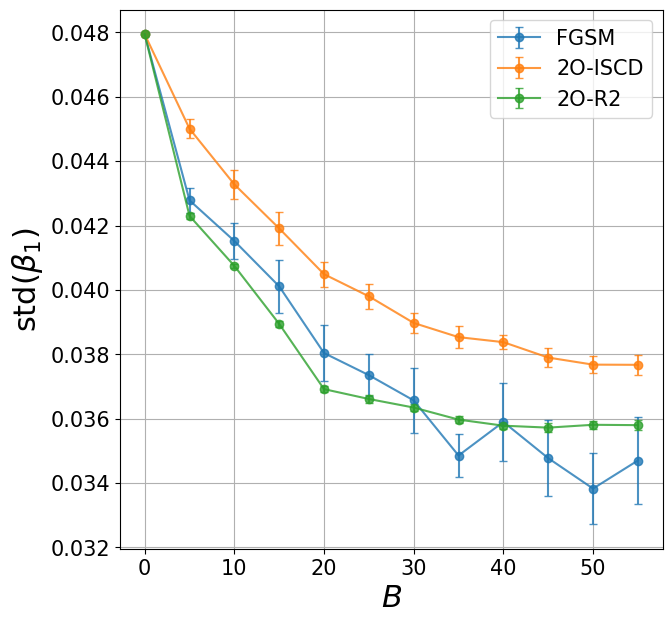}
    \subcaption{$\text{std}_w(\beta_1)$ vs $B$, $\rho=\frac{1}{10}$}\par
\end{multicols}
\caption{Metrics for the attack targeting the uncertainty over $\beta_1$. The two rows correspond respectively to attacks increasing and decreasing uncertainty.}
\label{fig:uncertainty_metrics}
\end{figure*}

\begin{figure*}[h!]
\begin{multicols}{2}
\centering
    \includegraphics[width=0.5\linewidth]{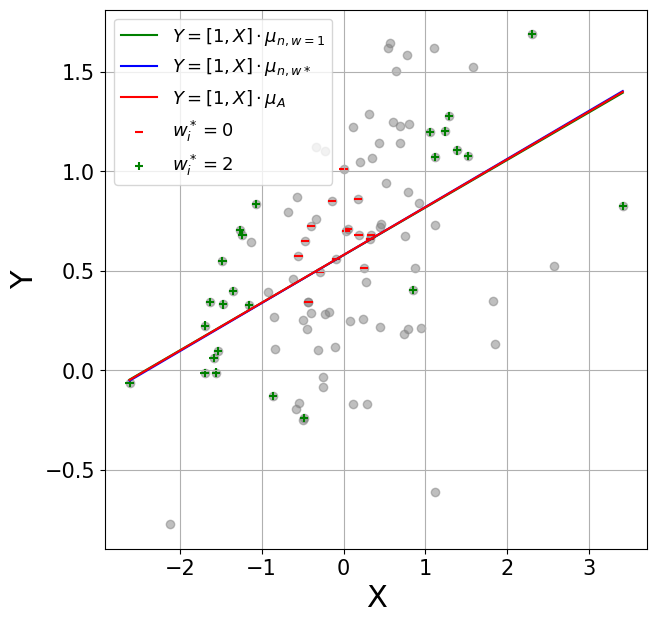}
    \subcaption{Decrease the uncertainty}\par
    \includegraphics[width=0.5\linewidth]{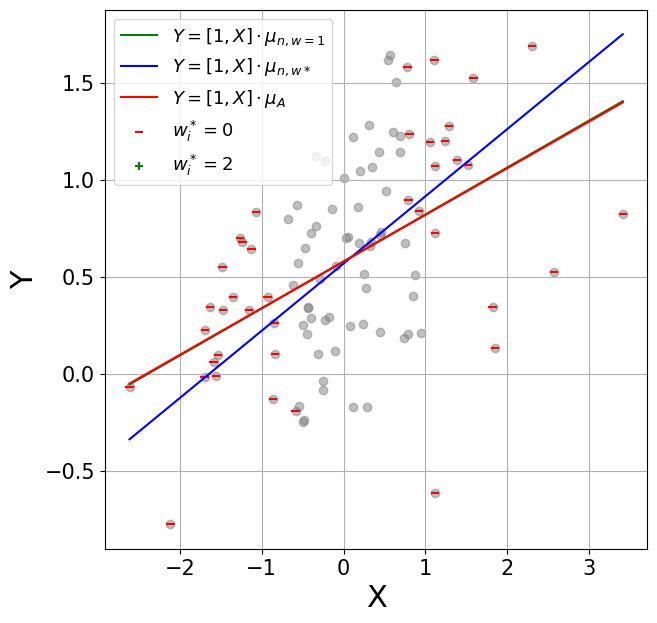}
    \subcaption{Increase the uncertainty}\par
\end{multicols}
\caption{Data modifications to attack the uncertainty by 2O-ISCD heuristic with $B=40$}
\label{fig:datapoints_uncertainty}
\end{figure*}

In the first row, where the attacker seeks to increase uncertainty, all three heuristics yield similar results across all metrics. FGSM performs slightly worse than the others for stronger attacks ($B \geq 30$) in terms of KL divergence. The original posterior standard deviation $\text{std}_{w=\mathbf{1}}(\beta_1) = 0.048$ is doubled with all heuristics for $B = 25$ (25\% of the dataset) and nearly tripled for $B = 40$ (40\% of the dataset).

In the second row, where the goal is to reduce uncertainty, FGSM underperforms in minimizing KL divergence but produces comparable results to other heuristics on the remaining two metrics. This is because these metrics do not account for the posterior location, and while the data modifications reduce uncertainty, they also shift the mean of $\beta_1$ away from the original value which coincides with the target mean. For all heuristics, $\sqrt{\Sigma_{n, 11}}$ is halved for $B \geq 35$, and $\text{std}_w(\beta_1)$ is reduced by about 25\%.

Figure \ref{fig:datapoints_uncertainty} presents examples of data modifications, one aiming to decrease uncertainty and the other to increase it. Many changes occur on points with the highest $|X_i|$, which are either duplicated for $\rho = \frac{1}{10}$ or deleted for $\rho = 10$. These changes can be explained by the closed-form expression for the precision coefficient of $\beta_1$: $\Lambda_{n, 11} = \Lambda_{0, 11} + \sum_{i=1}^n w_i X_i^2$. Deleting points with large $|X_i|$ reduces precision and increases uncertainty over $\beta_1$, while duplicating them has the opposite effect.

In summary, this experiment demonstrated that the proposed method can effectively attack the standard deviation of the posterior distribution, either increasing or decreasing the uncertainty surrounding the parameter of interest.

\newpage

\section{ADDITIONAL EXPERIMENTS: ATTACKS TO BAYESIAN LINEAR REGRESSION WITH REAL DATA}

This section presents additional experiments that complement the results discussed in Section 6.2.

\subsection{Boston houses dataset}

We provide additional information and experiments on poisoning attacks targeting Bayesian linear regression for house price prediction.

\subsubsection{Horseshoe prior regression specification}

For the experiments in Section 6.2, a Bayesian linear regression model was used with a horseshoe prior to induce sparsity. The full sampling model is specified as follows:

\begin{align*}
    \tau &\sim \mathcal{C}^+(1) \\
    \lambda_j &\sim \mathcal{C}^+(1), \quad &\forall\; 1 \leq j \leq d, \\
     \beta_j \mid \tau, \lambda_j &\sim \mathcal{N}(0, \tau^2 \lambda_j^2),\quad &\forall\ 1 \leq j \leq d, \\
    \alpha &\sim \mathcal{N}(0, \sigma_\alpha^2) \\
    \sigma &\sim \mathcal{C}^+(1) \\
    y_i \mid \alpha, \beta, X_i, \sigma &\sim \mathcal{N}(\alpha + \beta^{\top}X_i, \sigma^2), &\forall\ 1 \leq i \leq n,
\end{align*}

\noindent where $\mathcal{C}^+$ is the half-Cauchy distribution centered at zero and having the indicated scale parameter.

\subsubsection{Additional experiments}

Additional visualizations for the attacks on the Boston housing price dataset are provided herein. Figure \ref{fig:Boston_Horseshoe_metrics} presents further performance metrics comparing the heuristics under the same setting as in Section 6.2. Additionally, Figure \ref{fig:Boston_Horseshoe_varying_B} presents attacks at three different intensities, demonstrating that with $B=50$ (10\% of the data), the shift in the marginal posterior for $\beta_{RM}$ towards the adversarial target is substantial.

\begin{figure*}[h]
\begin{multicols}{4}
    \centering
    \includegraphics[width=\linewidth]{figures/Boston_lin_reg/RM/KL_div_HS.png}
    \subcaption{KL divergence vs $B$}\par
    \includegraphics[width=\linewidth]{figures/Boston_lin_reg/RM/mean_HS_MCMC.png}
    \subcaption{$\mathbb{E}_{w^*}(\beta_{RM})$ vs $B$}\par
    \includegraphics[width=\linewidth]{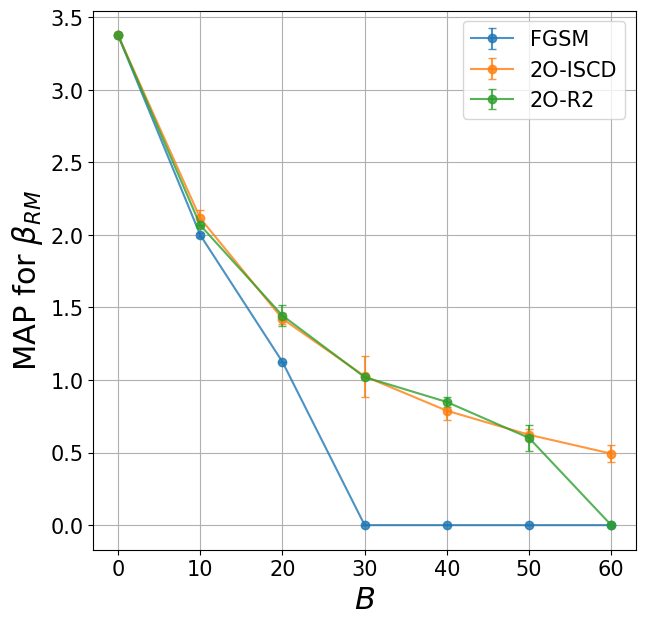}
    \subcaption{MAP of $\beta_{RM}$ vs $B$}\par
    \includegraphics[width=\linewidth]{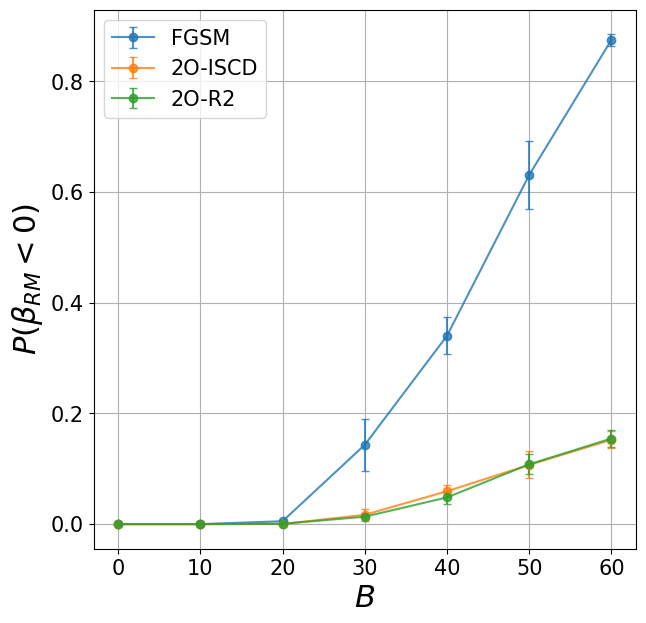}
    \subcaption{$\pi_{w^*}(\beta_{RM}<0)$ vs $B$}\par
\end{multicols}
\caption{Efficiency metrics for attacks against Horseshoe model depending on the intensity}
\label{fig:Boston_Horseshoe_metrics}
\end{figure*}

\begin{figure*}[h]
\begin{multicols}{3}
    \centering
    \includegraphics[width=\linewidth]{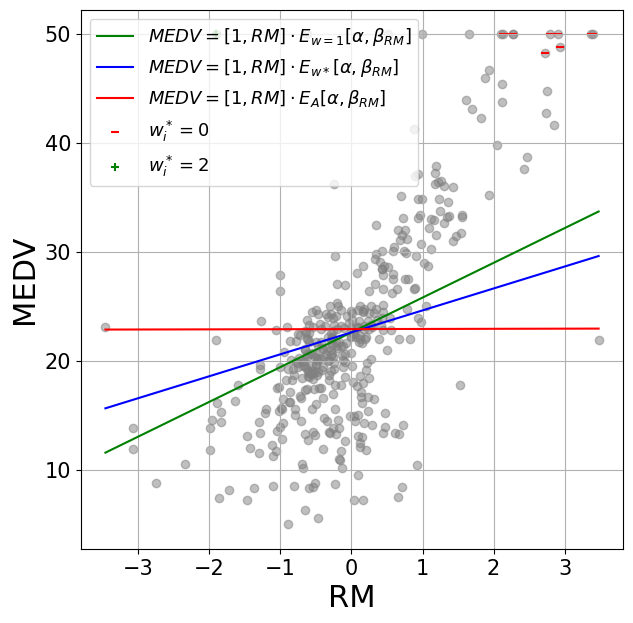}\par
    \includegraphics[width=\linewidth]{figures/Boston_lin_reg/RM/datapoints_2O-ISCD_B=30_HS.png}\par
    \includegraphics[width=\linewidth]{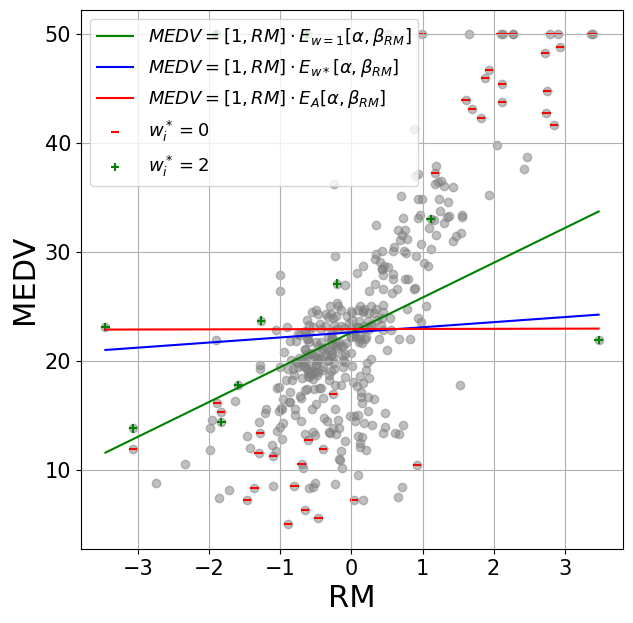}\par
\end{multicols}
\begin{multicols}{3}
    \centering
    \includegraphics[width=\linewidth]{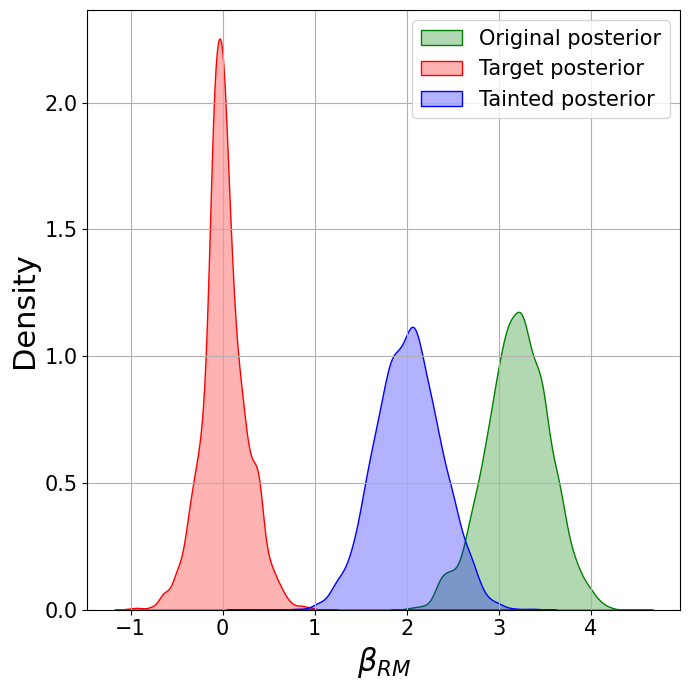}
    \subcaption{$B=10$}\par
    \includegraphics[width=\linewidth]{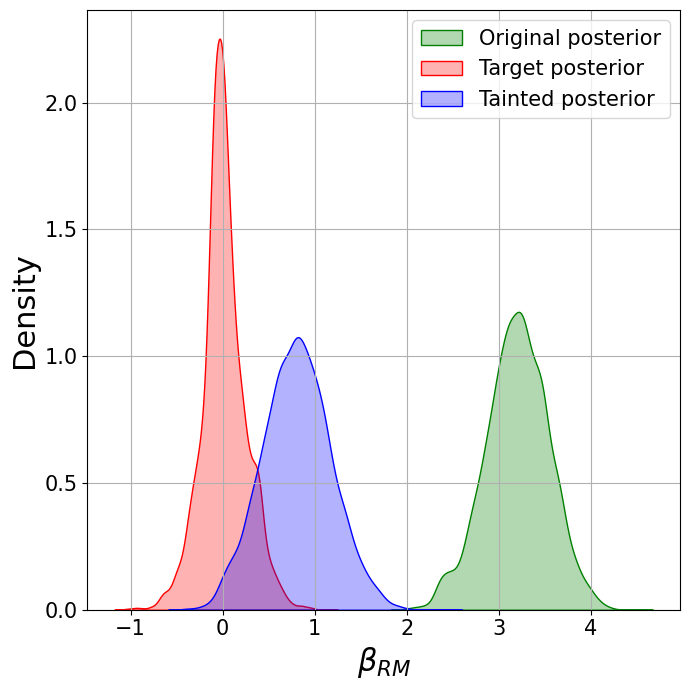}
    \subcaption{$B=30$}\par
    \includegraphics[width=\linewidth]{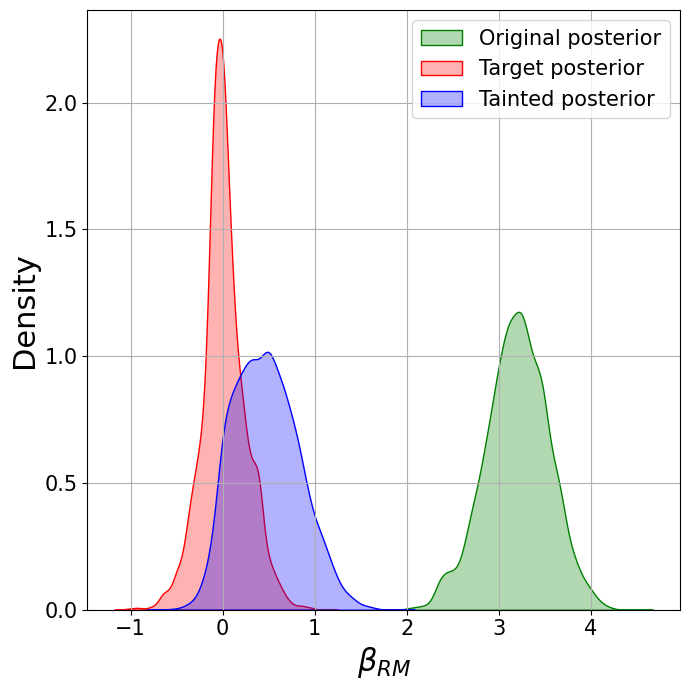}
    \subcaption{$B=50$}\par
\end{multicols}
\caption{Attacks of increasing intensities against the Horseshoe model on Boston dataset with 2O-ISCD heuristic.}x
\label{fig:Boston_Horseshoe_varying_B}
\end{figure*}

Figure \ref{fig:Horseshoe_allmargposteriors} shows all marginal posteriors induced by the 2O-ISCD attack with $B=30$, comparing them to the original and adversarial posteriors. These visualizations highlight the nuanced approach taken by the attacker to perturb the posterior, demonstrating how the attack successfully steers the tainted posterior toward the desired target. Notably, these plots reinforce the conclusion that our attack enables precise poisoning; while the marginal posterior of $\beta_{RM}$ is effectively shifted, the marginal posteriors for most of the other parameters remain largely unaffected, indicating the attacker’s intent to leave other parameters unchanged.

\begin{figure}[h!]
    \centering
    \begin{minipage}{0.15\textwidth}
        \centering
        \includegraphics[width=\textwidth]{figures/Boston_lin_reg/RM/marginals/beta_AGE.png}
        \caption*{$\beta_{\text{AGE}}$}
    \end{minipage}
    \begin{minipage}{0.15\textwidth}
        \centering
        \includegraphics[width=\textwidth]{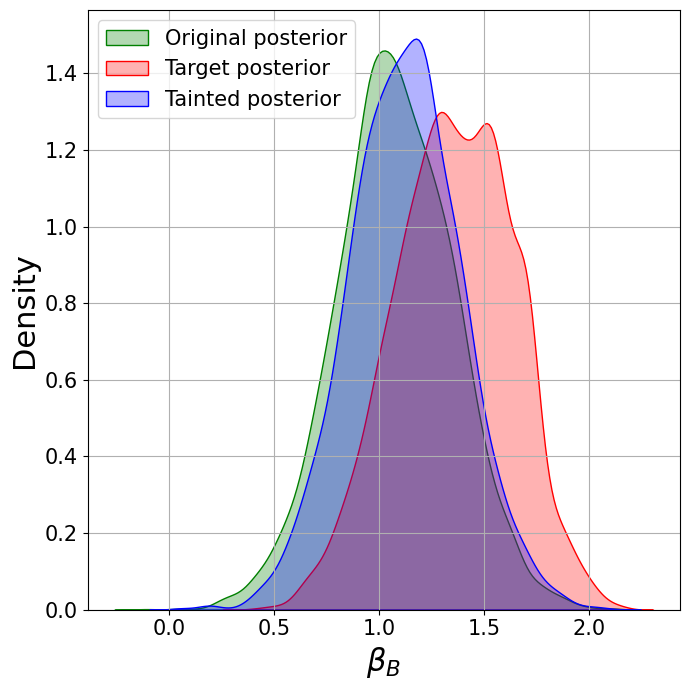}
        \caption*{$\beta_{\text{B}}$}
    \end{minipage}
    \begin{minipage}{0.15\textwidth}
        \centering
        \includegraphics[width=\textwidth]{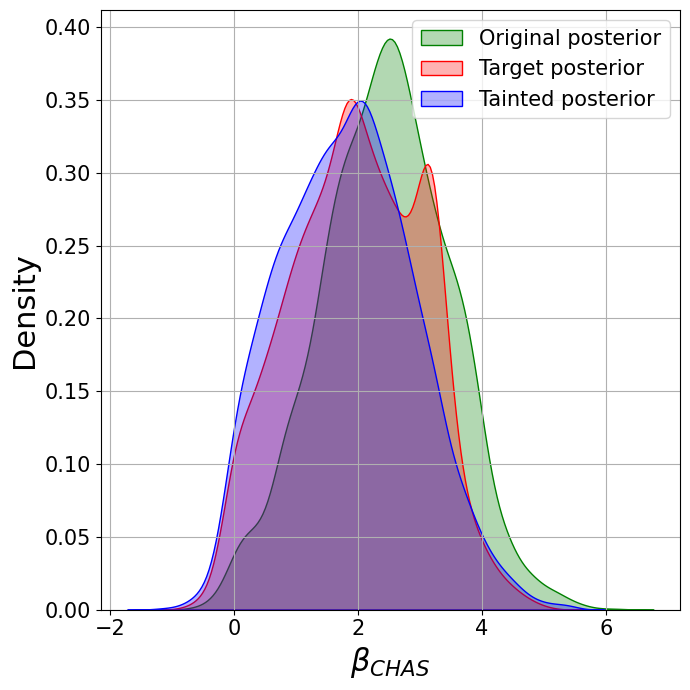}
        \caption*{$\beta_{\text{CHAS}}$}
    \end{minipage}
    \begin{minipage}{0.15\textwidth}
        \centering
        \includegraphics[width=\textwidth]{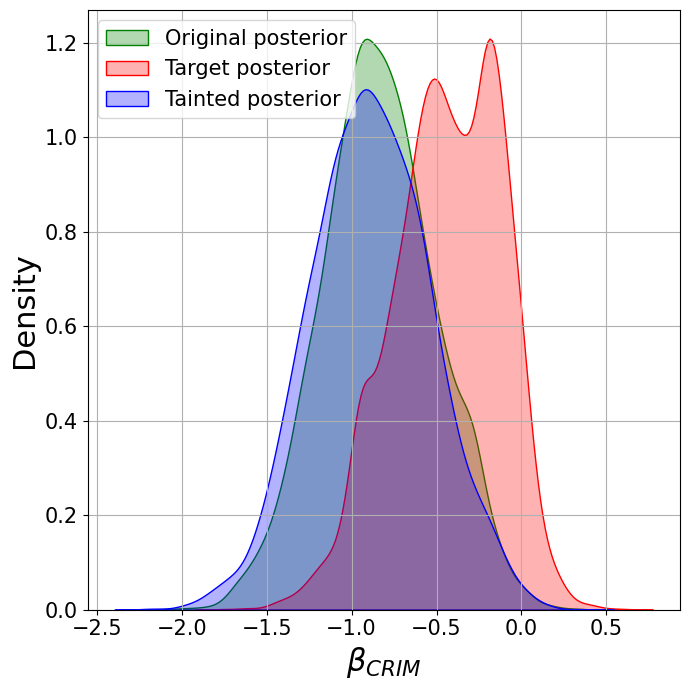}
        \caption*{$\beta_{\text{CRIM}}$}
    \end{minipage}
    \begin{minipage}{0.15\textwidth}
        \centering
        \includegraphics[width=\textwidth]{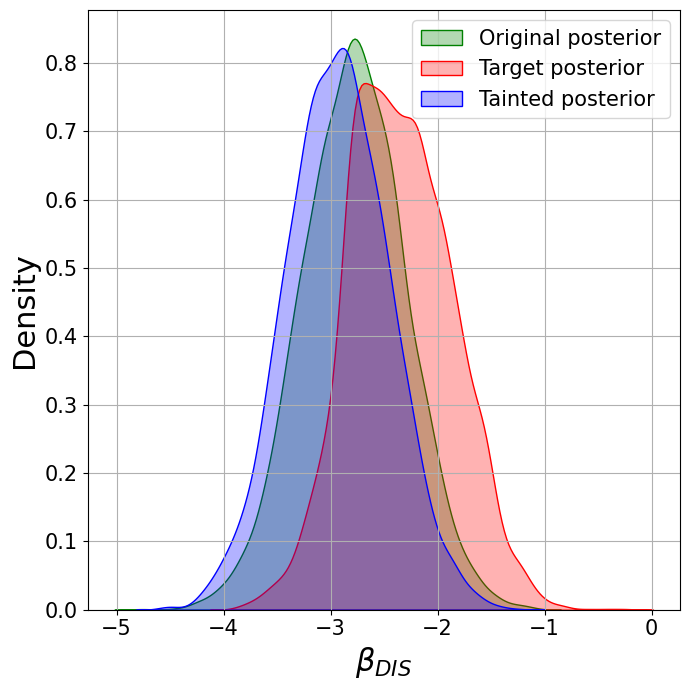}
        \caption*{$\beta_{\text{DIS}}$}
    \end{minipage}
    \begin{minipage}{0.15\textwidth}
        \centering
        \includegraphics[width=\textwidth]{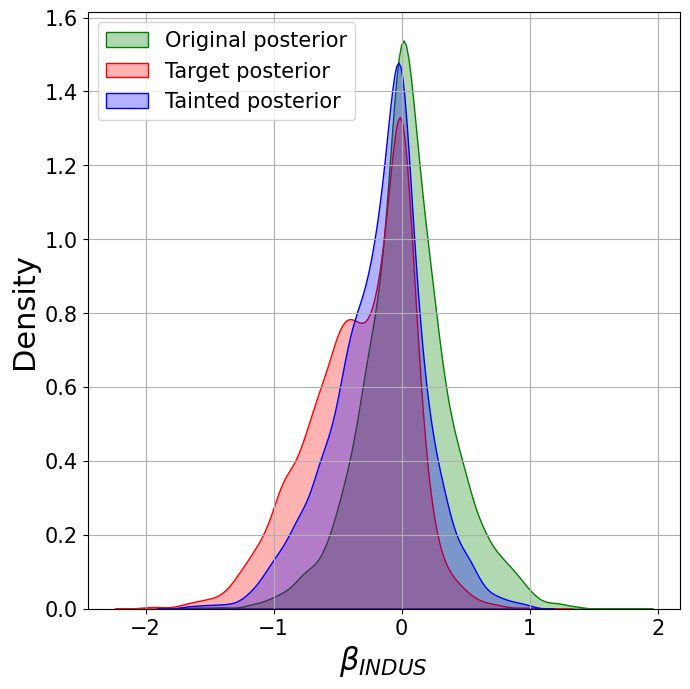}
        \caption*{$\beta_{\text{INDUS}}$}
    \end{minipage}
    
    \vspace{0.3cm}
    
    \begin{minipage}{0.15\textwidth}
        \centering
        \includegraphics[width=\textwidth]{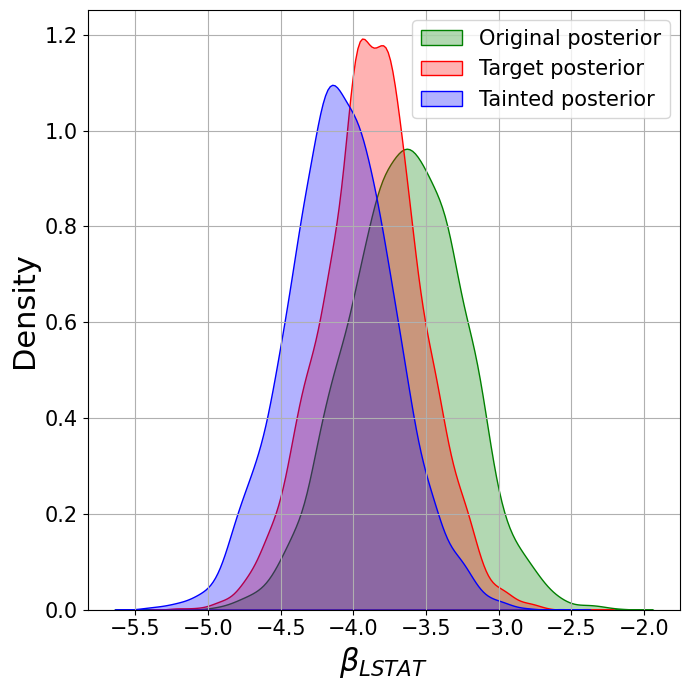}
        \caption*{$\beta_{\text{LSTAT}}$}
    \end{minipage}
    \begin{minipage}{0.15\textwidth}
        \centering
        \includegraphics[width=\textwidth]{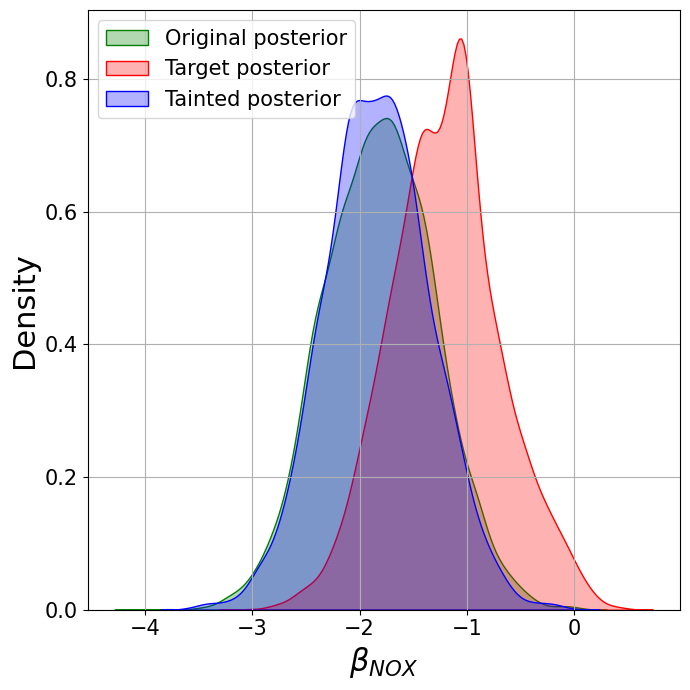}
        \caption*{$\beta_{\text{NOX}}$}
    \end{minipage}
    \begin{minipage}{0.15\textwidth}
        \centering
        \includegraphics[width=\textwidth]{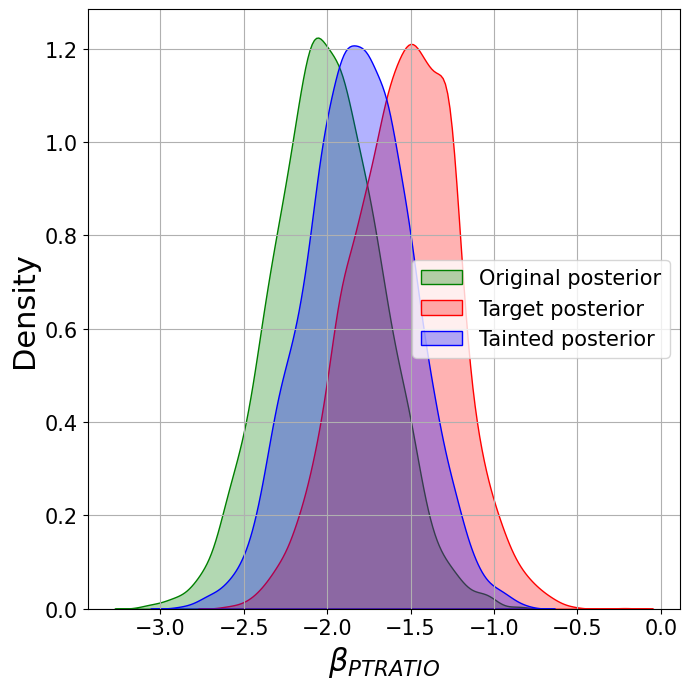}
        \caption*{$\beta_{\text{PTRATIO}}$}
    \end{minipage}
    \begin{minipage}{0.15\textwidth}
        \centering
        \includegraphics[width=\textwidth]{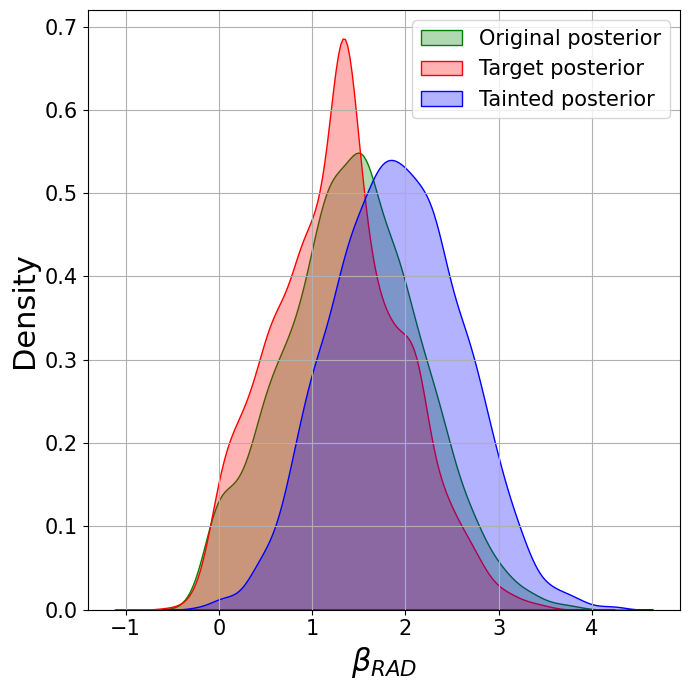}
        \caption*{$\beta_{\text{RAD}}$}
    \end{minipage}
    \begin{minipage}{0.15\textwidth}
        \centering
        \includegraphics[width=\textwidth]{figures/Boston_lin_reg/RM/marginals/beta_RM.png}
        \caption*{$\beta_{\text{RM}}$}
    \end{minipage}
    \begin{minipage}{0.15\textwidth}
        \centering
        \includegraphics[width=\textwidth]{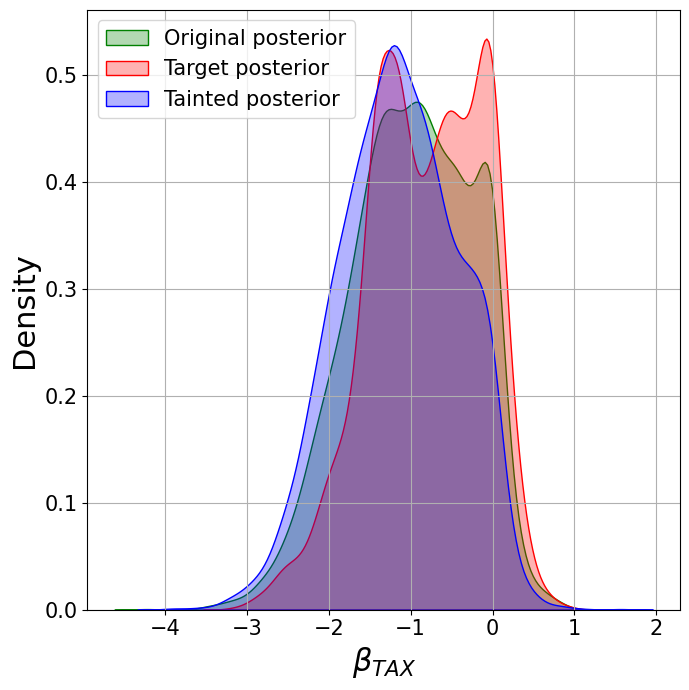}
        \caption*{$\beta_{\text{TAX}}$}
    \end{minipage}
    
    \vspace{0.3cm}
    
    \begin{minipage}{0.15\textwidth}
        \centering
        \includegraphics[width=\textwidth]{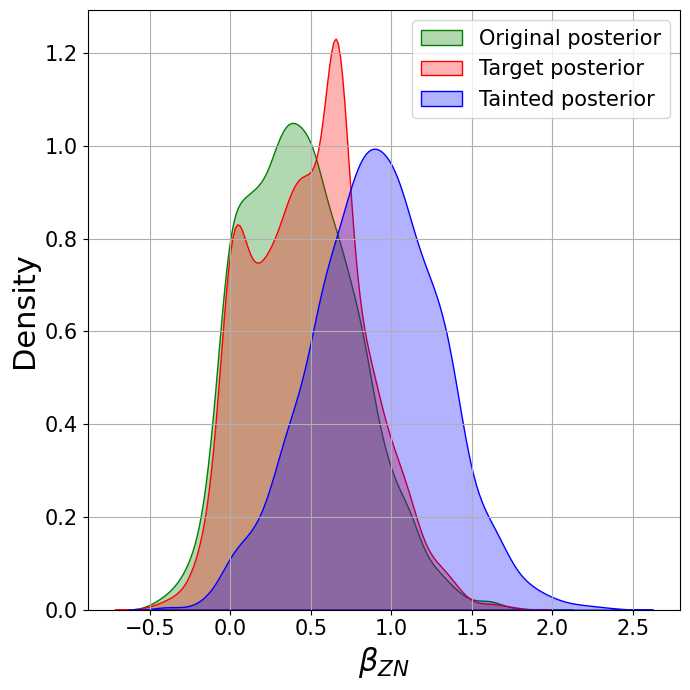}
        \caption*{$\beta_{\text{ZN}}$}
    \end{minipage}
    \begin{minipage}{0.15\textwidth}
        \centering
        \includegraphics[width=\textwidth]{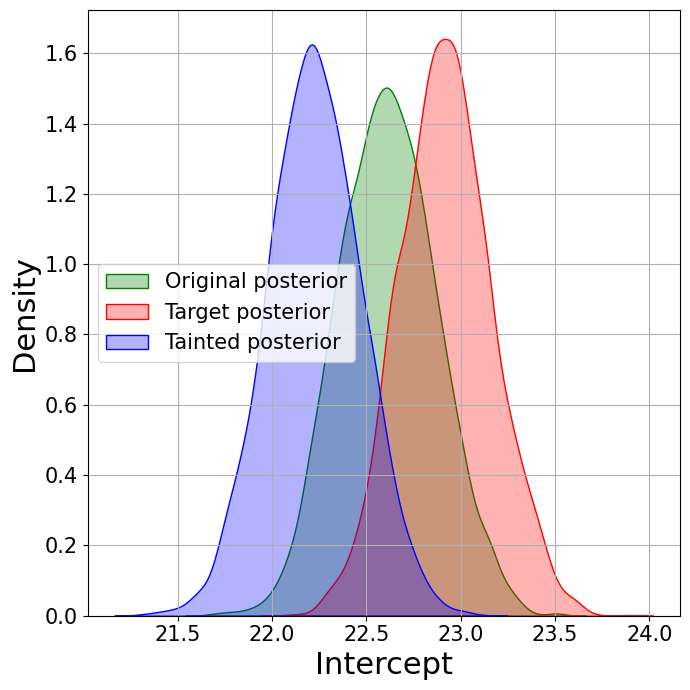}
        \caption*{Intercept}
    \end{minipage}
    \begin{minipage}{0.15\textwidth}
        \centering
        \includegraphics[width=\textwidth]{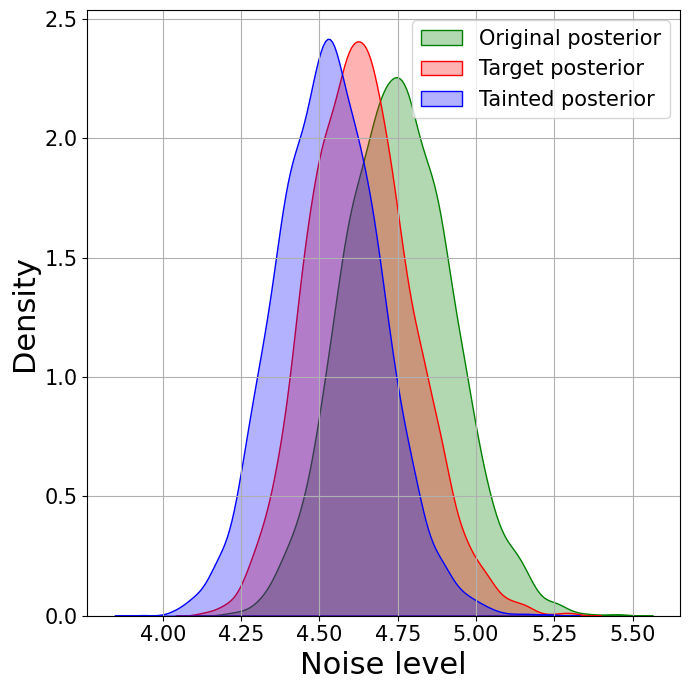}
        \caption*{$\sigma$}
    \end{minipage}
    \begin{minipage}{0.15\textwidth}
        \centering
        \includegraphics[width=\textwidth]{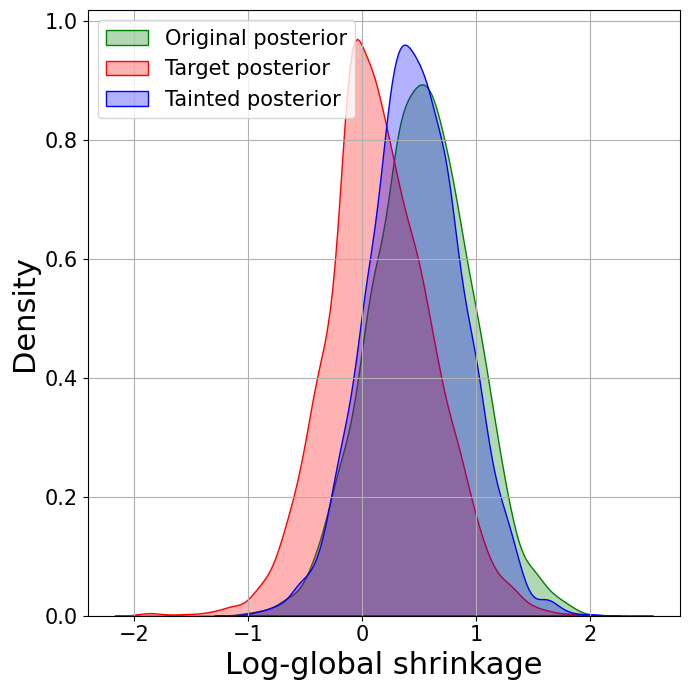}
        \caption*{\centering Log-global shrinkage}
    \end{minipage}
    \begin{minipage}{0.15\textwidth}
        \centering
        \includegraphics[width=\textwidth]{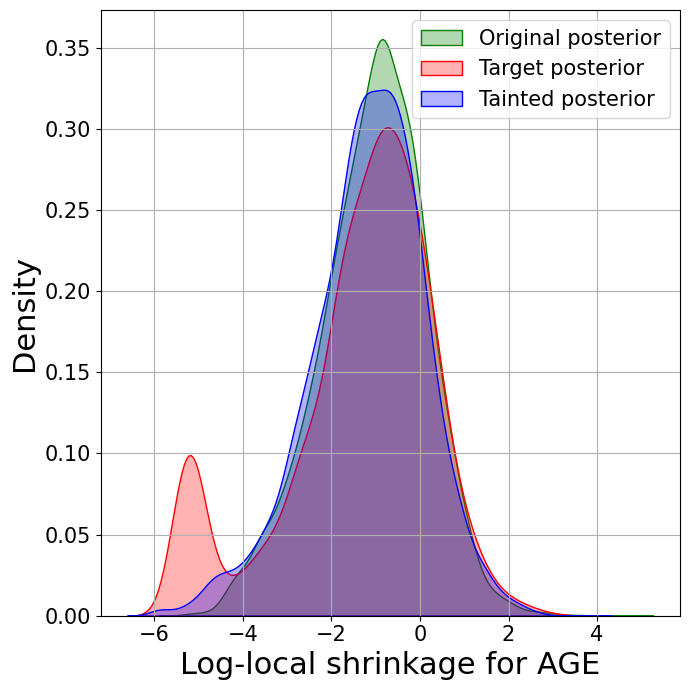}
        \caption*{\centering Log-local shrinkage for AGE}
    \end{minipage}
    \begin{minipage}{0.15\textwidth}
        \centering
        \includegraphics[width=\textwidth]{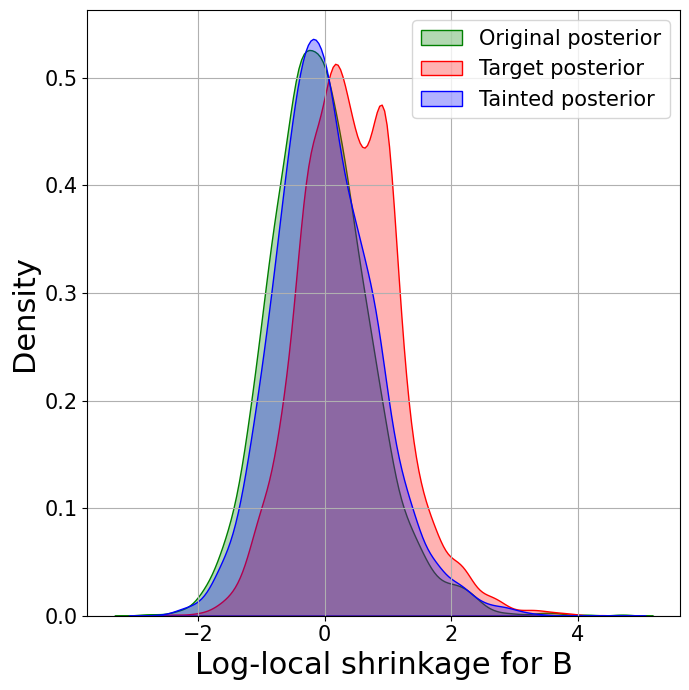}
        \caption*{\centering Local shrinkage for B}
    \end{minipage}
    
    \vspace{0.3cm}
    
    \begin{minipage}{0.15\textwidth}
        \centering
        \includegraphics[width=\textwidth]{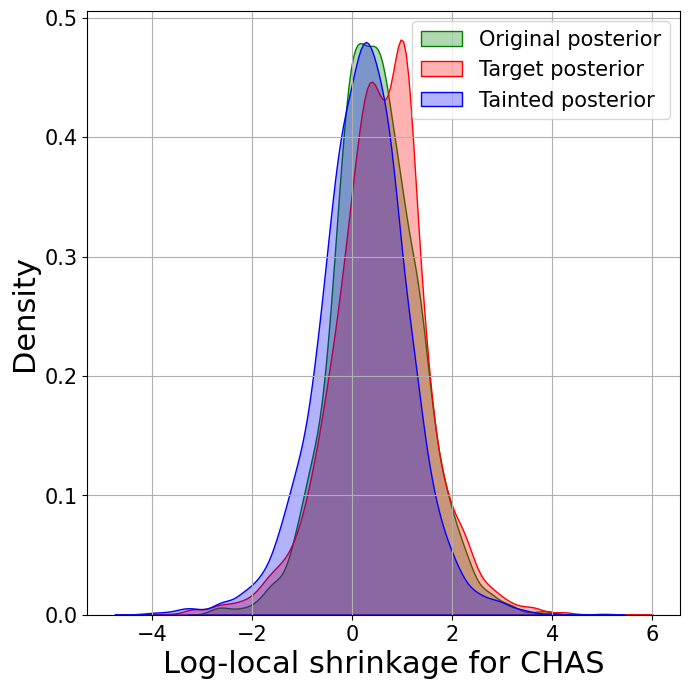}
        \caption*{\centering Log-local shrinkage for CHAS}
    \end{minipage}
    \begin{minipage}{0.15\textwidth}
        \centering
        \includegraphics[width=\textwidth]{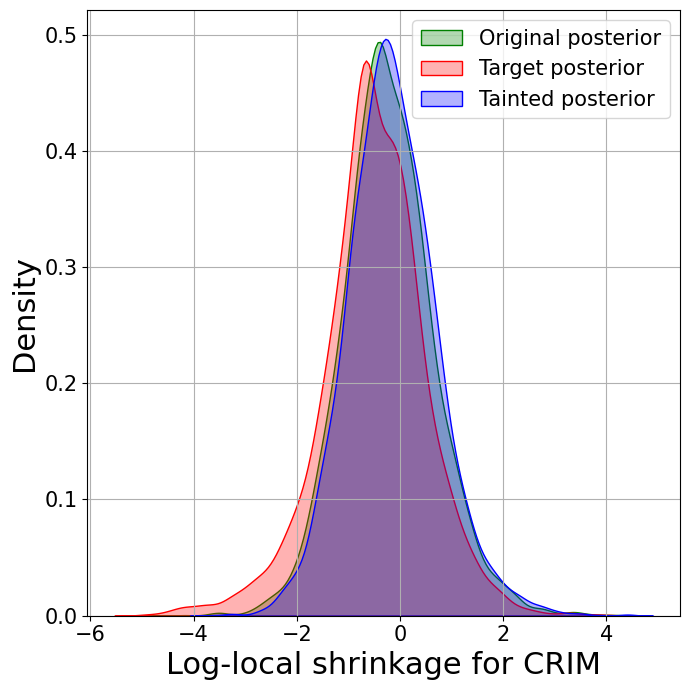}
        \caption*{\centering Log-local shrinkage for CRIM}
    \end{minipage}
    \begin{minipage}{0.15\textwidth}
        \centering
        \includegraphics[width=\textwidth]{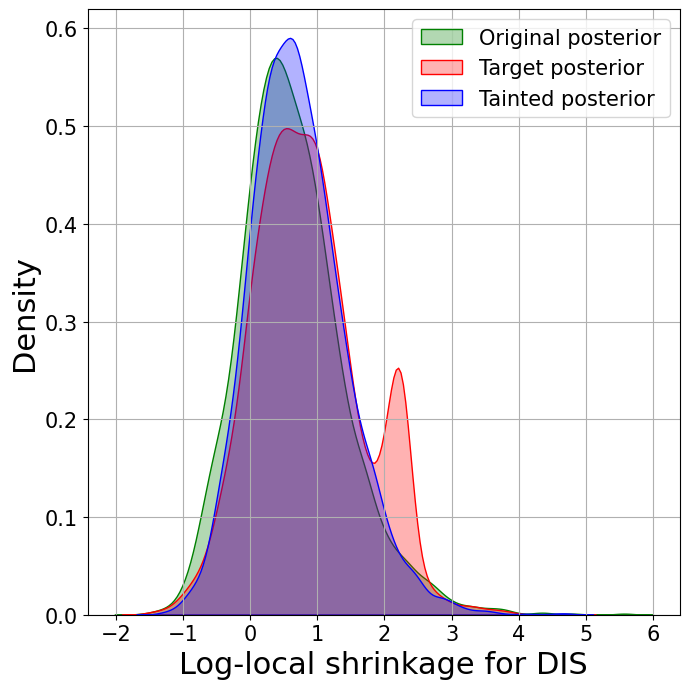}
        \caption*{\centering Log-local shrinkage for DIS}
    \end{minipage}
    \begin{minipage}{0.15\textwidth}
        \centering
        \includegraphics[width=\textwidth]{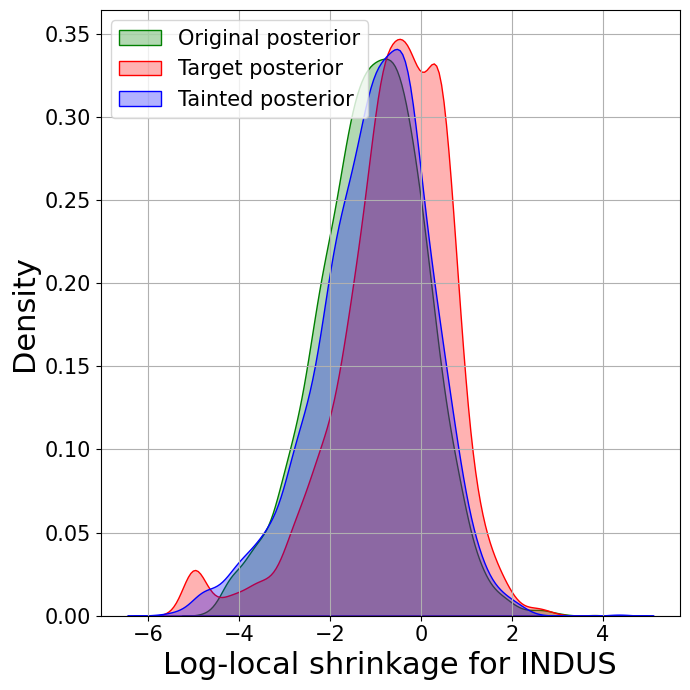}
        \caption*{\centering Log-local shrinkage for INDUS}
    \end{minipage}
    \begin{minipage}{0.15\textwidth}
        \centering
        \includegraphics[width=\textwidth]{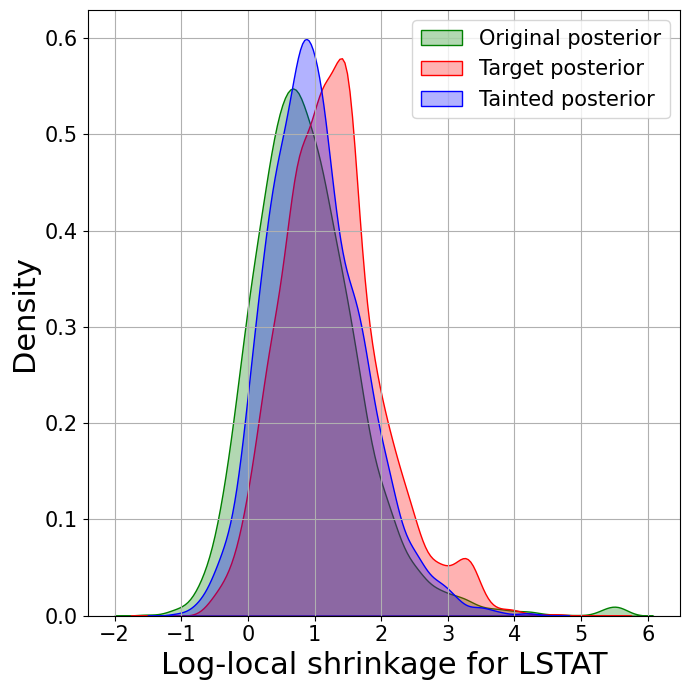}
        \caption*{\centering Log-local shrinkage for LSTAT}
    \end{minipage}
    \begin{minipage}{0.15\textwidth}
        \centering
        \includegraphics[width=\textwidth]{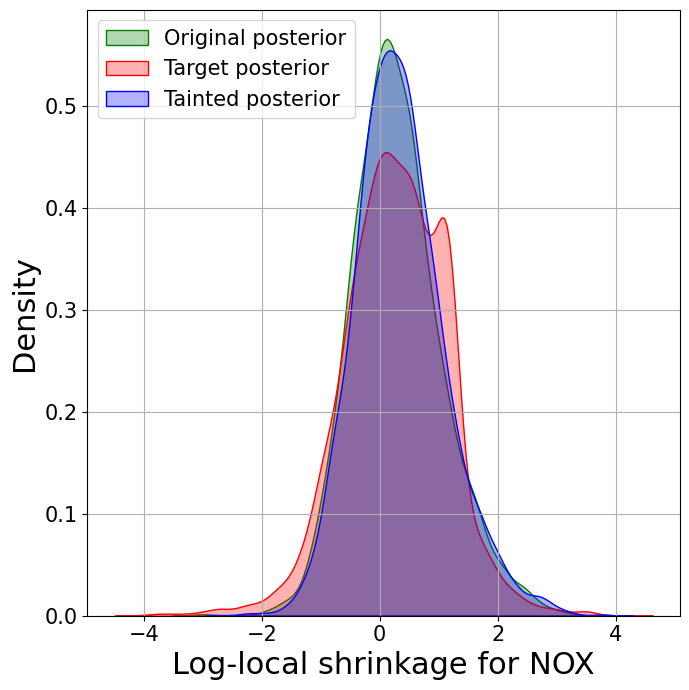}
        \caption*{\centering Log-local shrinkage for NOX}
    \end{minipage}
    
    \vspace{0.3cm}
    
    \begin{minipage}{0.15\textwidth}
        \centering
        \includegraphics[width=\textwidth]{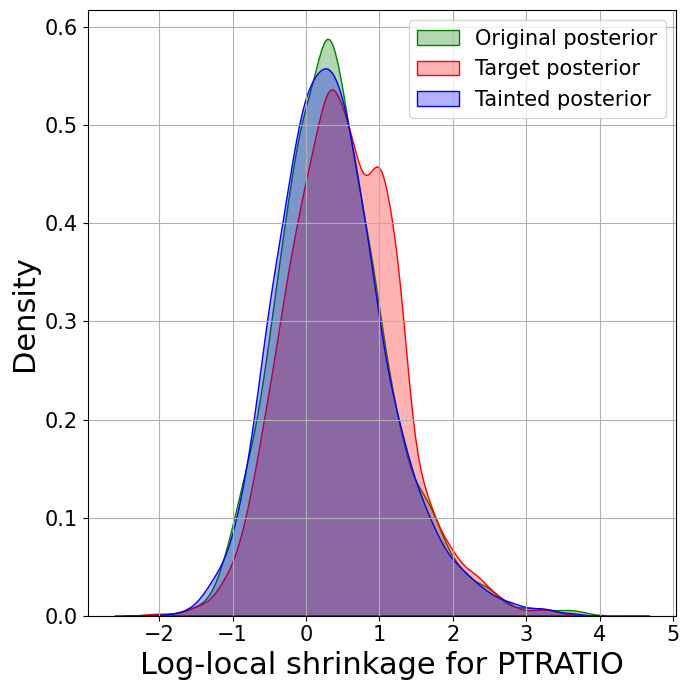}
        \caption*{\centering Log-local shrinkage for PTRATIO}
    \end{minipage}
    \begin{minipage}{0.15\textwidth}
        \centering
        \includegraphics[width=\textwidth]{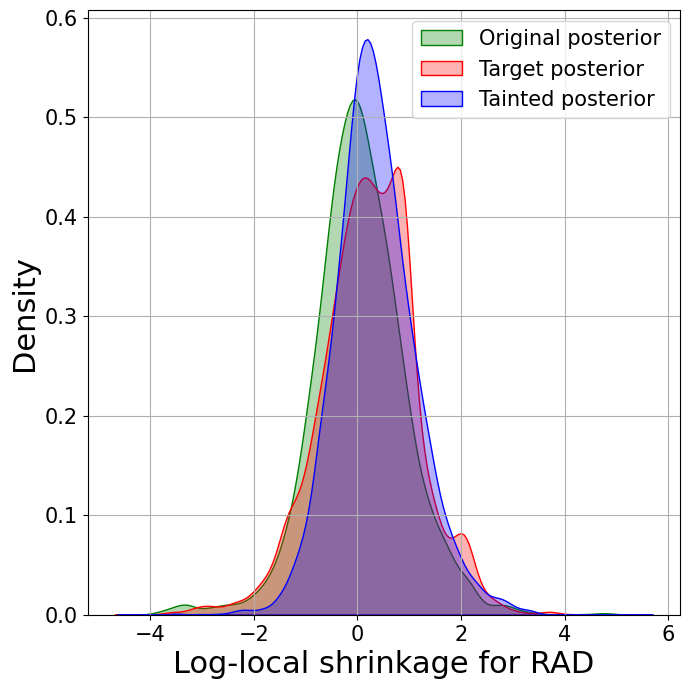}
        \caption*{\centering Log-local shrinkage for RAD}
    \end{minipage}
    \begin{minipage}{0.15\textwidth}
        \centering
        \includegraphics[width=\textwidth]{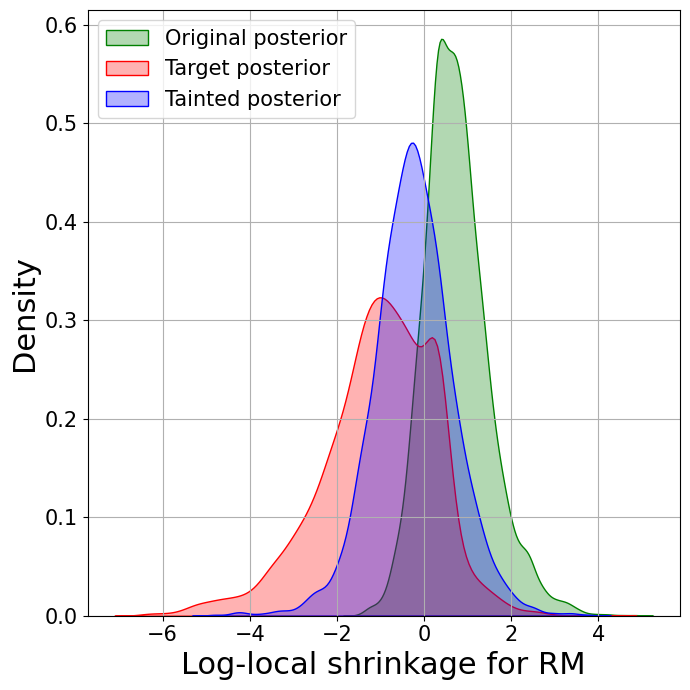}
        \caption*{\centering Log-local shrinkage for RM}
    \end{minipage}
    \begin{minipage}{0.15\textwidth}
        \centering
        \includegraphics[width=\textwidth]{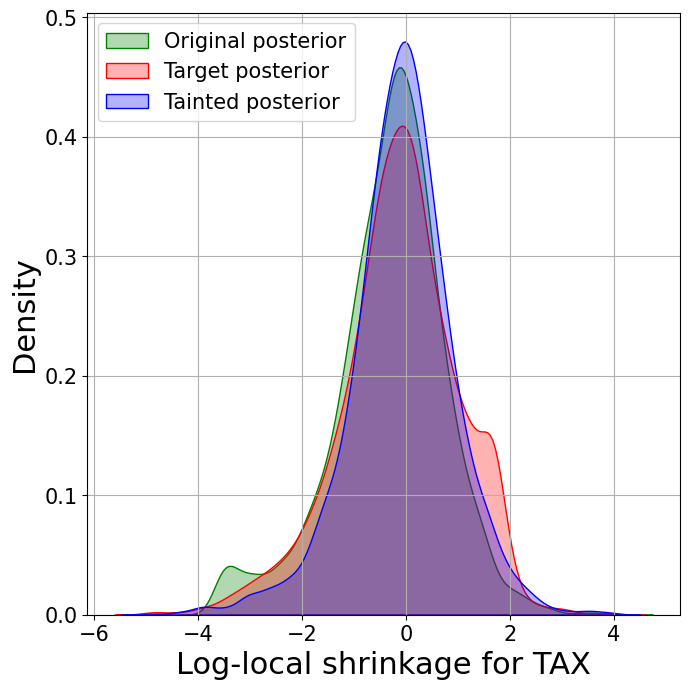}
        \caption*{\centering Log-local shrinkage for TAX}
    \end{minipage}
    \begin{minipage}{0.15\textwidth}
        \centering
        \includegraphics[width=\textwidth]{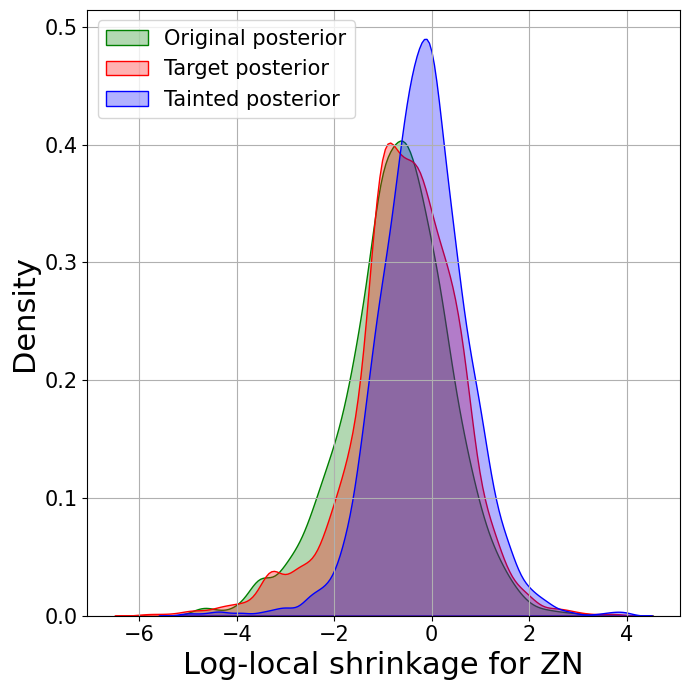}
        \caption*{\centering Log-local shrinkage for ZN}
    \end{minipage}
    
    \caption{Original, target and induced marginal posterior distributions for Horseshoe regression parameters under poisoning attack with $B=30$ calculated via 2O-ISCD heuristic.} \label{fig:Horseshoe_allmargposteriors}
    
\end{figure}

We also computed attacks against a Bayesian linear regression model with a NIG prior. Figure \ref{fig:NIG_vs_Horseshoe} compares the removed and duplicated data points, along with the original, target, and induced marginal posterior for $\beta_{RM}$, for attacks found by the ISCD heuristic with $B=30$ under both the NIG prior and the horseshoe prior. Interestingly, in both cases, similar data points were replicated and removed. This suggests that when ample data is available, the attacker may not need to know the true prior. By relying solely on the likelihood, the attacker can craft effective attacks using an assumed prior, extending the method to a gray-box setting.

\begin{figure*}[h]
\begin{multicols}{2}
    \centering
    \includegraphics[width=0.7\linewidth]{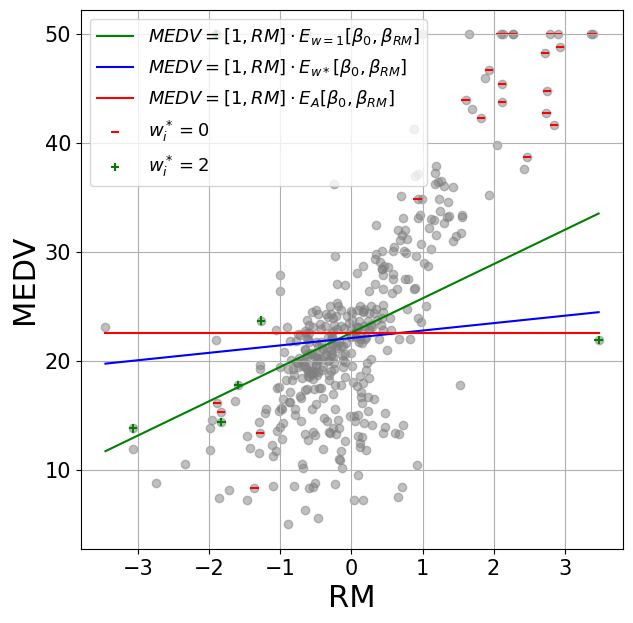}
    \subcaption{Data modifications for NIG}\par
    \includegraphics[width=0.7\linewidth]{figures/Boston_lin_reg/RM/datapoints_2O-ISCD_B=30_HS.png}
    \subcaption{Data modifications for Horseshoe}\par
\end{multicols}
\begin{multicols}{2}
    \centering
    \includegraphics[width=0.7\linewidth]{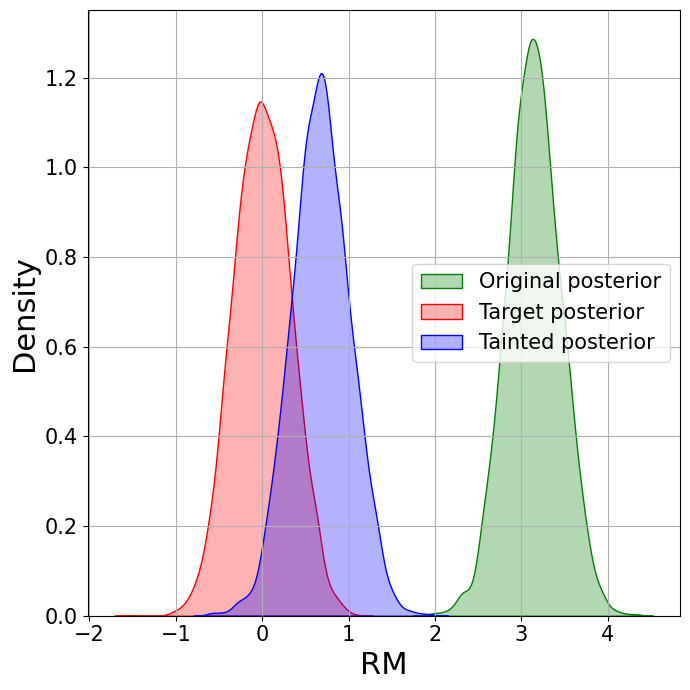}
    \subcaption{$\beta_{RM}$ posterior for NIG}\par
    \includegraphics[width=0.7\linewidth]{figures/Boston_lin_reg/RM/marginals/beta_RM.png}
    \subcaption{$\beta_{RM}$ posterior for Horseshoe}\par
\end{multicols}
\caption{Comparison between NIG and Horseshoe priors attacked with 2O-ISCD and $B=30$.}
\label{fig:NIG_vs_Horseshoe}
\end{figure*}
%



\clearpage
\subsubsection{Effects of different priors on attack performance}

In this section, we investigate the effects of different priors on various performance metrics. Using the same setting as in Section 6.2 (linear regression on the Boston housing dataset, where the poisoning goal is to shift the posterior of the parameter corresponding to the covariate ``number of rooms" toward 0), we computed attacks under different priors:

\begin{itemize}
    \item The horseshoe prior (used in Section 6.2).
    \item The ``non-informative" normal-inverse gamma prior, with a high prior variance (used in Section 6.2).
    \item An informative NIG prior, with a prior mean of 3.0 and prior variance 0.1.
\end{itemize}

For these experiments, we computed attacks using the 2O-ISCD heuristic, limiting the number of replications to two and allowing up to 30 data manipulations. We evaluated the effectiveness of the attacks using two metrics: (1) the KL divergence between the adversarial posterior and the induced posterior, and (2) the induced posterior expectation of the parameter of interest. Each attack (computed with a specific prior) was evaluated against three different models: one using the same prior as the attack (white-box setting) and two others using different priors (gray-box setting). Table \ref{tab:kl_divergence} presents the KL divergence between the adversarial and induced posteriors for each attack scenario:

\begin{table}[h]
    \centering
    \begin{tabular}{lccc}
        \toprule
        & Horseshoe & NIG & NIG Informative \\
        \midrule
        No attack & 96.05 & 93.15 & 95.06 \\
        Attack with Horseshoe prior & 22.16 $\pm$ 0.87 & 16.82 $\pm$ 0.63 & 17.63 $\pm$ 0.64 \\
        Attack with NIG prior & 23.27 $\pm$ 1.14 & 15.66 $\pm$ 0.18 & 16.32 $\pm$ 0.22 \\
        Attack with informative NIG prior & 24.03 $\pm$ 0.80 & 15.59 $\pm$ 0.29 & 16.16 $\pm$ 0.26 \\
        \bottomrule
    \end{tabular}
    \caption{KL divergence between adversarial and induced posteriors.}
    \label{tab:kl_divergence}
\end{table}

As expected, the KL divergence is lowest in the white-box setting (when the prior used to compute the attack matches the prior used to evaluate it). However, there is noticeable attack transferability in the gray-box setting: even with a misspecified prior, the attacks significantly reduce the KL divergence. Table \ref{tab:posterior_mean} shows the induced posterior mean under each attack scenario:

\begin{table}[h]
    \centering
    \begin{tabular}{lccc}
        \toprule
        & Horseshoe & NIG & NIG Informative \\
        \midrule
        No attack & 3.21 & 3.15 & 3.14 \\
        Attack with Horseshoe prior & 0.74 $\pm$ 0.07 & 0.75 $\pm$ 0.07 & 0.88 $\pm$ 0.06 \\
        Attack with NIG prior & 0.57 $\pm$ 0.06 & 0.61 $\pm$ 0.06 & 0.74 $\pm$ 0.05 \\
        Attack with informative NIG prior & 0.51 $\pm$ 0.03 & 0.54 $\pm$ 0.03 & 0.67 $\pm$ 0.03 \\
        \bottomrule
    \end{tabular}
    \caption{Induced posterior mean under different attack scenarios.}
    \label{tab:posterior_mean}
\end{table}

Two key observations are in order:

\begin{enumerate}
    \item \textbf{Attack transferability}: Regardless of the prior used to create the attack, the posterior mean of the parameter is significantly shifted toward zero. This suggests a level of generalization of the proposed attacks beyond the white-box setting.
    \item \textbf{Robustness of informative priors}: While somewhat expected, the informative prior demonstrates greater robustness to data deletion and replication attacks. This finding contributes to the ongoing debate in the Bayesian community about the utility of informative priors.
\end{enumerate}

To validate the insight in observation 2, we repeated the experiment with an even more informative prior (setting the variance to 0.01). The results for the induced posterior mean are shown in Table \ref{tab:posterior_mean_stronger}, further highlighting the robustness of more informative priors:

\begin{table}[ht]
    \centering
    \begin{tabular}{lccc}
        \toprule
        & Horseshoe & NIG & NIG Informative \\
        \midrule
        No attack & 3.21 & 3.15 & 3.10 \\
        Attack with Horseshoe prior & 0.74 $\pm$ 0.07 & 0.75 $\pm$ 0.07 & 1.59 $\pm$ 0.04 \\
        Attack with NIG prior & 0.57 $\pm$ 0.06 & 0.61 $\pm$ 0.06 & 1.49 $\pm$ 0.04 \\
        Attack with informative NIG prior & 0.43 $\pm$ 0.05 & 0.46 $\pm$ 0.05 & 1.38 $\pm$ 0.03 \\
        \bottomrule
    \end{tabular}
    \caption{Induced posterior mean under stronger informative priors.}
    \label{tab:posterior_mean_stronger}
\end{table}

\newpage

\subsection{Mexico microcredit dataset}


In this section, we include further experiments on the Mexico microcredit RCT example from Section 6.3.

Recall that, in the clean data scenario, the posterior mean for the slope parameter (ATE) $\beta_1$ is $-4.71$ with a standard deviation of $6.02$. The adversarial posterior utilized is defined using a modified dataset $\{(x_i, y_i + \lambda \mathds{1}_{\{x_i=1\}})\}_{i=1}^n$.

Using the 2O-ISCD attack with $\lambda = 100$, $B = 1$, and $L = 1$, the posterior is modified as shown in Figure \ref{fig:microcredit_B=1_lambda=100}. This attack deleted a single data point, resulting in a posterior for $\beta_1$ with a mean of $-3.67$ (slightly higher than the original) and a standard deviation of $5.62$. The most significant change was a slight decrease in the noise level $\sigma$.

\begin{figure*}[h]
\begin{multicols}{4}
    \includegraphics[width=\linewidth]{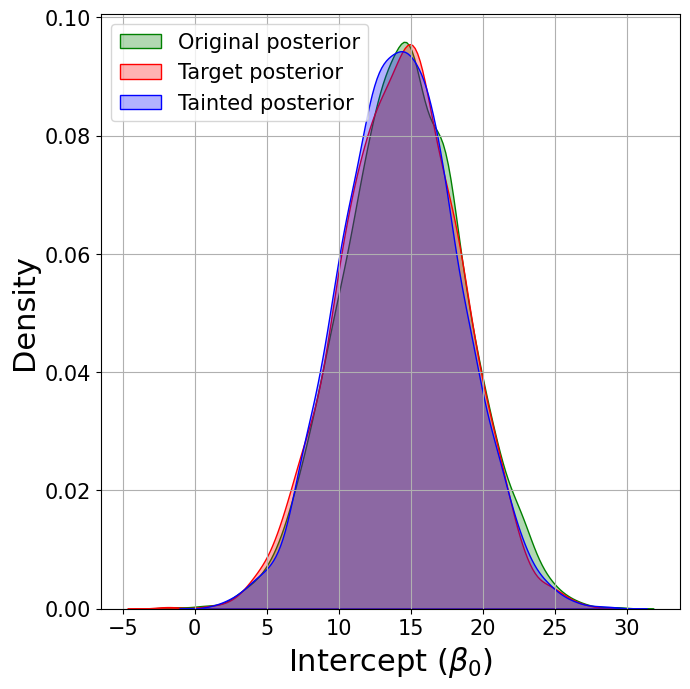}
    \subcaption{Posterior for $\beta_0$}\par
    \includegraphics[width=\linewidth]{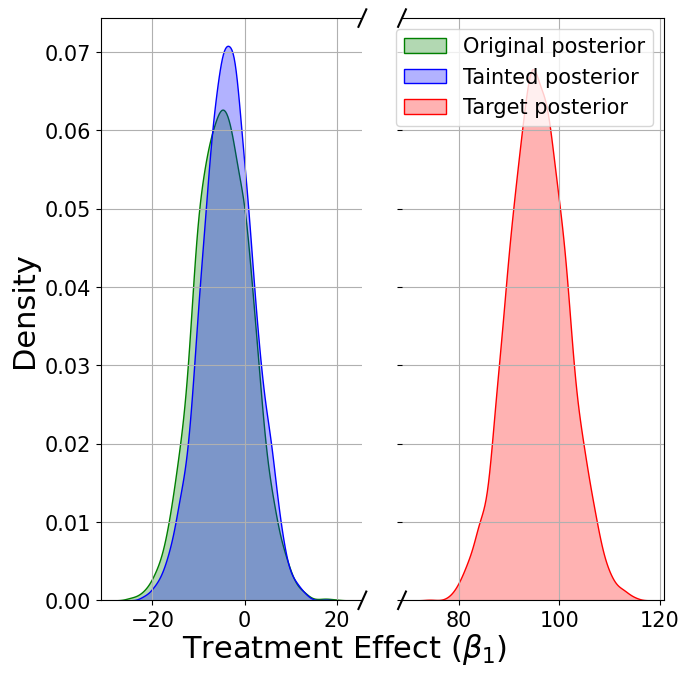}
    \subcaption{Posterior for $\beta_1$}\par
    \includegraphics[width=\linewidth]{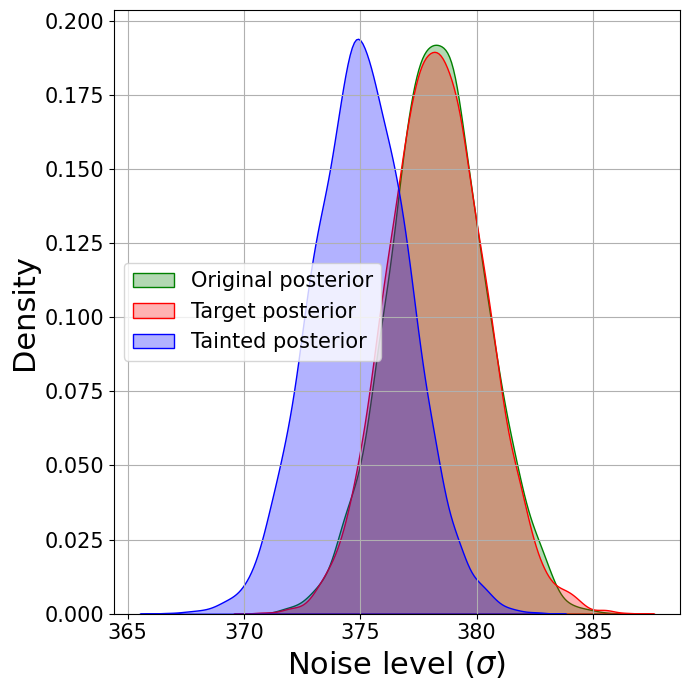}
    \subcaption{Posterior for $\sigma$}\par
    \includegraphics[width=\linewidth]{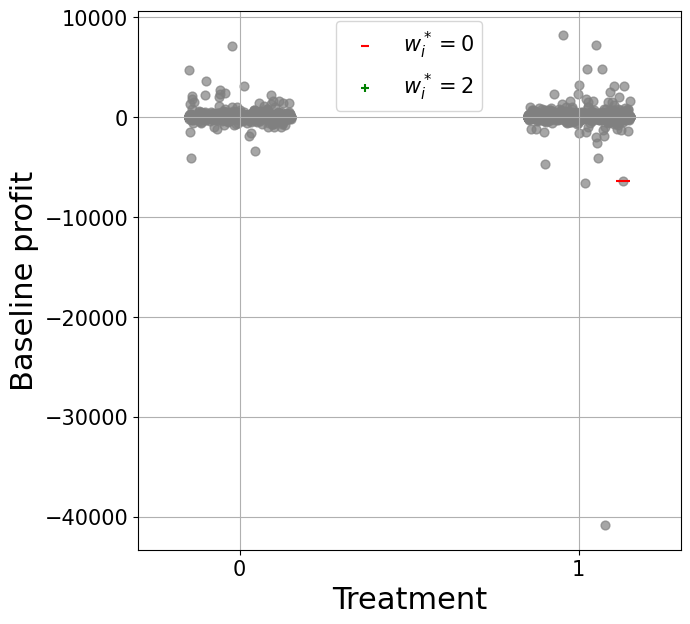}
    \subcaption{Deleted sample}\par
\end{multicols}
\caption{Attack on the microcredit linear regression for $B=L=1$, $\lambda=100$}
\label{fig:microcredit_B=1_lambda=100}
\end{figure*}

It is important to note that there is an extreme outlier in the treated group ($x_i=1$) that, if solely removed, significantly shifts the posterior for the treatment effect toward positive values. One might question why this point was not chosen by our attack. The answer is straightforward: whereas removing the outlier indeed causes a notable shift in the marginal posterior for the treatment effect in the correct direction, it also  substantially impacts the marginal posterior for $\sigma$, ultimately resulting in a higher KL divergence for the joint posteriors. However, when we increase $\lambda$ to 10,000, keeping the other parameters unchanged, the posterior is modified as shown in Figure \ref{fig:microcredit_B=1_lambda=10000}. Therein, the outlier is deleted, causing a significant impact on the marginal posterior for $\sigma$.  However, since the target marginal posterior for the treatment effect is now strongly positive, deleting the outlier becomes worthwhile in terms of KL divergence reduction, bringing the marginal posterior for $\beta_1$ closer to the adversarial target, despite the larger impact on $\sigma$.

\begin{figure*}[h]
\begin{multicols}{4}
    \includegraphics[width=\linewidth]{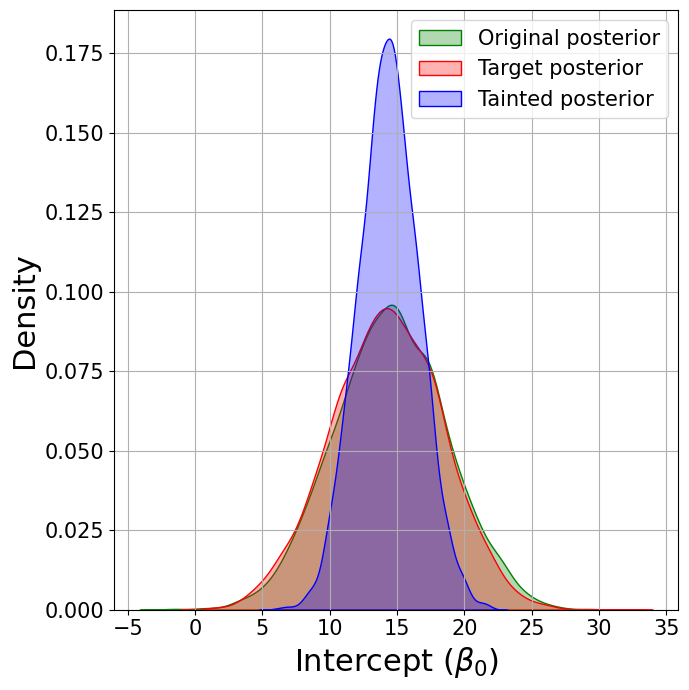}
    \subcaption{Posterior for $\beta_0$}\par
    \includegraphics[width=\linewidth]{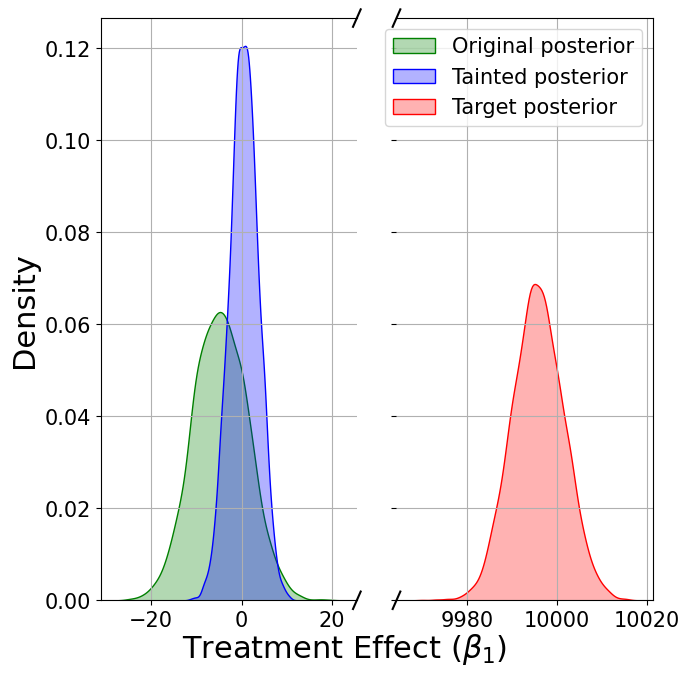}
    \subcaption{Posterior for $\beta_1$}\par
    \includegraphics[width=\linewidth]{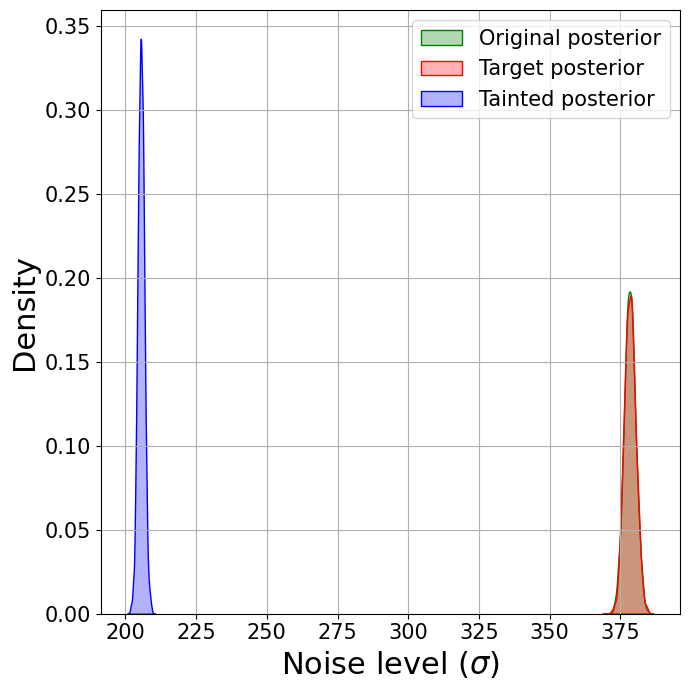}
    \subcaption{Posterior for $\sigma$}\par
    \includegraphics[width=\linewidth]{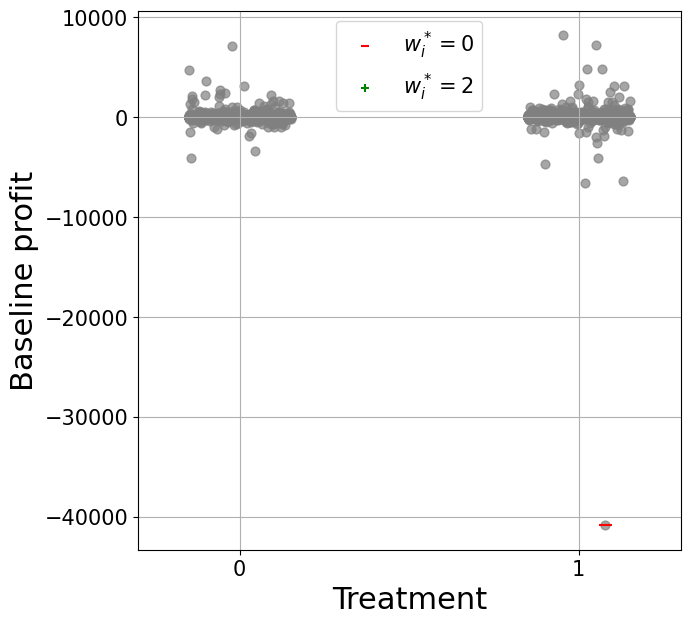}
    \subcaption{Deleted sample}\par
\end{multicols}
\caption{Attack on the microcredit linear regression for $B=L=1$, $\lambda=10000$}
\label{fig:microcredit_B=1_lambda=10000}
\end{figure*}

\newpage

Finally, Figure \ref{sfig:microcredit_B=20_lambda=10000} shows the obtained posteriors for Section 6.3 of the main text (i.e., an attack parameterized with $B = 20$, $L = 2$  and $\lambda = 10,000$). The new mean and standard deviation for $\beta_1$ are $6.28$ and $3.22$ respectively, with a 95\% credible interval of $[0.02,\ 12.43]$, that no longer includes zero. Therefore, if the defender’s decision about the impact of microcredit is based on whether the credible interval contains zero, it can be observed that, by manipulating only 0.12 \% of the data, an attacker can induce an incorrect conclusion.

\begin{figure*}[h]
\begin{multicols}{4}
    \includegraphics[width=\linewidth]{figures/mexico_microcredit/B=20_2O-ISCD/2O-ISCD_B=20_L=2_beta_0.png}
    \subcaption{Posterior for $\beta_0$}\par
    \includegraphics[width=\linewidth]{figures/mexico_microcredit/B=20_2O-ISCD/2O-ISCD_B=20_L=2_beta_1_break_intervals.png}
    \subcaption{Posterior for $\beta_1$}\par
    \includegraphics[width=\linewidth]{figures/mexico_microcredit/B=20_2O-ISCD/2O-ISCD_B=20_L=2_sigma.png}
    \subcaption{Posterior for $\sigma$}\par
    \includegraphics[width=\linewidth]{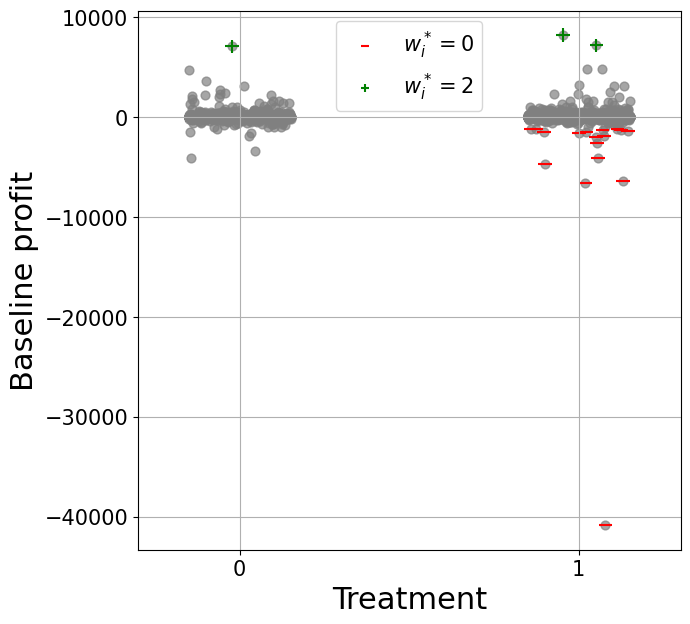}
    \subcaption{Modified samples}\par
\end{multicols}
\caption{Attack on the microcredit linear regression for $B=20$, $L=2$, $\lambda=10000$}
\label{sfig:microcredit_B=20_lambda=10000}
\end{figure*}

In summary, this example highlights the surgical precision of our attacks and their adaptability to different situations by selecting the appropriate adversarial target.

\clearpage

\subsection{Attacks to Bayesian logistic regression for spam classification}

In this section, we present attacks targeting a Bayesian logistic regression model for spam classification. We use the SMS spam classification dataset \citep{sms_spam_collection_228}, sub-sampling it to create a balanced dataset with 653 texts labeled as \textit{Spam} ($Y=1$) and 653 texts labeled as \textit{Ham} ($Y=0$). The dataset is split into a training set of 979 SMSs and a test set of 327 SMSs.
The defender uses a linear logistic regression model, where the features are binary variables indicating the presence or absence of specific words in each SMS. The set of words considered is selected from the training set (before any poisoning) by choosing words present in at least 50 SMSs, resulting in $d=48$ covariates.
To evaluate the model's performance, we estimate, via MCMC, the posterior predictive probability $\pi(Y_i | X_i, D)$ for each sample in both the training and test sets. Using a 0.5 threshold on this probability to classify an SMS as spam or ham, the model's accuracies on the training and test sets are 0.91 and 0.86, respectively.

The attacks involve poisoning the training set. Among the words significantly associated with the Spam class, the coefficients for the words \textit{your} and \textit{send} have posterior means of 1.01 and 0.58, standard deviations of 0.42 and 0.68, and probabilities of being negative of 0.006 and 0.19, respectively. These words appear in 191 and 51 SMSs in the training set. 
We perform one experiment for each of these words, with the attacker's goal being to reverse the inference on the target word, shifting the majority of the posterior mass to negative values. The adversarial target distribution $\pi_A(\beta)$ is defined as a multivariate Gaussian through the following steps:

\begin{itemize}
    \item Compute a Laplace approximation $\mathcal{N}(\mu_L, \Sigma_L)$ of the true posterior $\pi(\beta | D)$.
    \item Initialize $\mu_A \gets \mu_L$, and modify the value associated with the target word: $\mu_{A, \text{word}} \gets -\mu_{L, \text{word}}$.
    \item Define $\pi_A(\beta) \coloneq \mathcal{N}(\mu_A, \Sigma_L)$.
\end{itemize}

\begin{figure*}[h]
\begin{multicols}{3}
\centering
    \includegraphics[width=0.9\linewidth]{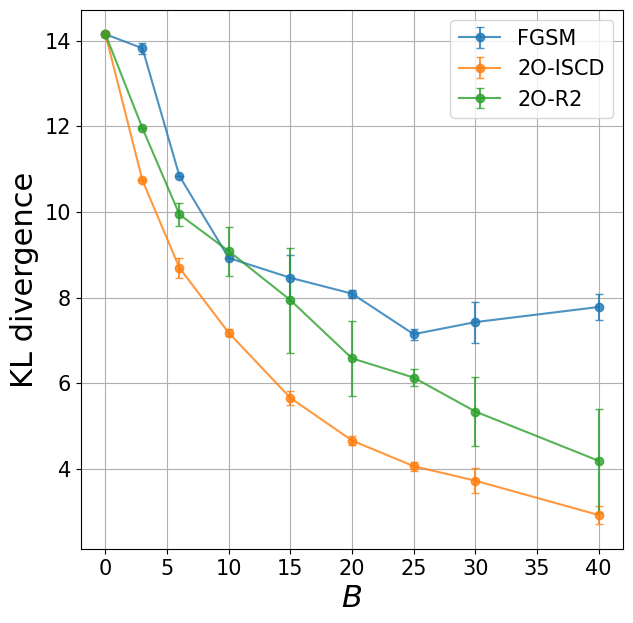}
    \subcaption{KL divergence vs $B$}\par
    
    \includegraphics[width=\linewidth]{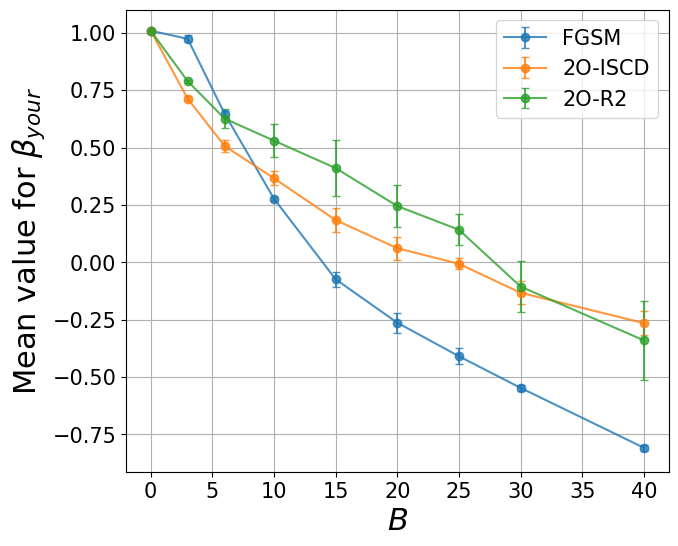}
    \subcaption{ $\mathbb{E}_{w^*}\left[\beta_{\text{your}}\right]$ vs $B$}\par
    
    \includegraphics[width=\linewidth]{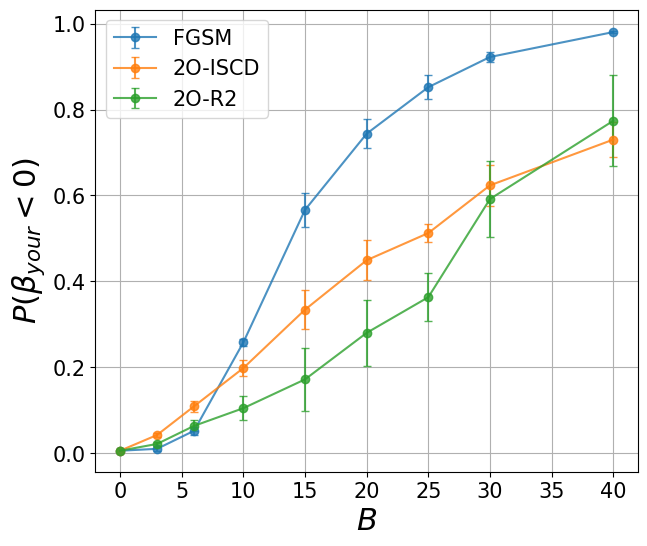}
    \subcaption{$\pi_{w^*}(\beta_{\text{your}}<0)$ vs $B$}\par
\end{multicols}
\caption{Metrics for the attack on the SMS spam classifier targeting the word \textit{your}}
\label{fig:SMS_metrics_your}
\end{figure*}

\begin{figure*}[h]
\begin{multicols}{3}
    \includegraphics[width=\linewidth]{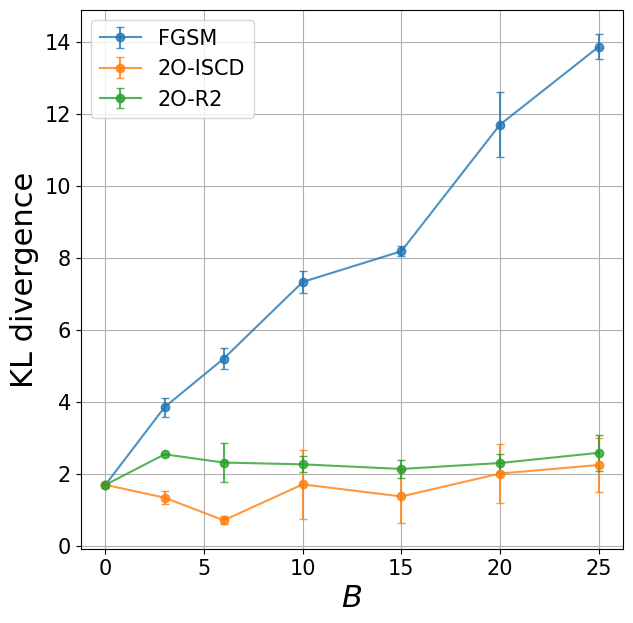}
    \subcaption{KL divergence vs $B$}\par
    \includegraphics[width=\linewidth]{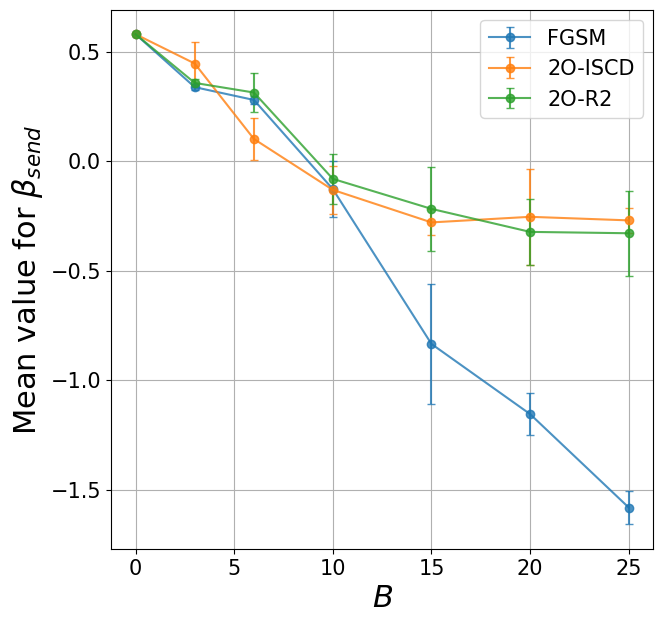}
    \subcaption{$\mathbb{E}_{w^*}\left[\beta_{\text{send}}\right]$ vs $B$}\par
    \includegraphics[width=\linewidth]{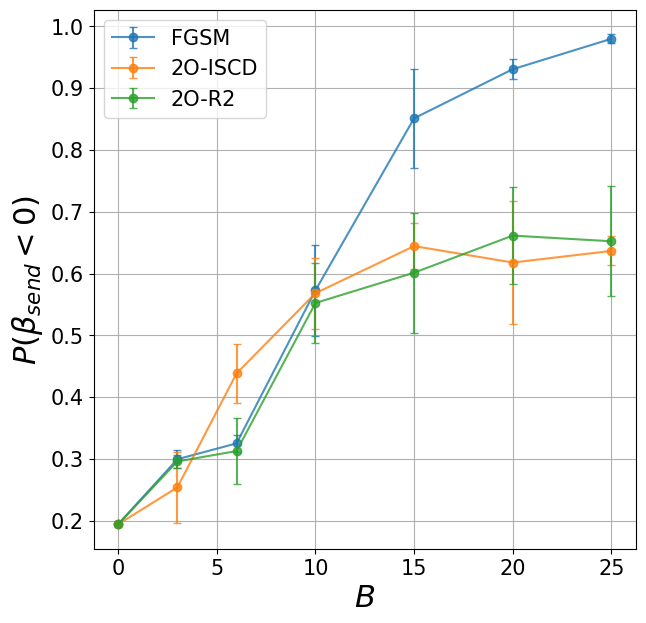}
    \subcaption{$\pi_{w^*}(\beta_{\text{send}}<0)$ vs $B$}\par
\end{multicols}
\caption{Metrics for the attack on the SMS spam classifier targeting the word \textit{send}}
\label{fig:SMS_metrics_send}
\end{figure*}

We ran the FGSM, 2O-ISCD, and 2O-R2 heuristics for both attacks across various intensities, repeating each experiment $N = 5$ times. For the target words \textit{your} and \textit{send}, Figures \ref{fig:SMS_metrics_your} and \ref{fig:SMS_metrics_send} show the KL divergence, the posterior mean of the attacked coefficient, and its probability of being negative as functions of $B$. The KL divergences are computed using Laplace approximations of the induced posterior.

In both experiments, FGSM produced the least effective attacks in terms of KL divergence. However, based on the last two metrics, FGSM outperformed the other two heuristics. To make $ \pi_{w^*}(\beta_{\text{target word}} < 0) $ exceed 0.9, FGSM required only 25 modifications for \textit{your}, corresponding to 12\% of the SMS messages containing the word and 2.2\% of the entire training set. For the word \textit{send}, 10 modifications were needed, affecting 15\% of the messages containing the word and 0.9\% of the training set. Neither 2O-ISCD nor 2O-R2 reached this threshold. 
This is likely due to the fact that, while FGSM is highly effective at shifting the marginal posterior of the parameter associated with the target word, it also significantly alters the posteriors of other parameters. As a result, this broader disruption causes an overall higher KL divergence compared to the more targeted modifications made by the other heuristics, which maintain more targeted attacks.



\begin{figure*}[h]
\begin{multicols}{2}
    \includegraphics[width=0.8\linewidth]{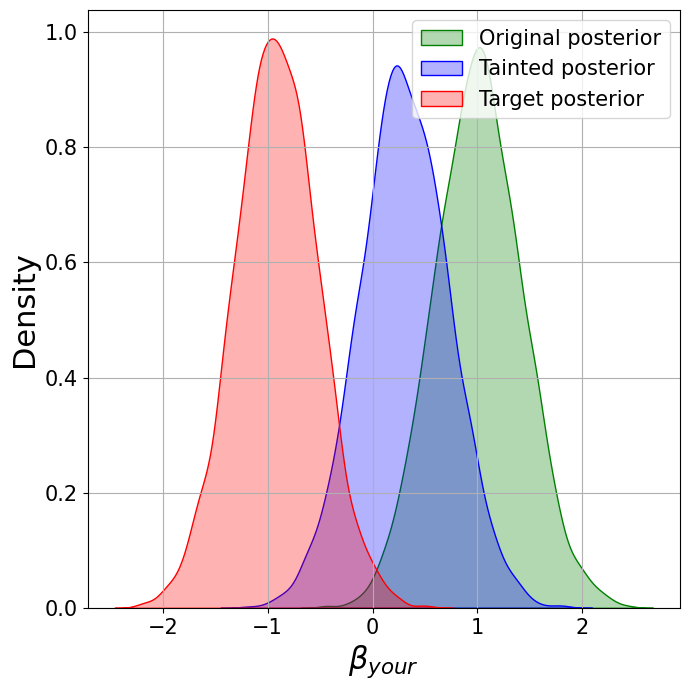}
    \subcaption{Targeting \textit{your}}\par
    \includegraphics[width=0.8\linewidth]{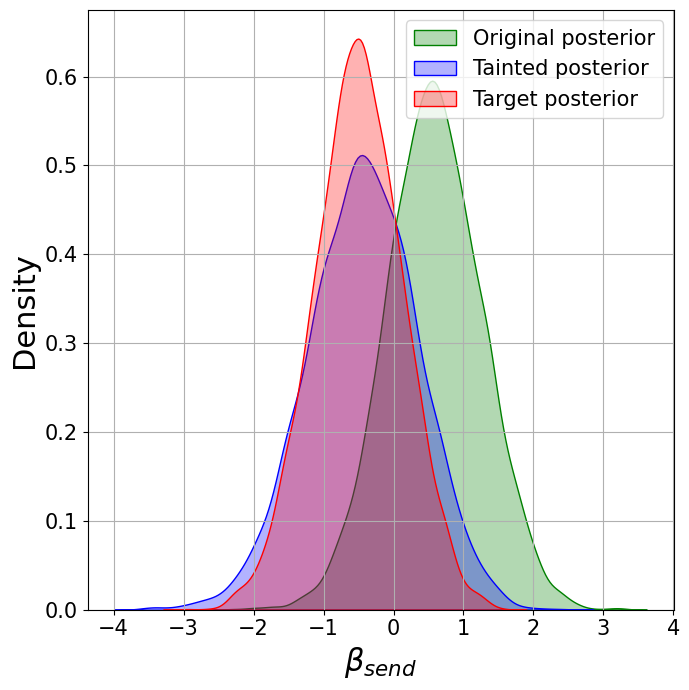}
    \subcaption{Targeting \textit{send}}\par
\end{multicols}
\caption{Tainted posteriors both obtained as attacks with 2O-R2 and $B=15$}
\label{fig:SMS_post_target_word}
\end{figure*}

\begin{table}[h!]
\centering
    \begin{multicols}{2}
    
    \begin{tabular}{|c||*{2}{c|}}
    \hline
       & with \textit{your}  & without \textit{your} \\
    \hline
    \hline
     Ham  & 30 + \textcolor{green}{2} - 0 = 32 & 466 + 0 - 0 = 466\\
    \hline
     Spam  & 161 + 0 - \textcolor{red}{8} = 153 & 322 + \textcolor{green}{1} - 0 = 323 \\
    \hline
    \end{tabular}
    \subcaption{2O-R2 with $B=15$, targeting $\beta_{your}$}
    \label{tab:words_your}
    \par
    
    \begin{tabular}{|c||*{2}{c|}}
    \hline
       & with \textit{send}  & without \textit{send} \\
    \hline
    \hline
     Ham  & 12 + \textcolor{green}{1} - 0 = 13 & 484 + 0 - 0 = 484 \\
    \hline
     Spam  & 39 + 0 - \textcolor{red}{7}  = 32 & 444 + \textcolor{green}{6} - 0 = 438\\
    \hline
    \end{tabular}
    \subcaption{2O-R2 with $B=15$, targeting $\beta_{send}$}
    \label{tab:words_send}
    \end{multicols}
\caption{SMS deletions and duplications grouped by their labels (Spam or Ham) and the presence/absence of the target word in the text. For each group, the values represent the number of texts in the training set, followed by the number of replicated SMSs, the number of deletions, and the number of SMSs in the poisoned dataset.
}
\end{table}

To provide insight into the effects of the attacks, Tables \ref{tab:words_your} and \ref{tab:words_send} detail the spam and ham SMSs containing the target word that were either duplicated or removed in two attacks with $B=15$, found using the 2O-R2 heuristic. As expected, since the attacker's goal is to shift the posterior of the coefficients associated with the target words towards negative values, leading the defender to infer that the presence of the target word indicates a ham message, most modifications involve deleting spam SMSs containing the target word and replicating some non-spam texts with the word.

\section{COMPUTATION TIMES}

This section reports the running times for the experiments presented in the main text.  
Table \ref{tab:time_simulation} shows the time required to compute attacks on the synthetic dataset with $n=100$ samples from Section 6.1.  
The small dataset size and the ability to sample directly from the posterior enable fast computations.  

\begin{table}[h]
    \centering
    \begin{tabular}{l c c c c c c}
        \toprule
        $B$ & FGSM & 1O-ISCD & 2O-ISCD & SGD-R2 & Adam-R2 & 2O-R2 \\
        \midrule
        5  & 0.0088  & 0.027  & 0.025  & 0.12  & 0.24  & 0.50  \\
        10 & 0.0046  & 0.038  & 0.038  & 0.075  & 0.16  & 0.16  \\
        15 & 0.0044  & 0.050  & 0.052  & 0.17  & 0.24  & 0.11  \\
        20 & 0.0044  & 0.065  & 0.067  & 0.20  & 0.26  & 0.11  \\
        25 & 0.0042  & 0.077  & 0.077  & 0.35  & 1.0  & 1.0  \\
        30 & 0.0042  & 0.10  & 0.077  & 0.35  & 0.12  & 1.9  \\
        35 & 0.0075  & 0.16  & 0.093  & 0.43  & 1.1  & 1.1  \\
        40 & 0.0041  & 0.16  & 0.12  & 0.43  & 0.89  & 0.89  \\
        45 & 0.0040  & 0.16  & 0.13  & 0.26  & 0.67  & 1.4  \\
        50 & 0.0039  & 0.15  & 0.14  & 0.28  & 2.6  & 3.6  \\
        55 & 0.0039  & 0.16  & 0.14  & 0.28  & 2.6  & 3.6  \\
        \bottomrule
    \end{tabular}
    \caption{
    Computation time (in seconds) for a single attack on the synthetic dataset, varying with $B$ and the chosen heuristic. Settings follow those in Section 6.1 for generating Figure 1.}
    \label{tab:time_simulation}
\end{table}

The running times for the experiments on the housing dataset (Section 6.2) are reported in Table \ref{tab:time_housing}.  
These attacks take longer than those in the simulation study due to the necessity to use MCMC to get samples from the posterior distributions. 
Additionally, the housing dataset is larger ($n=404$) than the synthetic one. 

\begin{table}[ht]
    \centering
    \begin{tabular}{lccc}
        \toprule
        B & FGSM & 2O-ISCD & 2O-R2 \\
        \midrule
        10  & 6.0   & 31.8  & 11.0   \\
        20  & 5.5 & 57.5  & 13.3 \\
        30  & 5.5 & 82.2  & 18.4 \\
        40  & 5.6 & 109.4 & 21.4 \\
        50  & 5.4 & 142.3 & 31.5 \\
        60  & 5.2 & 155.8 & 26.8 \\
        \bottomrule
    \end{tabular}
    \caption{
    Computation time (in seconds) for a single attack on the Housing dataset, varying with $B$ and the chosen heuristic. Settings follow those in Section 6.2 for generating Figure 2.}
    \label{tab:time_housing}
\end{table}

Table \ref{tab:time_microcredit} reports the running times for our experiments on the Mexico Microcredit dataset, including those from both the main text and the supplementary materials.  
The large dataset size ($n=16560$) makes MCMC sampling computationally expensive, increasing the overall runtime.

\begin{table}[ht]
    \centering
    \begin{tabular}{lcc}
        \toprule
        $B$ & $\lambda=100$ & $\lambda=10000$ \\
        \midrule
        1 & 20.3 & 19.8 \\
        20 & - & 160.1 \\
        \bottomrule
    \end{tabular}
    \caption{
    Computation time (in seconds) for a single attack on the Microcredit dataset using 2O-ISCD, varying with \( B \). Settings follow those in Section 6.3 of the main text and Section 6.2 of the supplementary materials.}
    \label{tab:time_microcredit}
\end{table}

\end{document}